%% file: arxiv.tex
\pdfoutput=1
\documentclass{article} 
\usepackage{times}
\usepackage{geometry}

\usepackage[utf8]{inputenc} 
\usepackage[T1]{fontenc}    
\usepackage[colorlinks,citecolor=blue]{hyperref}       
\usepackage{url}            
\usepackage{booktabs}       
\usepackage{amsfonts}       
\usepackage{nicefrac}       
\usepackage{microtype}      

\usepackage{mkolar_definitions}
\usepackage{url}
\usepackage{authblk}

\usepackage{graphicx}
\usepackage{subfigure}
\usepackage{multirow}
\usepackage{floatrow}
\newfloatcommand{capbtabbox}{table}[][\FBwidth]

\PassOptionsToPackage{numbers}{natbib}
\usepackage{natbib}

\usepackage{algorithm,algorithmicx,algpseudocode}
\usepackage{xspace}
\usepackage{graphicx}

\usepackage{enumerate}
\usepackage{amsthm,amsmath,amssymb}
\usepackage{amsfonts,dsfont}
\usepackage{nicefrac}
\usepackage{microtype}
\usepackage{mathtools}

\usepackage{xcolor}
\newcount\Comments  
\definecolor{darkgreen}{rgb}{0,0.5,0}
\definecolor{darkred}{rgb}{0.7,0,0}
\definecolor{teal}{rgb}{0.3,0.8,0.8}
\newcommand{\kibitz}[2]{\ifnum\Comments=1\parbox{\linewidth}{\textcolor{#1}{\textsf{\footnotesize #2}}}\fi}

\newcounter{qcounter}
 {\end{list}}

\usepackage{ifthen}
\newcommand{\version}{arxiv}
\newcommand{\df}[1]{#1}
\newcommand{\dfc}[1]{}

\newcommand{\figsqueeze}{0.0cm}

\input{notation.tex}

\newcommand\blfootnote[1]{%
  \begingroup
  \renewcommand\thefootnote{}\footnote{#1}%
  \addtocounter{footnote}{-1}%
  \endgroup
}

\begin{document} 

\title{Active Learning for Cost-Sensitive Classification}

\author[1]{
Akshay Krishnamurthy}
\author[2]{
Alekh Agarwal}
\author[3]{
Tzu-Kuo Huang}
\author[1]{
Hal Daum\'{e} III}
\author[1]{
John Langford}

\affil[1]{Microsoft Research, New York, NY}
\affil[2]{Microsoft Research, Redmond, WA}
\affil[3]{Uber Advanced Technology Center, Pittsburgh, PA}
\maketitle
\blfootnote{akshay@cs.umass.edu, alekha@microsoft.com, tkhuang@protonmail.com, hal@umiacs.umd.edu, jcl@microsoft.com}
\vspace{-1cm}

\input{abstract.tex}

\input{main.tex}

\section*{Acknowledgements}
Part of this research was completed while TKH was at Microsoft
Research and AK was at University of Massachusetts, Amherst. AK thanks
Chicheng Zhang for insightful conversations. AK is supported in part
by NSF Award IIS-1763618.

\newpage

\appendix
\input{appendix.tex}

\bibliographystyle{plainnat}
\bibliography{bib}

\end{document}

%% file: notation.tex
\def\dis{\ensuremath{\textrm{DIS}}}

\def\dismin{\ensuremath{\widetilde{\textrm{DIS}}}}

\def\regconst{\ensuremath{\kappa}}

\newcommand{\logfactor}[1][n]{\ensuremath{\nu_{#1}}} 

\newcommand{\vf}{\ensuremath{f}}
\newcommand{\f}[2]{\ensuremath{\vf(#1;#2)}}
\newcommand{\fstar}[2]{\ensuremath{\vf^\star(#1;#2)}}

\def\vF{\mathcal{F}}
\def\vFcsr{\mathcal{F}_{\textrm{csr}}}
\def\vFmin{\widetilde{\mathcal{F}}}

\newcommand{\Qtree}{\ensuremath{\mathcal{Q}}}
\newcommand{\Qf}{\ensuremath{\textbf{Q}}}

\def\Rhat{\ensuremath{\widehat{R}}}

\newcommand{\cmin}[1]{\ensuremath{c_{-}(#1)}}
\newcommand{\cmax}[1]{\ensuremath{c_{+}(#1)}}
\newcommand{\cminhat}[1]{\ensuremath{\widehat{c_{-}}(#1)}}
\newcommand{\cmaxhat}[1]{\ensuremath{\widehat{c_{+}}(#1)}}

\newcommand{\oracle}{\ensuremath{\textsc{Oracle}}}
\newcommand{\order}{\ensuremath{\mathcal{O}}}
\newcommand{\otil}{\ensuremath{\tilde{\order}}}
\newcommand\ind[1]{\ensuremath{\mathds{1}\left(#1\right)}}
\newcommand{\alglong}{Cost Overlapped Active Learning\xspace}
\newcommand{\alg}{\textsc{COAL}\xspace}
\newcommand{\maxcost}{\textsc{MaxCost}\xspace}

\newcommand{\mincost}{\textsc{MinCost}\xspace}
\newcommand{\csoaa}{\textsc{csoaa}\xspace}
\newcommand{\aggravate}{\textsc{Aggravate}\xspace}

\def\perf{\ensuremath{\mathrm{perf}}}
\def\query{\ensuremath{\mathrm{query}}}

\def\auc{\ensuremath{\mathrm{AUC}}}
\def\median{\ensuremath{\mathrm{median}}}
\def\length{\ensuremath{len}}
\def\mel{\ensuremath{mel}}

\def\pdim{\ensuremath{\mathrm{Pdim}}}

\def\tol{\ensuremath{\textsc{tol}}}
\def\mw{\ensuremath{\mu}}

%% file: abstract.tex
\begin{abstract}
We design an active learning algorithm for cost-sensitive multiclass
classification: problems where different errors have different
costs. Our algorithm, \alg, makes predictions by regressing to each
label's cost and predicting the smallest.  On a new example, it uses a
set of regressors that perform well on past data to estimate possible
costs for each label. It queries only the labels that \emph{could be}
the best, ignoring the sure losers. We prove \alg can be efficiently
implemented for any regression family that admits squared loss
optimization; it also enjoys strong guarantees with respect to
predictive performance and labeling effort.  We empirically compare
\alg to passive learning and several active learning baselines,
showing significant improvements in labeling effort and test cost on
real-world datasets.
\end{abstract}

%% file: main.tex
\section{Introduction}
The field of active learning studies how to efficiently elicit
relevant information so learning algorithms can make good
decisions. Almost all active learning algorithms are designed for
binary classification problems, leading to the natural question: How
can active learning address more complex prediction problems?
Multiclass and importance-weighted classification require only minor
modifications but we know of no active learning algorithms that enjoy
theoretical guarantees for more complex problems.

One such problem is cost-sensitive multiclass classification
(CSMC). In CSMC with $K$ classes, passive learners receive input
examples $x$ and cost vectors $c \in \mathbb{R}^K$, where $c(y)$ is
the cost of predicting label $y$ on $x$.\footnote{Cost here refers to
  prediction cost and not labeling effort or the cost of acquiring
  different labels.}  A natural design for an \emph{active} CSMC
learner then is to adaptively query the costs of only a (possibly
empty) subset of labels on each $x$.  Since measuring label complexity
is more nuanced in CSMC (e.g., is it more expensive to query three
costs on a single example or one cost on three examples?), we track
both the number of examples for which at least one cost is queried,
along with the total number of cost queries issued.  The first
corresponds to a fixed human effort for inspecting the example. The
second captures the additional effort for judging the cost of each
prediction, which depends on the number of labels queried. (By
querying a label, we mean querying the cost of predicting that label
given an example.)

In this setup, we develop a new active learning algorithm for CSMC
called \alglong (\alg). \alg assumes access to a set of regression
functions, and, when processing an example $x$, it uses the functions
with good past performance to compute the range of possible costs that
each label might take.  Naturally, \alg only queries labels with large
cost range, akin to uncertainty-based approaches in active
regression~\citep{castro2005faster}, but furthermore, it only queries
labels that could possibly have the smallest cost, avoiding the
uncertain, but surely suboptimal labels. The key algorithmic
innovation is an efficient way to compute the cost range realized by
good regressors. This computation, and \alg as a whole, only requires
that the regression functions admit efficient squared loss optimization, in
contrast with prior algorithms that require 0/1 loss
optimization~\citep{BeygelDL09,Hanneke2014}.

Among our results, we prove that when processing $n$ (unlabeled)
examples with $K$ classes and a regression class with pseudo-dimension
$d$ (See Definition~\ref{def:pdim}),
\begin{enumerate}
\item The algorithm needs to solve $\order(K n^5)$ regression problems
  over the function class (Corollary~\ref{cor:oracle_bound}). Thus
  \alg runs in polynomial time for convex regression sets.
\item With no assumptions on the noise in the problem, the algorithm
  achieves generalization error $\otil(\sqrt{Kd/n})$ and requests
  $\otil(n\theta_2\sqrt{Kd})$ costs from $\otil(n\theta_1\sqrt{Kd})$
  examples (Theorems~\ref{thm:reg_bound} and~\ref{thm:high_noise}) where
  $\theta_1,\theta_2 $ are the disagreement coefficients
  (Definition~\ref{def:dis_coeff})\footnote{$\otil(\cdot)$ suppresses
    logarithmic dependence on $n$, $K$, and $d$.}. The worst case
  offers minimal improvement over passive learning, akin to active
  learning for binary classification.
\item With a Massart-type noise assumption
  (Assumption~\ref{as:low_noise}), the algorithm has generalization
  error $\otil(Kd/n)$ while requesting $\otil(Kd (\theta_2 +K\theta_1)
  \log n )$ labels from $\otil(Kd \theta_1 \log n)$ examples
  (Corollary~\ref{cor:reg_low_noise}, Theorem~\ref{thm:massart}). Thus
  under favorable conditions, \alg requests \emph{exponentially fewer}
  labels than passive learning.
\end{enumerate}
We also derive generalization and label complexity bounds under a
milder Tsybakov-type noise condition
(Assumption~\ref{as:tsybakov}). Existing lower bounds from binary
classification~\citep{Hanneke2014} \df{suggest} that \dfc{all of} our results are
optimal in their dependence on $n$\df{, although these lower bounds do not directly apply to our setting}.  We also discuss some intuitive
examples highlighting the benefits of using \alg.

\begin{figure}[t]
\centering
\includegraphics[width=0.4\textwidth,height=0.3\textwidth]{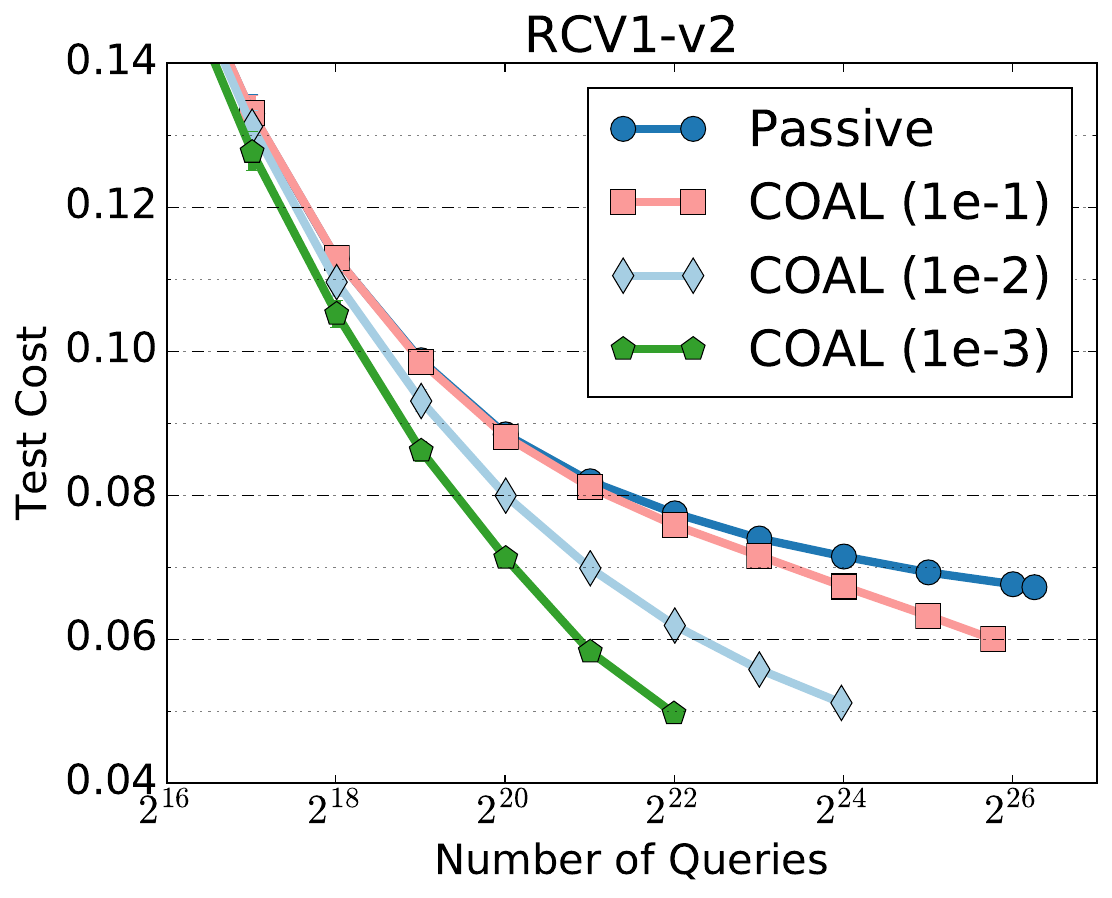}
\vspace{\figsqueeze}
\caption{Empirical evaluation of \alg on Reuters text categorization dataset. Active learning achieves \emph{better} test cost than passive, with a factor of $16$ fewer queries. See Section~\ref{sec:experiments} for details. }
\label{fig:rcv1_results}
\vspace{\figsqueeze}
\end{figure}

CSMC provides a more expressive language for success and failure than
multiclass classification, which allows learning algorithms to make
the trade-offs necessary for good performance and broadens potential
applications. For example, CSMC can naturally express partial failure
in hierarchical
classification~\citep{silla2011survey}. Experimentally, we show that
\alg substantially outperforms the passive learning baseline with
orders of magnitude savings in the labeling effort on a number of
hierarchical classification datasets (see
Figure~\ref{fig:rcv1_results} for comparison between passive learning
and \alg on Reuters text categorization).

CSMC also forms the basis of learning to avoid cascading failures in
joint prediction tasks like structured prediction and reinforcement
learning~\citep{daume09searn,ross2014reinforcement,chang2015learning}.
As our second application, we consider learning to search algorithms
for joint or structured prediction, which operate by a reduction to
CSMC. In this reduction, evaluating the cost of a class often involves
a computationally expensive ``roll-out,'' so using an active learning
algorithm inside such a passive joint prediction method can lead to
significant computational savings. We show that using \alg within the
\aggravate algorithm~\citep{ross2014reinforcement,chang2015learning}
reduces the number of roll-outs by a factor of $\frac 1 4$ to $\frac 3
4$ on several joint prediction tasks.

Our code is publicly available as part of the Vowpal Wabbit machine
learning library.\footnote{\url{http://hunch.net/~vw}}


\section{Related Work}
\label{sec:related}
Active learning is a thriving research area with many theoretical and
empirical studies. We recommend the survey
of~\citet{settles2012active} for an overview of more empirical
research. We focus here on theoretical results.



Our work falls into the framework of \emph{disagreement-based active
  learning}, which studies general hypothesis spaces typically in an
agnostic setup (see \citet{Hanneke2014} for an excellent
survey). Existing results study binary classification, while our work
generalizes to CSMC, assuming that we can accurately predict costs
using regression functions from our class. One difference that is
natural for CSMC is that our query rule checks the range of predicted
costs for a label.


The other main difference is that we use a square loss oracle to
search the version space.  In contrast, prior work either explicitly
enumerates the version
space~\citep{balcan2006agnostic,zhang2014beyond} or uses a 0/1 loss
\emph{classification} oracle for the
search~\citep{DasguptaHM07,BeygelDL09,BeygelHLZ10,huang2015efficient}.
In most instantiations, the oracle solves an NP-hard problem and so
does not directly lead to an efficient algorithm, although practical
implementations using heuristics are still quite effective. Our
approach instead uses a squared-loss \emph{regression} oracle, which
can be implemented efficiently via convex optimization and leads to a
polynomial time algorithm.

\df{In addition to disagreement-based approaches, much research has
  focused on plug-in rules for active learning in binary
  classification, where one estimates the class-conditional regression
  function~\citep{Castro2008,Minsker2012plug,hanneke2012surrogate,carpentier2017adaptivity}. Apart
  from~\citet{hanneke2012surrogate}, these works make smoothness
  assumptions and have a nonparametric
  flavor. Instead,~\citet{hanneke2012surrogate} assume a calibrated
  surrogate loss and abstract realizable function class, which is more
  similar to our setting. While the details vary, our work and these
  prior results employ the same algorithmic recipe of maintaining an
  implicit version space and querying in a suitably-defined
  disagreement region. Our work has two notable differences: (1) our
  algorithm operates in an oracle computational model, only accessing
  the function class through square loss minimization problems, (2)
  our results apply to general CSMC, which exhibit significant
  differences from binary classification. See
  Subsection~\ref{ssec:comparisons} for further discussion.}


\df{Focusing on linear
  representations,~\citet{balcan2007margin,balcan2013active} study
  active learning with distributional assumptions, while the selective
  sampling framework from the online learning community considers
  adversarial
  assumptions~\citep{Cavallanti2011,Dekel2010,Orabona2011,agarwal2013selective}. These
  methods use query strategies that are specialized to linear
  representations and do not naturally generalize to other hypothesis
  classes. }

Supervised learning oracles that solve NP-hard optimization problems
in the worst case have been used in other problems including
contextual bandits~\citep{agarwal2014taming,syrgkanis2016} and
structured prediction~\citep{daume09searn}. Thus we hope that our work
can inspire new algorithms for these settings as well.

Lastly, we mention that square loss regression has been used to
estimate costs for passive CSMC~\citep{langford2005sensitive}, but, to our knowledge,
using a square loss oracle for active CSMC is new.

\paragraph{Advances over~\citet{krishnamurthy2017active}.}
Active learning for CSMC was introduced recently
in~\citet{krishnamurthy2017active} with an algorithm that also uses
cost ranges to decide where to query. They compute cost ranges by
using the regression oracle to perform a binary search for the maximum
and minimum costs, but this computation results in a sub-optimal label
complexity bound. We resolve this sub-optimality with a novel cost
range computation that is inspired by the multiplicative weights
technique for solving linear programs. This algorithmic improvement
also requires a significantly more sophisticated statistical analysis
for which we derive a novel uniform Freedman-type inequality for
classes with bounded pseudo-dimension. This result may be of
independent interest.

\citet{krishnamurthy2017active} also introduce an online approximation
for additional scalability and use this algorithm for their
experiments. Our empirical results use this same online approximation
and are slightly more comprehensive. Finally, we also derive
generalization and label complexity bounds for our algorithm in a
setting inspired by Tsybakov's low noise
condition~\citep{mammen1999smooth,tsybakov2004optimal}.

\paragraph{\df{Comparison with~\citet{foster2018practical}.}}
\df{In a follow-up to the present paper,~\citet{foster2018practical}
  build on our work with a regression-based approach for contextual
  bandit learning, a problem that bears some similarities to active
  learning for CSMC. The results are incomparable due to the
  differences in setting, but it is worth discussing their
  techniques. As in our paper,~\citet{foster2018practical} maintain an
  implicit version space and compute maximum and minimum costs for
  each label, which they use to make predictions. They resolve the
  sub-optimality in~\citet{krishnamurthy2017active} with epoching,
  which enables a simpler cost range computation than our
  multiplicative weights approach. However, epoching incurs an
  additional $\log(n)$ factor in the label complexity, and under
  low-noise conditions where the overall bound is
  $\order(\textrm{polylog}(n))$, this yields a polynomially worse
  guarantee than ours.}

\section{Problem Setting and Notation}
We study cost-sensitive multiclass classification (CSMC) problems with
$K$ classes, where there is an instance space $\Xcal$, a label space
$\Ycal = \{1,\ldots, K\}$, and a distribution $\Dcal$ supported on $\Xcal
\times [0,1]^K$.\footnote{In general, labels just serve as indices for
  the cost vector in CSMC, and the data distribution is over $(x,c)$
  pairs instead of $(x,y)$ pairs as in binary and multiclass
  classification.}  If $(x,c) \sim \Dcal$, we refer to $c$ as the
\emph{cost-vector}, where $c(y)$ is the cost of predicting $y \in \Ycal$.
A classifier $h: \Xcal \rightarrow \Ycal$ has expected cost $\EE_{(x,c)
  \sim \Dcal}[c(h(x))]$ and we aim to find a classifier with minimal
expected cost.

Let $\Gcal \triangleq \{g: \Xcal \mapsto [0,1]\}$ denote a set of base
regressors and let $\Fcal \triangleq \Gcal^K$ denote a set of vector
regressors where the $y^{\textrm{th}}$ coordinate of $\vf \in \Fcal$
is written as $\f{\cdot}{y}$. The set of classifiers under
consideration is $\Hcal \triangleq \{h_{\vf} \mid \vf \in \Fcal\}$
where each $\vf$ defines a classifier $h_{\vf}: \Xcal
\mapsto \Ycal$ by
\begin{align}
  h_{\vf}(x) \triangleq \argmin_{y} \f{x}{y}. \label{eq:reg_to_class}
\end{align}

When using a set of regression functions for a classification task, it
is natural to assume that the expected costs under $\Dcal$ can be
predicted by some function in the set. This motivates the
following realizability assumption.
\begin{assum}[Realizability]
  Define the Bayes-optimal regressor $\vf^\star$, which has
  $\fstar{x}{y} \triangleq \EE_{c}[c(y)|x], \forall x \in \Xcal$ (with $\Dcal(x) > 0$), $y \in \Ycal$.  We
  assume that $\vf^\star \in \Fcal$.
  \label{assumption:realizability}
\end{assum}
While $\vf^\star$ is always well defined, note that the cost itself
may be noisy.  In comparison with our assumption, the existence of a
zero-cost classifier in $\Hcal$ (which is often assumed in active
learning) is stronger, while the existence of $h_{f^\star}$ in $\Hcal$
is weaker but has not been leveraged in active learning.

We also require assumptions on the complexity of the class $\Gcal$ for our
statistical analysis. To this end, we assume that $\Gcal$ is a
compact convex \df{sub}set \df{of $L_{\infty}(\Xcal)$} with finite pseudo-dimension, which is a natural
extension of VC-dimension for real-valued predictors.
\begin{definition}[Pseudo-dimension]
\label{def:pdim}
The pseudo-dimension $\pdim(\Fcal)$ of a function class $\Fcal: \Xcal
\rightarrow \RR$ is defined as the VC-dimension of the set of
threshold functions $\Hcal^+ \triangleq \{(x,\xi) \mapsto \one\{f(x) > \xi\} :
f \in \Fcal\} \subset \Xcal \times \RR \rightarrow \{0,1\}$.
\end{definition}
\begin{assum}
  We assume that $\Gcal$ is a compact convex set with $\pdim(\Gcal) = d < \infty$.
\end{assum}
\df{As an example, linear functions in some basis representation,
  e.g., $g(x) = \sum_{i=1}^dw_i\phi_i(x)$, where weights $w_i$ are
  bounded in some norm, have pseudodimension $d$. In fact, our result
  can be stated entirely in terms of covering numbers, and we
  translate to pseudo-dimension using the fact that such classes have
  ``parametric" covering numbers of the form $(1/\varepsilon)^d$. }
Thus, our
results extend to classes with ``nonparametric" growth rates as well
\df{(e.g., Holder-smooth functions)}, although we focus on the
parametric case for simplicity. Note that this is a significant
departure from~\citet{krishnamurthy2017active}, which assumed that
$\Gcal$ was finite.

Our assumption that $\Gcal$ is a compact convex set introduces a
computational challenging of managing this infinitely large
set.
\dfc{Indeed, several early results on active learning and other
information acquisition problems assume the hypothesis space is finite
and enumerate the
class,
which is
clearly not possible in our setting. }
To address this challenge, we
follow the trend in active learning of leveraging existing algorithmic
research on supervised
learning~\citep{DasguptaHM07,BeygelHLZ10,BeygelDL09} and access
$\Gcal$ exclusively through a \emph{regression oracle}. Given an
importance-weighted dataset $D = \{x_i,c_i,w_i\}_{i=1}^n$ where $x_i
\in \Xcal, c_i \in \RR, w_i \in \RR_+$, the regression oracle computes
\begin{align}
  \oracle(D) \in \argmin_{g \in \Gcal} \sum_{i=1}^n w_i (g(x_i) - c_i)^2.
  \label{eq:oracle}
\end{align}
Since we assume that $\Gcal$ is a compact convex set it is amenable to
standard convex optimization techniques, so this imposes no additional
restriction. However, in the special case of linear functions, this
optimization is just least squares and can be computed in closed
form. Note that this is fundamentally different from prior works that
use a 0/1-loss minimization
oracle~\citep{DasguptaHM07,BeygelHLZ10,BeygelDL09}, which involves an
NP-hard optimization in most cases of interest.

\begin{remark}
\df{Our assumption that $\Gcal$ is convex is only for computational
  tractability, as it is crucial in the efficient implementation of
  our query strategy, but is not required for our generalization and
  label complexity bounds. Unfortunately recent guarantees for
  learning with non-convex
  classes~\citep{liang2015learning,rakhlin2017empirical} do not
  immediately yield efficient active learning strategies. Note also
  that~\citet{krishnamurthy2017active} obtain an efficient algorithm
  without convexity, but this yields a suboptimal label complexity
  guarantee. }
\end{remark}

Given a set of examples and queried costs, we often restrict attention
to regression functions that predict these costs well and assess the
uncertainty in their predictions given a new example $x$. For a subset
of regressors $G \subset \Gcal$, we measure uncertainty over possible
cost values for $x$ with
\begin{align}
  \gamma(x, G) &\triangleq \cmax{x,G} - \cmin{x,G}, \qquad \cmax{x,G} \triangleq \max_{g \in G} g(x), \qquad  \cmin{x,G} \triangleq\min_{g \in G}g(x)\label{eq:max_cost}.
\end{align}
For vector regressors $F\subset\Fcal$, we define the \emph{cost range}
for a label $y$ given $x$ as $\gamma(x,y,F) \triangleq
\gamma(x,G_F(y))$ where $G_F(y) \triangleq \{f(\cdot;y)~|~f\in F\}$
are the base regressors induced by $F$ for $y$.  \df{Note that since
  we are assuming realizability, whenever $f^\star \in F$, the
  quantities $\cmax{x,G_F(y)}$ and $\cmin{x,G_F(y)}$ provide valid
  upper and lower bounds on $\EE[c(y)|x]$.}

To measure the labeling effort, we track the number of examples for
which even a single cost is queried as well as the total number of
queries. This bookkeeping captures settings where the editorial effort
for inspecting an example is high but each cost requires minimal
further effort, as well as those where each cost requires substantial
effort.  Formally, we define $Q_i(y) \in \{0,1\}$ to be the indicator
that the algorithm queries label $y$ on the $i^{\textrm{th}}$ example
and measure
\begin{align}
L_1 \triangleq \sum_{i=1}^n \bigvee_y Q_i(y), \textrm{ and } L_2
\triangleq \sum_{i=1}^n\sum_y Q_i(y).
\label{eq:label_effort}
\end{align}

\section{\alglong}

\begin{algorithm}[t]
  \caption{\alglong (\alg)}
  \label{alg:a_cs}
  \begin{algorithmic}[1]
    \State Input: Regressors $\Gcal$, failure probability $\delta \leq 1/e$.
    \State Set $\psi_i = 1 /\sqrt{i}$, $\regconst = 3$, $\logfactor = 324(d\log(n) + \log(8Ke(d+1)n^2/\delta))$.
    \State Set $\Delta_i = \regconst \min\{\frac{\logfactor}{i-1}, 1\}$.
    \For{$i=1,2,\ldots,n$}
    \State $g_{i,y} \gets \arg \min_{g \in \Gcal}
    \Rhat_i(g;y)$. (See~\eqref{eq:empirical_sq_loss}).
    \State Define $f_i \gets \{g_{i,y}\}_{y=1}^K$. \label{line:fi}
    \State (Implicitly define) $\Gcal_i(y) \gets \{g \in \Gcal_{i-1}(y) \mid \Rhat_i(g;y)
    \leq \Rhat_i(g_{i,y};y) +
    \Delta_{i}\}$.\label{line:version_space}
    \State Receive new example $x$. $Q_i(y) \gets 0, \forall y \in \Ycal$.
    \For{every $y \in \Ycal$}
    \State $\cmaxhat{y} \gets \maxcost((x,y),\psi_i/4)$ and $\cminhat{y} \gets \mincost((x,y), \psi_i/4)$.
    \EndFor
    \State $Y' \gets \{y \in \Ycal \mid \cminhat{y} \leq \min_{y'}
    \cmaxhat{y'}\}.$ \label{line:domination}
    \If{$|Y'| > 1$}
    \State $Q_i(y) \gets 1$ if $y \in Y'$ and $
    \cmaxhat{y} - \cminhat{y} > \psi_i$.\label{line:range}
    \EndIf
    \State Query costs of each $y$ with $Q_i(y) = 1$.
    \EndFor
  \end{algorithmic}
\end{algorithm}

The pseudocode for our algorithm, \alglong (\alg), is given in
Algorithm~\ref{alg:a_cs}.  Given an example $x$, \alg queries the
costs of some of the labels $y$ for $x$. These costs are chosen by (1)
computing a set of good regression functions based on the past data
(i.e., the version space), (2) computing the range of predictions
achievable by these functions for each $y$, and (3) querying each $y$
that could be the best label \emph{and} has substantial
uncertainty. We now detail each step.

To compute an approximate version space we first find the regression
function that minimizes the empirical risk for each label $y$, which
at round $i$ is:
\begin{align}
  \Rhat_i(g;y) \triangleq \frac{1}{i-1}\sum_{j=1}^{i-1}(g(x_j) -
  c_j(y))^2Q_j(y). \label{eq:empirical_sq_loss}
\end{align}
Recall that $Q_j(y)$ is the indicator that we query label $y$ on the
$j^{\textrm{th}}$ example.  Computing the minimizer requires one
oracle call. We implicitly construct the version space $\Gcal_i(y)$ in
Line~\ref{line:version_space} as the surviving regressors with low
square loss regret to the empirical risk minimizer. The tolerance on
this regret is $\Delta_i$ at round $i$, which scales like
$\otil(d/i)$, where recall that $d$ is the pseudo-dimension of the
class $\Gcal$.

\alg then computes the maximum and minimum costs predicted by the
version space $\Gcal_i(y)$ on the new example $x$. Since the true
expected cost is $f^\star(x;y)$ and, as we will see, $f^\star(\cdot;y)
\in \Gcal_i(y)$, these quantities serve as a confidence bound for this
value. The computation is done by the \maxcost and \mincost
subroutines which produce approximations to $\cmax{x,\Gcal_i(y)}$ and
$\cmin{x,\Gcal_i(y)}$ respectively (See~\eqref{eq:max_cost}).

Finally, using the predicted costs, \alg issues (possibly zero)
queries. The algorithm queries any \emph{non-dominated} label that has
a large \emph{cost range}, where a label is non-dominated if its
estimated minimum cost is smaller than the smallest maximum cost
(among all other labels) and the cost range is the difference between the
label's estimated maximum and minimum costs.

Intuitively, \alg queries the cost of every label which cannot be
ruled out as having the smallest cost on $x$, but only if there is
sufficient ambiguity about the actual value of the cost.  The idea is
that labels with little disagreement do not provide much information
for further reducing the version space, since by construction all
regressors would suffer similar square loss. Moreover, only the labels
that could be the best need to be queried at all, since the
cost-sensitive performance of a hypothesis $h_f$ depends only on the
label that it predicts. Hence, labels that are dominated or have small
cost range need not be queried.

Similar query strategies have been used in prior works on binary and multiclass
classification~\citep{Orabona2011,Dekel2010,agarwal2013selective}, but
specialized to linear representations. The key advantage of the linear
case is that the set $\Gcal_i(y)$ (formally, a different set with
similar properties) along with the maximum and minimum costs have
closed form expressions, so that the algorithms are easily
implemented. However, with a general set $\Gcal$ and a regression
oracle, computing these confidence intervals is less
straightforward. We use the \maxcost and \mincost subroutines, and
discuss this aspect of our algorithm next.

\subsection{Efficient Computation of Cost Range}
\label{sec:cost_range_exact}

\begin{algorithm}[t]
  \caption{\maxcost}
  \label{alg:mw_max_cost}
  \begin{algorithmic}[1]
    \State Input: $(x,y)$, tolerance $\tol$, (implicitly) risk functionals $\{\Rhat_j(\cdot;y)\}_{j=1}^i$.
    \State Compute $g_{j,y} = \argmin_{g \in \Gcal} \Rhat_j(g;y)$ for each $j$.
    \State Let $\Delta_j = \regconst \min\{\frac{\logfactor}{j-1},1\}$,  $\tilde{\Delta}_j = \Rhat_j(g_{j,y};y) + \Delta_j$ for each $j$.
    \State Initialize parameters: $c_\ell \gets  0 , c_h \gets 1, T \gets \frac{\log(i+1)(12/\df{\Delta_i})^2}{\tol^4}, \eta \gets \sqrt{\log(i+1)/T}$.
    \While{$|c_\ell - c_h| > \tol^2/2$} \label{line:mw_outer_while}
    \State $c \gets \frac{c_h - c_\ell}{2}$
    \State $\mw^{(1)} \gets \one \in \RR^{i+1}$.     \Comment Use MW to check feasibility of Program~\eqref{eq:max_cost_feasibility}.
    \For{$t = 1,\ldots,T$}
    \State Use the regression oracle to find
    \begin{align}
      g_t \gets \argmin_{g \in \Gcal} \mw^{(t)}_0(g(x)-1)^2 + \sum_{j=1}^i \mw^{(t)}_j \Rhat_j(g;y)
      \label{eq:mw_separation_problem}
    \end{align}
    \State If the objective in~\eqref{eq:mw_separation_problem} for $g_t$ is at least
    $\mw^{(t)}_0c + \sum_{j=1}^i \mw^{(t)}_j\tilde{\Delta}_j$, $c_\ell \gets c$, go to~\ref{line:mw_outer_while}. \label{line:mw_infeasible}
    \State Update
    \begin{align*}
      \mw^{(t+1)}_j \gets \mw^{(t)}_j \left( 1-\eta\frac{\tilde{\Delta}_j - \Rhat_j(g_t;y)}{\Delta_j+1} \right), \quad \mw_0^{(t+1)} \gets \mw_0^{(t)} \left(1-\eta\frac{c-(g_t(x)-1)^2}{2}\right).
    \end{align*}
    \EndFor
    \State $c_h \gets c$.
      \EndWhile
      \State Return $\cmaxhat{y} = 1 - \sqrt{c_\ell}$.
  \end{algorithmic}
\end{algorithm}

In this section, we describe the $\maxcost$ subroutine which uses the
regression oracle to approximate the maximum cost on label $y$ realized by
$\Gcal_i(y)$, as defined in~\eqref{eq:max_cost}. The minimum cost
computation requires only minor modifications that we discuss at the
end of the section.

Describing the algorithm requires some additional notation.  Let
$\tilde{\Delta}_j \triangleq \Delta_j + \Rhat_j(g_{j,y};y)$ be the
right hand side of the constraint defining the version space at round
$j$, where $g_{j,y}$ is the ERM at round $j$ for label $y$,
$\Rhat_j(\cdot;y)$ is the risk functional, and $\Delta_j$ is the
radius used in \alg. Note that this quantity can be efficiently
computed since $g_{j,y}$ can be found with a single oracle call.  Due
to the requirement that $g \in \Gcal_{i-1}(y)$ in the definition of
$\Gcal_i(y)$, an equivalent representation is $\Gcal_i(y) =
\bigcap_{j=1}^{i}\{g : \Rhat_j(g;y) \le \tilde{\Delta}_j\}$. Our
approach is based on the observation that given an example $x$ and a
label $y$ at round $i$, finding a function $g \in \Gcal_i(y)$ which
predicts the maximum cost for the label $y$ on $x$ is equivalent to
solving the minimization problem:
\begin{align}
\label{eq:max_cost_opt}
\textrm{minimize}_{g \in \Gcal} (g(x) - 1)^2 \quad \textrm{ such that
} \quad \forall 1 \le j \le i, \Rhat_j(g;y) \le \tilde{\Delta}_j.
\end{align}
Given this observation, our strategy will be to find an approximate
solution to the problem~\eqref{eq:max_cost_opt} and it is not
difficult to see that this also yields an approximate value for the
maximum predicted cost on $x$ for the label $y$.

In Algorithm~\ref{alg:mw_max_cost}, we show how to efficiently solve
this program using the regression oracle. We begin by exploiting the
convexity of the set $\Gcal$, meaning that we can further rewrite the
optimization problem~\eqref{eq:max_cost_opt} as
\begin{align}
\label{eq:max_cost_conv}
\textrm{minimize}_{P \in \Delta(\Gcal)} \EE_{g \sim P}\left[(g(x) -
  1)^2\right] \quad \textrm{ such that } \quad \forall 1 \le j \le i,
\EE_{g \sim P}\left[\Rhat_j(g;y)\right] \le \tilde{\Delta}_j.
\end{align}
The above rewriting is effectively cosmetic as $\Gcal = \Delta(\Gcal)$
by the definition of convexity, but the upshot is that our rewriting
results in both the objective and constraints being linear in the
optimization variable $P$. Thus, we effectively wish to solve a linear
program in $P$, with our computational tool being a regression oracle
over the set $\Gcal$. To do this, we create a series of feasibility
problems, where we repeatedly guess the optimal objective value for
the problem~\eqref{eq:max_cost_conv} and then check whether there is
indeed a distribution $P$ which satisfies all the constraints and
gives the posited objective value. That is, we check
\begin{align}
\label{eq:max_cost_feasibility}
?\exists P \in \Delta(\Gcal) \textrm{ such that } \EE_{g \sim
  P}(g(x)-1)^2 \le c \textrm{ and } \forall 1 \le j \le i, \EE_{g \sim
  P}\Rhat_j(g;y) \le \tilde{\Delta}_j.
\end{align}
If we find such a solution, we increase our guess, and otherwise we reduce the
guess and proceed until we localize the optimal value to a small enough
interval.

It remains to specify how to solve the feasibility
problem~\eqref{eq:max_cost_feasibility}. Noting that this is a linear
feasibility problem, we jointly invoke the Multiplicative Weights (MW)
algorithm and the regression oracle in order to either find an
approximately feasible solution or certify the problem as
infeasible. MW is an iterative algorithm that maintains weights $\mw$
over the constraints. At each iteration it (1) collapses the
constraints into one, by taking a linear combination weighted by
$\mw$, (2) checks feasibility of the simpler problem with a single
constraint, and (3) if the simpler problem is feasible, it updates the
weights using the slack of the proposed solution. Details of steps (1)
and (3) are described in Algorithm~\ref{alg:mw_max_cost}.

For step (2), the simpler problem that we must solve takes the form
\begin{align*}
?\exists P \in \Delta(\Gcal) \textrm{ such that } \mw_0 \EE_{g \sim
    P}(g(x)-1)^2 + \sum_{j=1}^i \mw_j \EE_{g \sim P}\Rhat_j(g;y) \le
  \mw_0c + \sum_{j=1}^i \mw_j\tilde{\Delta}_j.
\end{align*}
This program can be solved by a single call to the regression oracle,
since all terms on the left-hand-side involve square losses while the
right hand side is a constant. Thus we can efficiently implement the
MW algorithm using the regression oracle. Finally, recalling that the
above description is for a fixed value of objective $c$, and recalling
that the maximum can be approximated by a binary search over $c$ leads
to an oracle-based algorithm for computing the maximum cost. For this
procedure, we have the following computational guarantee.

\begin{theorem}
\label{thm:mw_guarantee}
Algorithm~\ref{alg:mw_max_cost} returns an estimate $\cmaxhat{x;y}$
such that $\cmax{x;y} \le \cmaxhat{x;y} \le \cmax{x;y} + \tol$ and
runs in polynomial time with $\order(\max\{1, i^2/\logfactor^2\}
\log(i)\log(1/\tol)/\tol^4)$ calls to the regression oracle.
\end{theorem}
The minimum cost can be estimated in exactly the same way, replacing
the objective $(g(x)-1)^2$ with $(g(x)-0)^2$ in
Program~\eqref{eq:max_cost_opt}.  In \alg, we set $\tol = 1/\sqrt{i}$
at iteration $i$ and have $\logfactor = \otil(d)$. As a
consequence, we can bound the total oracle complexity after processing
$n$ examples.
\begin{corollary}
\label{cor:oracle_bound}
After processing $n$ examples, \alg makes $\otil(K(d^3 + n^5/d^2))$ calls
to the square loss oracle.
\end{corollary}

Thus \alg can be implemented in polynomial time for any set $\Gcal$
that admits efficient square loss optimization. Compared
to~\citet{krishnamurthy2017active} which required $\order(n^2)$ oracle
calls, the guarantee here is, at face value, worse, since the
algorithm is slower. However, the algorithm enforces a much stronger
constraint on the version space which leads to a much better
statistical analysis, as we will discuss next. Nevertheless, these
algorithms that use batch square loss optimization in an iterative or
sequential fashion are too computational demanding to scale to larger
problems. Our implementation alleviates this with an alternative
heuristic approximation based on a sensitivity analysis of the oracle,
which we detail in Section~\ref{sec:experiments}.

\section{Generalization Analysis}
In this section, we derive generalization guarantees for \alg. We
study three settings: one with minimal assumptions and two low-noise
settings.

Our first low-noise assumption is related to the Massart noise
condition~\citep{massart2006risk}, which in binary classification
posits that the Bayes optimal predictor is bounded away from $1/2$ for
all $x$.  Our condition generalizes this to CSMC and posits that the
expected cost of the best label is separated from the expected cost of
all other labels.
\begin{assum}
  \label{as:low_noise}
  A distribution $\Dcal$ supported over $(x,c)$ pairs satisfies the
  \emph{Massart noise condition} with parameter $\tau > 0$, if for all
  $x$ (with $\Dcal(x) > 0$),
  \begin{align*}
    \fstar{x}{y^\star(x)} \le \min_{y \ne y^\star(x)} \fstar{x}{y} - \tau,
  \end{align*}
  where $y^\star(x) \triangleq \mathop{\textrm{argmin}}_{y}\fstar{x}{y}$ is the
  true best label for $x$.
\end{assum}

The Massart noise condition describes favorable prediction problems
that lead to sharper generalization and label complexity bounds for
\alg. We also study a milder noise assumption, inspired by the
Tsybakov condition~\citep{mammen1999smooth,tsybakov2004optimal},
again generalized to CSMC. See also~\citet{agarwal2013selective}.
\begin{assum}
  \label{as:tsybakov}
  A distribution $\Dcal$ supported over $(x,c)$ pairs satisfies the \emph{Tsbyakov noise condition} with parameters $(\tau_0,\alpha,\beta)$ if for all $0 \le \tau \le \tau_0$,
  \begin{align*}
    \PP_{x \sim \Dcal}\left[\min_{y \ne y^\star(x)} \fstar{x}{y} - \fstar{x}{y^\star(x)} \le \tau \right]\le \beta \tau^\alpha,
  \end{align*}
  where $y^\star(x) \triangleq \mathop{\textrm{argmin}}_y\fstar{x}{y}$.
\end{assum}
Observe that the Massart noise condition in
Assumption~\ref{as:low_noise} is a limiting case of the Tsybakov
condition, with $\tau = \tau_0$ and $\alpha \rightarrow \infty$.
The Tsybakov condition states that it is polynomially unlikely for the
cost of the best label to be close to the cost of the other
labels. This condition has been used in previous work on
cost-sensitive active learning~\citep{agarwal2013selective} and is
also related to the condition studied by \citet{Castro2008} with the
translation that $\alpha = \frac{1}{\kappa-1}$, where $\kappa \in
[0,1]$ is their noise level.

Our generalization bound is stated in terms of the noise level in the
problem so that they can be readily adapted to the favorable
assumptions. We define the noise level using the following quantity,
given any $\zeta > 0$.
\begin{align}
  \label{eq:noise_level}
  P_\zeta \triangleq \PP_{x \sim \Dcal}\left[\min_{y \ne y^\star(x)}
    \vf^\star(x;y) - \vf^\star(x;y^\star(x)) \le \zeta\right].
\end{align}
$P_\zeta$ describes the probability that the expected cost of the best
label is close to the expected cost of the second best label. When
$P_\zeta$ is small for large $\zeta$ the labels are well-separated so
learning is easier. For instance, under a Massart condition $P_\zeta =
0$ for all $\zeta \leq \tau$.

We now state our generalization guarantee.
\begin{theorem}
  For any $\delta < 1/e$, for all $i \in [n]$, with probability at least $1-\delta$,
  we have
  \begin{align*}
    & \EE_{x,c}[c(h_{f_{i+1}}(x)) - c(h_{\vf^\star}(x))] \leq
    \min_{\zeta > 0}\left\{ \zeta P_\zeta + \frac{32 K
      \logfactor}{\zeta i} \right\},
  \end{align*}
  where $\logfactor$, $f_i$ are defined in Algorithm~\ref{alg:a_cs},
  and $h_{f_i}$ is defined in~\eqref{eq:reg_to_class}.
  \label{thm:reg_bound}
\end{theorem}
In the worst case, we bound $P_\zeta$ by $1$ and optimize for $\zeta$
to obtain an $\otil(\sqrt{Kd\log(1/\delta)/i})$ bound after $i$
samples, where recall that $d$ is the pseudo-dimension of
$\Gcal$. This agrees with the standard generalization bound of
$O(\sqrt{\pdim(\Fcal)\log(1/\delta)/i})$ for VC-type classes because
$\Fcal = \Gcal^K$ has $O(Kd)$ statistical complexity.  However, since
the bound captures the difficulty of the CSMC problem as measured by
$P_\zeta$, we can obtain sharper results under
Assumptions~\ref{as:low_noise} and~\ref{as:tsybakov} by appropriately
setting $\zeta$.
\begin{corollary}
  \label{cor:reg_low_noise}
  Under Assumption~\ref{as:low_noise}, for any $\delta < 1/e$, with
  probability at least $1-\delta$, for all $i \in [n]$, we have
    \begin{align*}
      \EE_{x,c}[c(h_{f_{i+1}}(x)) - c(h_{\vf^\star}(x))] \leq \frac{32 K\logfactor}{i\tau}.
    \end{align*}
\end{corollary}
\begin{corollary}
\label{cor:reg_tsybakov}
Under Assumption~\ref{as:tsybakov}, for any $\delta < 1/e$, with
probability at least $1-\delta$, for all
$\frac{32\df{K}\logfactor}{\beta\tau_0^{\alpha+2}} \le i \le n$, we have
\begin{align*}
      \EE_{x,c}[c(h_{f_{i+1}}(x)) - c(h_{\vf^\star}(x))] \leq 2\beta^{\frac{1}{\alpha+2}}\left(\frac{32\df{K}\logfactor}{i}\right)^{\frac{\alpha+1}{\alpha+2}}.
\end{align*}
\end{corollary}
Thus, Massart and Tsybakov-type conditions lead to a faster
convergence rate of $\otil(1/n)$ and
$\otil(n^{-\frac{\alpha+1}{\alpha+2}})$. This agrees with the
literature on active learning for
classification~\citep{massart2006risk} and can be viewed as a
generalization to CSMC. Both generalization bounds match the optimal
rates for binary classification under the analogous low-noise
assumptions~\citep{massart2006risk,tsybakov2004optimal}. We emphasize
that \alg obtains these bounds as is, without changing any parameters,
and hence \alg is \emph{adaptive} to favorable noise conditions.


\section{Label Complexity Analysis}
Without distributional assumptions, the label complexity of \alg can
be $\order(n)$, just as in the binary classification case, since there
may always be confusing labels that force querying. In line with prior
work, we introduce two \emph{disagreement coefficients} that
characterize favorable distributional properties. We first define a set of good classifiers, the \emph{cost-sensitive regret} ball:
\begin{align*}
  \vFcsr(r) &\triangleq \left\{\vf \in \Fcal \;\Big |\; \EE\left[c(h_{\vf}(x)) -
    c(h_{\vf^\star}(x))\right] \le r \right\}.
\end{align*}
We also recall our earlier notation $\gamma(x,y,F)$ (see~\eqref{eq:max_cost} and the subsequent discussion) for a subset $F \subseteq \Fcal$ which indicates the range of expected costs for $(x,y)$ as predicted by the regressors corresponding to the classifiers in $F$. We now define the disagreement coefficients.
\begin{definition}[Disagreement coefficients]
  \label{def:dis_coeff}
  Define
  \begin{align*}
    \dis(r,y) &\triangleq \left\{x \mid \exists f,f' \in \vFcsr(r), h_f(x) =y\ne h_{f'}(x)\right\}.
  \end{align*}
  Then the disagreement coefficients are defined as:
  \begin{align*}
    \theta_1 &\triangleq \sup_{ \psi, r > 0} \frac{\psi}{r}
    \PP\left(\exists y \mid \gamma(x,y,\vFcsr(r)) > \psi \wedge x \in \dis(r,y)\right),\\
    \theta_2 &\triangleq \sup_{ \psi , r > 0} \frac{\psi}{r}
    \sum_y \PP\left(\gamma(x,y,\vFcsr(r)) > \psi \wedge x \in \dis(r,y)\right).
  \end{align*}
\end{definition}

Intuitively, the conditions in both coefficients correspond to the
checks on the \emph{domination} and \emph{cost range} of a label in
Lines~\ref{line:domination} and~\ref{line:range} of
Algorithm~\ref{alg:a_cs}. Specifically, when $x \in \dis(r,y)$, there
is confusion about whether $y$ is the optimal label or not, and hence
$y$ is not dominated. The condition on $\gamma(x,y,\vFcsr(r))$
additionally captures the fact that a small cost range provides little
information, even when $y$ is non-dominated. Collectively, the
coefficients capture the probability of an example $x$ where the good
classifiers disagree on $x$ in both predicted costs and
labels. Importantly, the notion of good classifiers is via the
algorithm-independent set $\vFcsr(r)$, and is only a property of
$\Fcal$ and the data distribution.

The definitions are a natural adaptation from binary
classification~\citep{Hanneke2014}, where a similar disagreement
region to $\dis(r,y)$ is used.  Our definition asks for confusion
about the optimality of a specific label $y$, which provides more
detailed information about the cost-structure than simply asking for
any confusion among the good classifiers. The $1/r$ scaling \df{is in agreement with previous related definitions~\citep{Hanneke2014}}, and we also
scale by the cost range parameter $\psi$, so that the favorable
settings for active learning can be concisely expressed as having
$\theta_1,\theta_2$ bounded, as opposed to a complex function of
$\psi$.

The next three results bound the labeling effort~\eqref{eq:label_effort}, in the high noise and low noise cases
respectively. The low noise assumptions enable significantly
sharper bounds. Before stating the bounds, we recall that $L_1$ corresponds to the number of examples where at least one cost is queried, while $L_2$ is the total number of costs queried across all examples.
\begin{theorem}
  \label{thm:high_noise}
  With probability at least $1-\delta$, the label complexity of the
  algorithm over $n$ examples is at most
\ifthenelse{\equal{\version}{arxiv}}{
  \begin{align*}
    L_1 = \order\left(n\theta_1\sqrt{K\logfactor}+ \log(1/\delta)\right), \qquad
    L_2 = \order\left(n\theta_2\sqrt{K\logfactor} + K\log(1/\delta)\right).
  \end{align*}
}{
  \begin{align*}
    L_1 &= \order\left(n\theta_1\sqrt{K\logfactor}+ \log(1/\delta)\right),\\
    L_2 &= \order\left(n\theta_2\sqrt{K\logfactor} + K\log(1/\delta)\right).
  \end{align*}
}
\end{theorem}

\begin{theorem}
  \label{thm:massart}
  Assume the Massart noise condition holds.  With probability at least
  $1-\delta$ the label complexity of the algorithm over $n$ examples
  is at most
  \begin{align*}
    L_1 & = \order\left(\frac{K\log(n)\logfactor}{\tau^2}\theta_1 + \log(1/\delta)\right), \df{\qquad L_2 = \order\left(KL_1\right)} 
  \end{align*}
\end{theorem}

\begin{theorem}
  \label{thm:tsybakov}
  Assume the Tsybakov noise condition holds.  With probability at least
  $1-\delta$ the label complexity of the algorithm over $n$ examples
  is at most
  \begin{align*}
    L_1 &= \order\left(\theta_1^{\frac{\alpha}{\alpha+1}}(K\logfactor)^{\frac{\alpha}{\alpha+2}}n^{\frac{2}{\alpha+2}} + \log(1/\delta)\right), \df{\qquad L_2 = \order\left(KL_1\right)} 
  \end{align*}
\end{theorem}

In the high-noise case, the bounds scales with $n\theta$ for the
respective coefficients. \df{In comparison, for binary classification
  the leading term is $\otil\rbr{n\theta \textrm{error}(h_{f^\star})}$
  which involves a different disagreement coefficient and which scales
  with the error of the optimal classifier
  $h_{f^\star}$~\citep{Hanneke2014,huang2015efficient}. Qualitatively
  the bounds have similar worst-case behavior, demonstrating minimal
  improvement over passive learning, but by scaling with
  $\textrm{error}(h_{f^\star})$ the binary classification bound
  reflects improvements on benign instances. For the special case of multiclass classification, we
  are able to recover the dependence on $\textrm{error}(h_{f^\star})$ and the standard disagreement coefficient
  with a simple modification to our proof, which we discuss in detail
  in the next subsection.}


On the other hand, in both low noise cases the label complexity scales
sublinearly with $n$\df{. With bounded disagreement coefficients, this
  improves over the standard passive learning analysis where all
  labels are queried on $n$ examples to achieve the generalization
  guarantees in Theorem~\ref{thm:reg_bound},
  Corollary~\ref{cor:reg_low_noise}, and
  Corollary~\ref{cor:reg_tsybakov} respectively.}
In
particular, under the Massart condition, both $L_1$ and $L_2$ bounds
scale with $\theta \log(n)$ for the respective disagreement
coefficients, which is an \emph{exponential improvement} over \df{the }passive
learning \df{ analysis}. Under the milder Tsybakov condition, the bounds scale with
$\theta^{\frac{\alpha}{\alpha+1}}n^{\frac{2}{\alpha+2}}$, which
improves polynomially over passive learning. These label complexity
bounds agree with analogous results from binary
classification~\citep{Castro2008,Hanneke2014,hanneke2015minimax} in
their dependence on $n$. 

Note that $\theta_2 \df{ \leq K\theta_1}$ \df{always and it can be much smaller}, as
demonstrated through an example in the next section. In such cases,
only a few labels are ever queried and the $L_2$ bound in the high
noise case reflects this additional savings over passive
learning. Unfortunately, in low noise conditions, 
\dfc{the $L_2$ bound
depends directly on $K\theta_1$, so that }
we do not benefit when
$\theta_2 \ll K\theta_1$. This can be resolved by letting $\psi_i$ in the algorithm
depend on the noise level $\tau$, but we prefer to use the more robust
choice $\psi_i = 1/\sqrt{i}$ which still allows \alg to partially
adapt to low noise and achieve low label complexity.

The main improvement over~\citet{krishnamurthy2017active} is
demonstrated in the label complexity bounds under low noise
assumptions. For example, under Massart noise, our bound has the
optimal $\log(n)/\tau^2$ rate, while the bound
in~\citet{krishnamurthy2017active} is exponentially worse, scaling
with $n^\beta/\tau^2$ for $\beta \in (0,1)$. This improvement comes
from explicitly enforcing monotonicity of the version space, so that
once a regressor is eliminated it can never force \alg to query
again. Algorithmically, computing the maximum and minimum costs with
the monotonicity constraint is much more challenging and requires the
new subroutine using MW.

\subsection{\df{Recovering Hanneke's Disagreement Coefficient}}
\label{ssec:comparisons}
\df{In this subsection we show that in many cases we can obtain
  guarantees in terms of Hanneke's disagreement
  coefficient~\citep{Hanneke2014}, which has been used extensively in
  active learning for binary classification. We also show that, for
  multiclass classification, the label complexity scales with the
  error of the optimal classifier $h^\star$, a refinement on
  Theorem~\ref{thm:high_noise}. The guarantees require no
  modifications to the algorithm and enable a precise comparison with
  prior results. Unfortunately, they do not apply to the general CSMC
  setting, so they have not been incorporated into our main theorems.}


\df{We start with defining Hanneke's disagreement
  coefficient~\citep{Hanneke2014}. Define the \emph{disagreement ball}
  $\vFmin(r) \triangleq \{f \in \Fcal: \PP[h_f(x) \ne
    h_{f^\star}(x)]\leq r\}$ and the disagreement region $\dismin(r)
  \triangleq \{x \mid \exists f, f' \in \vFmin(r), h_f(x)\ne
  h_{f'}(x)\}$. The coefficient is defined as
\begin{align}
\theta_0 \triangleq \sup_{r > 0} \frac{1}{r} \PP\sbr{x \in \dismin(r)}. \label{eq:bc_dis}
\end{align}
This coefficient is known to be $O(1)$ in many cases, for example when
the hypothesis class consists of linear separators and the marginal
distribution is uniform over the unit sphere~\citep[Chapter
  7]{Hanneke2014}. In comparison with Definition~\ref{def:dis_coeff},
the two differences are that $\theta_1,\theta_2$ include the
cost-range condition and involve the cost-sensitive regret ball
$\vFcsr(r)$ rather than $\vFmin(r)$. As $\vFmin(r) \subset \vFcsr(r)$,
we expect that $\theta_1$ and $\theta_2$ are typically larger than
$\theta_0$, so bounds in terms of $\theta_0$ are more desirable. We
now show that such guarantees are possible in many cases.}

\paragraph{\df{The low noise case.}}
\df{For general CSMC, low noise conditions admit the following:
\begin{proposition}
\label{prop:bc_disagreement}
Under Massart noise, with probability at least $1-\delta$ the label
complexity of the algorithm over $n$ examples is at most $L_1 =
\order\rbr{\frac{\log(n)\nu_n}{\tau^2}\theta_0 +
  \log(1/\delta)}$. Under Tsybakov noise, the label complexity is at
most $L_1 = \order\rbr{\theta_0 n^{\frac{2}{\alpha+2}} \rbr{\log(n)
    \nu_n}^{\frac{\alpha}{\alpha+2}} + \log(1/\delta)}$. In both cases
we have $L_2 = O(KL_1)$.
\end{proposition}}

\df{That is, for any low noise CSMC problem, \alg obtains a label
  complexity bound in terms of Hanneke's disagreement coefficient
  $\theta_0$ directly. Note that this adaptivity requires no change to
  the algorithm. Proposition~\ref{prop:bc_disagreement} enables a
  precise comparison with disagreement-based active learning for
  binary classification. In particular, this bound matches the
  guarantee for \textsc{CAL}~\citep[Theorem 5.4]{Hanneke2014} with the
  caveat that our measure of statistical complexity is the
  pseudodimension of the $\Fcal$ instead of the VC-dimension of the
  hypothesis class. As a consequence, under low noise assumptions,
  \alg has favorable label complexity in all examples where $\theta_0$
  is small. }

\paragraph{\df{The high noise case.}}
\df{Outside of the low noise setting, we can introduce $\theta_0$ into
  our bounds, but only for multiclass classification, where we always
  have $c \triangleq \one - e_y$ for some $y \in [K]$. Note that $f(x; y)$ is now interpreted as a prediction for $1 - P(y | x)$, so that the least cost prediction $y^\star(x)$ corresponds to the most likely label.  We also obtain
  a further refinement by introducing $\textrm{error}(h_{f^\star})
  \triangleq \EE_{(x,c)} [c(h_{f^\star}(x))]$.
\begin{proposition}
\label{prop:bc_dis_high_noise}
For multiclass classification, with probability at least $1-\delta$,
the label complexity of the algorithm over $n$ examples is at most
\begin{align*}
L_1 = 4\theta_0 n \cdot \textrm{error}(h_{f^\star}) + \order\rbr{ \theta_0 \rbr{\sqrt{Kn\logfactor \cdot \textrm{error}(h_{f^\star})} + K\regconst\logfactor\log(n)} + \log(1/\delta)}.
\end{align*}
\end{proposition}}

\df{This result exploits two properties of the multiclass cost
  structure. First we can relate $\vFcsr(r)$ to the disagreement ball
  $\vFmin(r)$, which lets us introduce Hanneke's disagreement
  coefficient $\theta_0$. Second, we can bound $P_\zeta$ in
  Theorem~\ref{thm:reg_bound} in terms of
  $\textrm{error}(h_{f^\star})$. Together the bound is comparable to
  prior results for active learning in binary
  classification~\citep{hsu2010algorithms,hanneke2012surrogate,Hanneke2014}, with a
  slight generalization to the multiclass setting. Unfortunately, both
  of these refinements do not apply for general CSMC.}

\df{\paragraph{Summary.} In important special cases, \alg achieves
  label complexity bounds directly comparable with results for active
  learning in binary classification, scaling with $\theta_0$ and
  $\textrm{error}(h_{f^\star})$. In such cases, whenever $\theta_0$ is
  bounded --- for which many examples are known --- \alg has favorable
  label complexity. However, in general CSMC without low-noise
  assumptions, we are not able to obtain a bound in terms of these
  quantities, and we believe a bound involving $\theta_0$ does not
  hold for \alg. We leave understanding natural settings where
  $\theta_1$ and $\theta_2$ are small, or obtaining sharper guarantees
  as intriguing future directions.}

\subsection{\df{Three} Examples}
\label{sec:examples}
We now describe \df{three} examples to give more intuition for \alg and
our label complexity bounds.  Even in the low noise case, our label
complexity analysis does not demonstrate all of the potential benefits
of our query rule. In this section we give \df{three} examples to
further demonstrate these advantages.

Our first example shows the benefits of using the domination criterion
in querying, in addition to the cost range condition. Consider a
problem under Assumption~\ref{as:low_noise}, where the optimal cost is
predicted perfectly, the second best cost is $\tau$ worse and all the
other costs are substantially worse, but with variability in the
predictions. Since all classifiers predict the correct label, we get
$\theta_1 = \theta_2 = 0$, so our label complexity bound is
$\order(1)$. Intuitively, since every regressor is certain of the
optimal label and its cost, we actually make zero queries. On the
other hand, all of the suboptimal labels have large cost ranges, so
querying based solely on a cost range criteria, as would happen with
an active regression algorithm~\citep{castro2005faster}, leads to a
large label complexity.

A related example demonstrates the improvement in our query rule over
more na\"ive approaches where we query either no label or all labels,
which is the natural generalization of query rules from multiclass
classification~\citep{agarwal2013selective}.  In the above example, if
the best and second best labels are confused occasionally $\theta_1$
may be large, but we expect $\theta_2 \ll K\theta_1$ since no other
label can be confused with the best.  Thus, the $L_2$ bound in
Theorem~\ref{thm:high_noise} is a factor of $K$ smaller than with a
na\"{i}ve query rule since \alg only queries the best and second best
labels. Unfortunately, without setting $\psi_i$ as a function of the
noise parameters, the bounds in the low noise cases do not reflect
this behavior.

\df{The third example shows that both $\theta_0$ and $\theta_1$ yield
  pessimistic bounds on the label complexity of \alg in some
  cases. The example is more involved, so we describe it in detail. We
  focus on statistical issues, using a finite regressor class
  $\Fcal$. Note that our results on generalization and label
  complexity hold in this setting, replacing $d \log (n)$ with $\log
  |\Fcal|$, and the algorithm can be implemented by enumerating
  $\Fcal$. Throughout this example, we use $\tilde{O}(\cdot)$ to further
  suppress logarithmic dependence on $n$.}

\df{Let $\Xcal \triangleq \{x_1,\ldots,x_M\}$, $\Ycal \triangleq \{0,1\}$, and consider
  functions $\Fcal \triangleq \{f^\star, f_1,\ldots,f_M\}$. We have $f^\star(x)
  \triangleq (1/4,1/2), \forall x \in \Xcal$ and $f_i(x_i) \triangleq (1/4,0)$ and
  $f_i(x_j) \triangleq (1/4,1)$ for $i \ne j$. The marginal distribution is
  uniform and the true expected costs are given by $f^\star$ so that
  the problem satisfies the Massart noise condition with
  $\tau=1/4$. The key to the construction is that $f_i$s have high
  square loss on labels that they do not predict.}

\df{Observe that as $\PP[h_{f_i}(x) \ne h_{f^\star}(x)] = 1/M$ and
  $h_{f_i}(x_i) \ne h_{f^\star}(x_i)$ for all $i$, the probability of
  disagreement is $1$ until all $f_i$ are eliminated. As such, we have
  $\theta_0 = M$. Similarly, we have $\EE [c(h_{f_i}(x)) -
  c(h_{f^\star}(x))] = \frac{1}{4M}$ and
  $\gamma(x,1,\vFcsr(\frac{1}{4M})) = 1$, so $\theta_1 =
  4M$. Therefore, the bounds in Theorem~\ref{thm:massart} and
  Proposition~\ref{prop:bc_disagreement} are both
  $\tilde{O}(M\log|\Fcal|) = \tilde{O}(|\Fcal|)$. On the other hand,
  since $(f_i(x_j,1) - f^\star(x_j,1))^2 = 1/4$ for all $i,j \in [M]$,
  \alg eliminates every $f_i$ once it has made a total of
  $\tilde{O}(\log |\Fcal|)$ queries to label $y=1$. Thus the label
  complexity is actually just $\tilde{O}(\log |\Fcal|)$, which is
  exponentially better than the disagreement-based analyses. Thus,
  \alg can perform much better than suggested by the
  disagreement-based analyses, and an interesting future direction is
  to obtain refined guarantees for cost-sensitive active learning. }

\section{Experiments}
\label{sec:experiments}
We now turn to an empirical evaluation of \alg. For further
computational efficiency, we implemented an approximate version of
\alg using\df{: 1) a relaxed version space
$\Gcal_i(y) \gets \{ g \in \Gcal \mid \Rhat_i(g;y) \leq \Rhat_i(g_{i,y};y) + \Delta_{i}\}$, which
does not enforce monotonicity, and 2)} {\em online optimization}, based on online linear
least-squares regression. The algorithm processes the data in one
pass, and the idea is to (1) replace $g_{i,y}$, the ERM, with an
approximation $g_{i,y}^o$ obtained by online updates, and (2) compute
the minimum and maximum costs via a sensitivity analysis of the online
update. We describe this algorithm in detail in
Subsection~\ref{ssec:online_approx}. Then, we present our experimental
results, first for simulated active learning
(Subsection~\ref{ssec:active_experiments}) and then for learning to
search for joint prediction (Subsection~\ref{ssec:l2s_experiments}).

\subsection{Finding Cost Ranges with Online Approximation}
\label{ssec:online_approx}
Consider the maximum and minimum costs for a fixed example $x$ and
label $y$ at round $i$, all of which may be suppressed.
\df{We ignore all the constraints
on the empirical square losses for the past rounds.}
First, define
$\Rhat(g,w,c;y) \triangleq \Rhat(g;y) + w(g(x)-c)^2$, which is the
risk functional augmented with a fake example with weight $w$ and cost
$c$. Also define
\begin{align*}
  \underline{g}_{w} \triangleq \arg \min_{g \in \Gcal} \Rhat(g,w,0;y), \qquad  \overline{g}_{w} \triangleq \arg \min_{g \in \Gcal} \Rhat(g,w,1;y),
\end{align*}
and recall that $g_{i,y}$ is the ERM given in
Algorithm~\ref{alg:a_cs}.  The functional $\Rhat(g,w,c;y)$ has a
monotonicity property that we exploit here, proved in
Appendix~\ref{app:lemmata}.
\begin{lemma}
\label{lemma:monotone}
For any $c$ and for $w' \ge w \ge 0$, define
$g=\argmin_{g}\Rhat(g,w,c)$ and $g' = \argmin_{g}
\Rhat(g,w',c)$. Then
\begin{align*}
\Rhat(g')\ge \Rhat(g) \textrm{ and } (g'(x) -c)^2 \le (g(x) - c)^2.
\end{align*}
\end{lemma}
As a result, an alternative to \mincost and \maxcost is to find
\begin{align}
  \underline{w} & \triangleq \max \{ w \mid \Rhat(\underline{g}_w)  - \Rhat(g_{i,y}) \leq \Delta_i\},   \label{eq:barw}\\
  \overline{w} & \triangleq \max \{ w \mid \Rhat(\overline{g}_w)  - \Rhat(g_{i,y}) \leq \Delta_i\}, \label{eq:bbarw}
\end{align}
and return $\underline{g}_{\underline{w}}(x)$ and
$\overline{g}_{\overline{w}}(x)$ as the minimum and maximum costs.  We
use two steps of approximation here.  Using the definition of
$\overline{g}_w$ and $\underline{g}_w$ as the minimizers of
$\Rhat(g,w,1;y)$ and $\Rhat(g,w,0;y)$ respectively, we have
\begin{align*}
  \Rhat(\underline{g}_w) - \Rhat(g_{i,y}) \leq w\cdot g_{i,y}(x)^2 - w\cdot \underline{g}_{w}(x)^2, \\
    \Rhat(\overline{g}_w) - \Rhat(g_{i,y}) \leq w\cdot (g_{i,y}(x)-1)^2 - w\cdot (\overline{g}_{w}(x)-1)^2.
\end{align*}
We use this upper bound in place of $\Rhat(g_w) - \Rhat(g_{i,y})$
in~\eqref{eq:barw} and \eqref{eq:bbarw}.  Second, we replace
$g_{i,y}$, $\underline{g}_{w}$, and $\overline{g}_{w}$ with
approximations obtained by online updates.  More specifically, we
replace $g_{i,y}$ with $g_{i,y}^o$, the current regressor produced by
all online linear least squares updates so far, and approximate the
others by
\begin{align*}
  \underline{g}_{w}(x) \approx g_{i,y}^o(x) -  w \cdot s(x,0,g_{i,y}^o), \qquad
  \overline{g}_{w}(x) \approx g_{i,y}^o(x) +  w \cdot s(x,1,g_{i,y}^o),
\end{align*}
where $s(x,y,g_{i,y}^o) \geq 0$ is a {\em sensitivity} value that
approximates the change in prediction on $x$ resulting from an online
update to $g_{i,y}^o$ with features $x$ and label $y$.  The
computation of this sensitivity value is governed by the actual online
update where we compute the derivative of the change in the prediction
as a function of the importance weight $w$ for a hypothetical example
with cost $0$ or cost $1$ and the same features.  This is possible for
essentially all online update rules on importance weighted examples,
and it corresponds to taking the limit as $w \rightarrow 0$ of the
change in prediction due to an update, divided by $w$. Since we are
using linear representations, this requires only $\order(s)$ time per
example, where $s$ is the average number of non-zero features.  With
these two steps, we obtain approximate minimum and maximum costs using
\begin{align*}
  g_{i,y}^o(x) - \underline{w}^o \cdot s(x, 0, g_{i,y}^o), \qquad
  g_{i,y}^o(x) + \overline{w}^o \cdot s(x, 1, g_{i,y}^o),
\end{align*}
where
\begin{align*}
  \underline{w}^o \;\triangleq\; \max \{w \mid w \left( g_{i,y,}^o(x)^2 - (g_{i,y}^o(x) - w \cdot s(x,0,g_{i,y}^o))^2\right) \leq \Delta_i\} \\
  \overline{w}^o \;\triangleq\; \max \{w \mid w \left( (g_{i,y,}^o(x)-1)^2 - (g_{i,y}^o(x) + w \cdot s(x,1,g_{i,y}^o)-1)^2\right) \leq \Delta_i\}.
\end{align*}
The online update guarantees that $g_{i,y}^o(x) \in [0,1]$. Since the
minimum cost is lower bounded by 0, we have $\underline{w}^o\in
\left(0,\frac{g_{i,y}^o(x)}{s(x,0,g_{i,y}^o)}\right]$. Finally, because
  the objective $w(g_{i,y}^o(x))^2 - w(g_{i,y}^o(x) - w \cdot
  s(x,0,g_{i,y}^o))^2$ is increasing in $w$ within this range (which can
    be seen by inspecting the derivative), we can find
    $\underline{w}^o$ with binary search. Using the same techniques,
    we also obtain an approximate maximum cost.

    \dfc{It is worth noting that the approximate cost ranges, without the
    sensitivity trick, are contained in the exact cost ranges because
    we approximate the difference in squared error by an {\em upper
      bound}. Hence, the query rule in this online algorithm should be
    more aggressive than the query rule in Algorithm \ref{alg:a_cs}.}

\subsection{Simulated Active Learning}
  \label{ssec:active_experiments}
  We performed simulated active learning experiments with three
  datasets.  ImageNet 20 and 40 are sub-trees of the ImageNet
  hierarchy covering the 20 and 40 most frequent classes, where each
  example has a single zero-cost label, and the cost for an incorrect
  label is the tree-distance to the correct one.  The feature vectors
  are the top layer of the Inception neural
  network~\citep{Inception15}.  The third,
  RCV1-v2~\citep{lewis2004rcv1}, is a multilabel text-categorization
  dataset, which has 103 labels, organized as a tree with a
  similar tree-distance cost structure as the ImageNet data.  Some
  dataset statistics are in Table~\ref{tab:stats}.

\begin{figure*}[t]
\begin{center}
  \includegraphics[width=0.32\textwidth,height=0.25\textwidth]{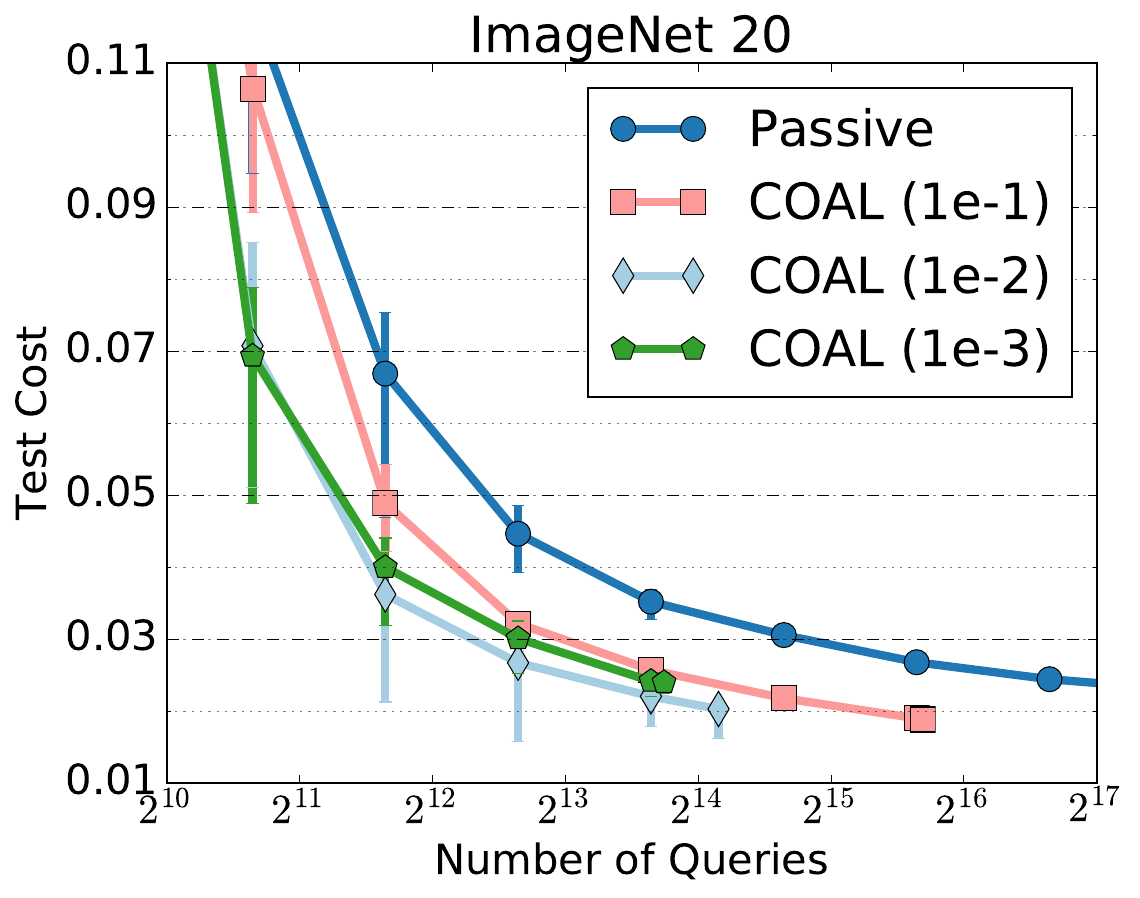}
  \includegraphics[width=0.32\textwidth,height=0.25\textwidth]{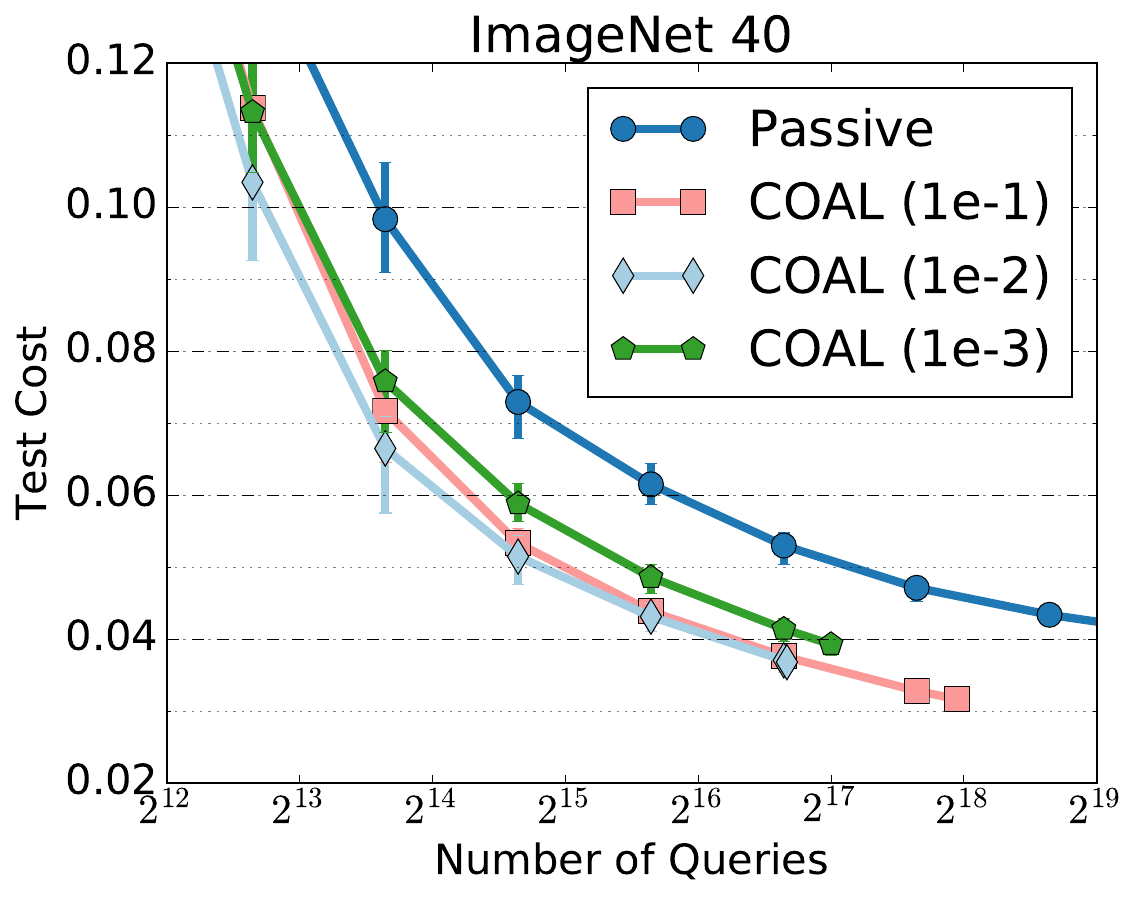}
  \includegraphics[width=0.32\textwidth,height=0.25\textwidth]{./rcv1_error.pdf}
  \\
  \includegraphics[width=0.32\textwidth,height=0.25\textwidth]{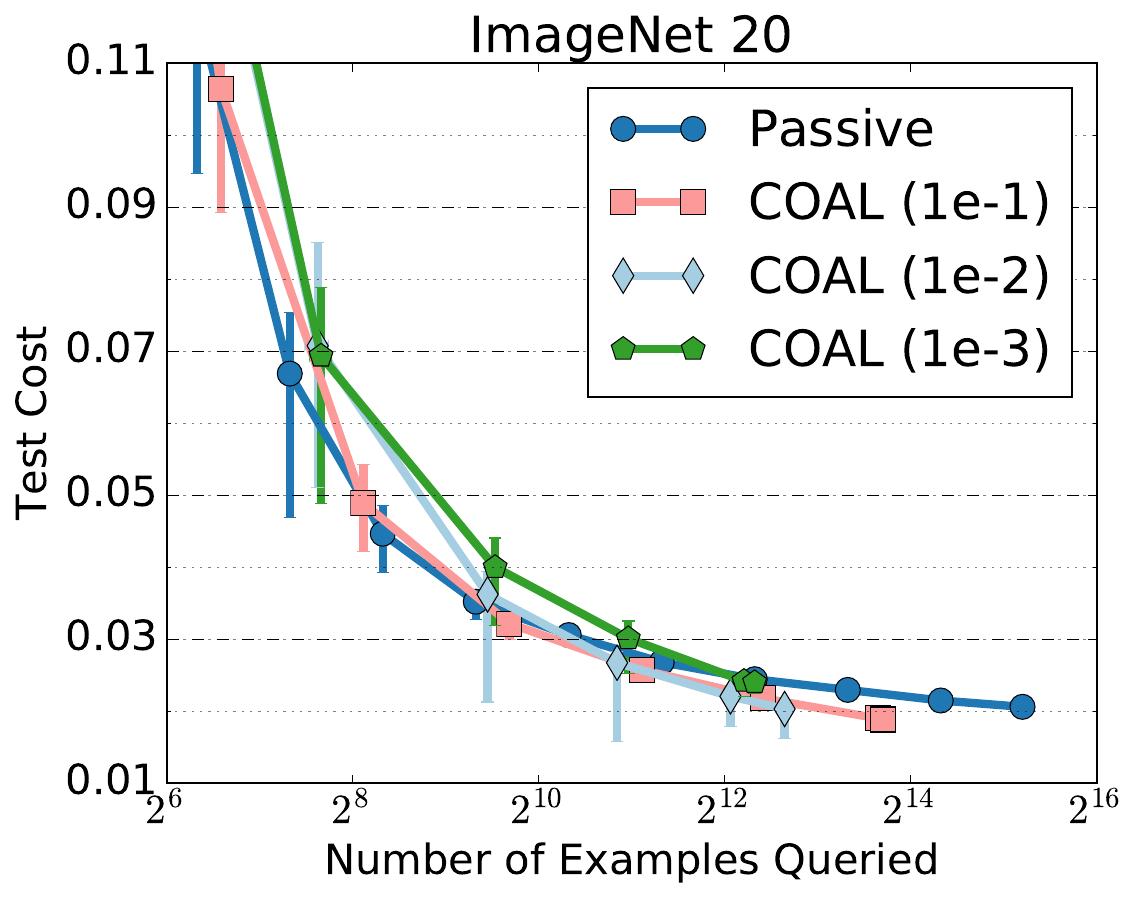}
  \includegraphics[width=0.32\textwidth,height=0.25\textwidth]{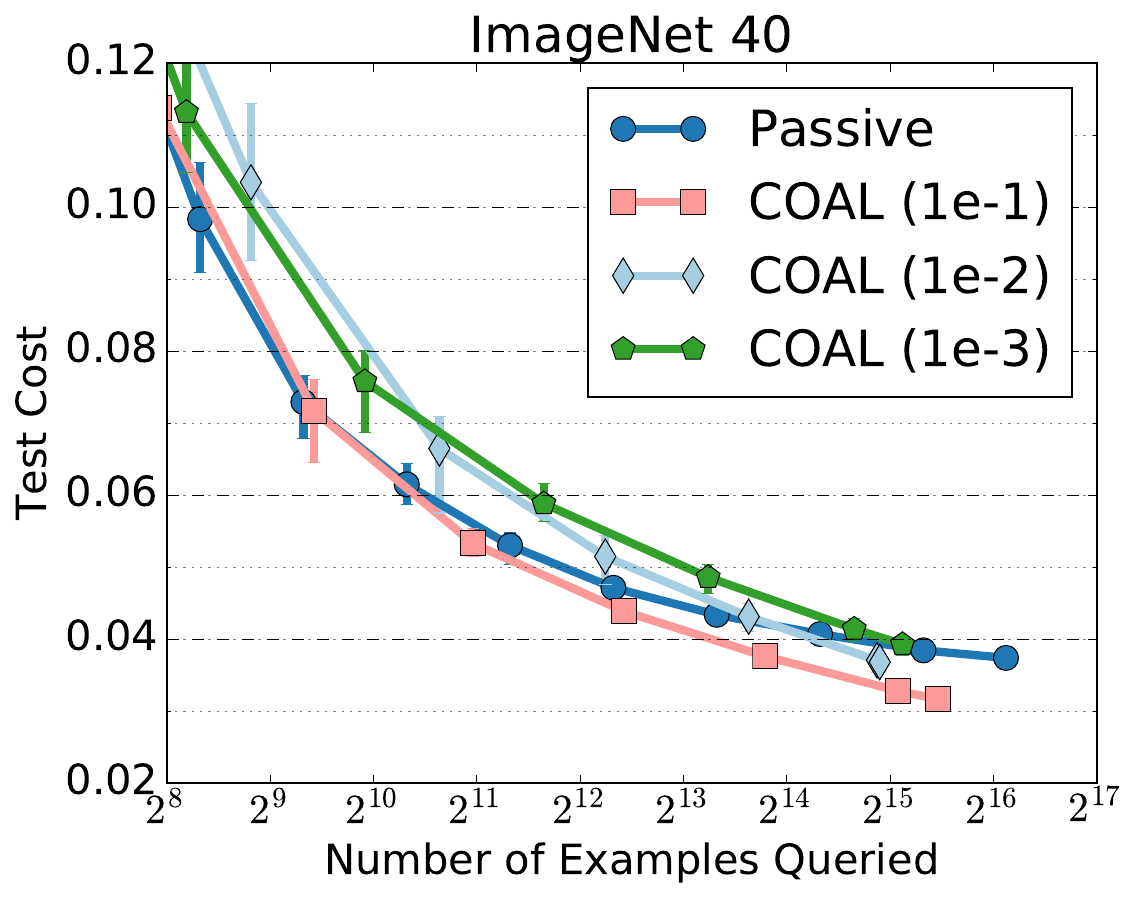}
  \includegraphics[width=0.32\textwidth,height=0.25\textwidth]{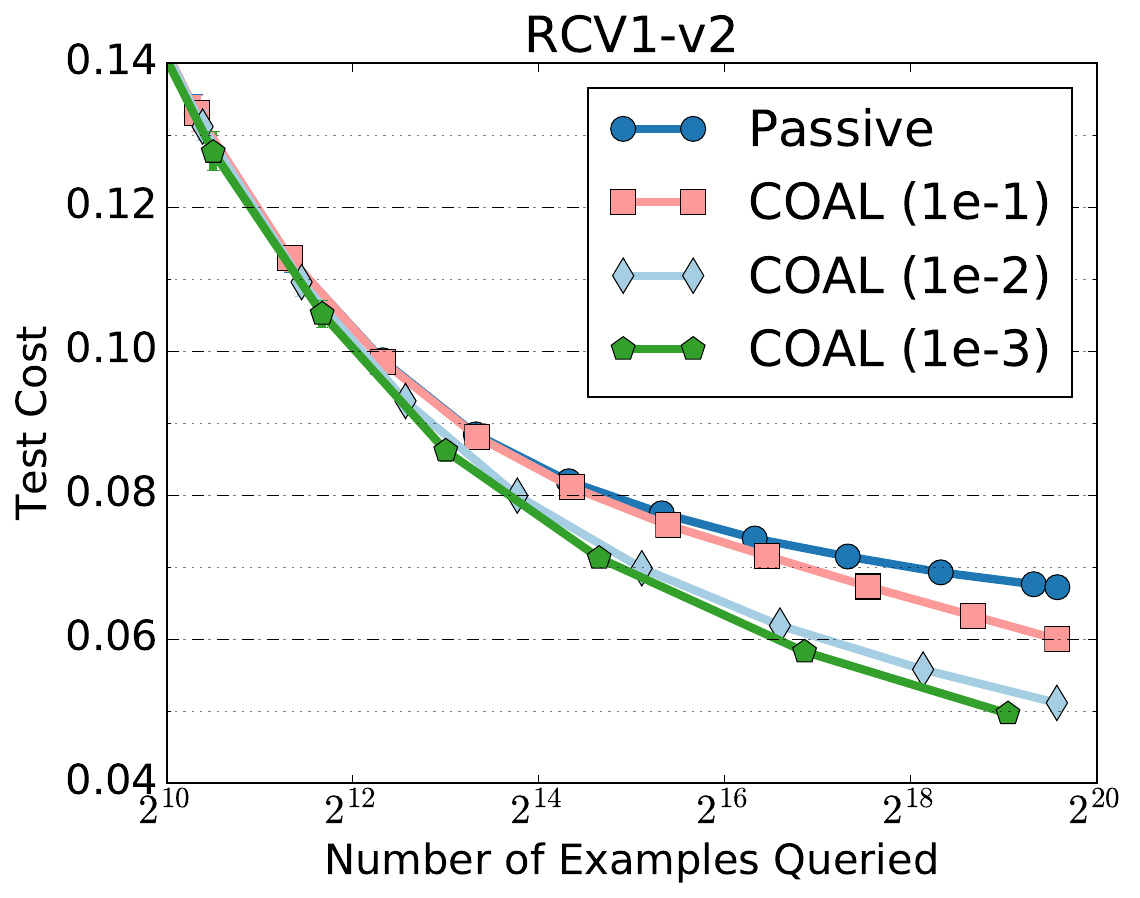}
  \caption{Experiments with \alg. Top row shows test cost vs. number of
    queries for simulated active learning experiments. Bottom row shows
    test cost vs. number of examples with even a single label query for simulated active learning experiments. }
  \label{fig:simulated_results}
\end{center}
\end{figure*}

\begin{table}
    \centering
    \begin{tabular}{l@{~~}r@{~~}r@{~~}r@{~~}r@{~~}} \toprule
      & $K$ & $n$~~ & feat & density \\ \midrule
      ImageNet 20 & 20 & 38k &  6k & 21.1\% \\
      ImageNet 40 & 40 & 71k &  6k & 21.0\% \\
      RCV1-v2     &103 & 781k & 47k & 0.16\% \\
      \bottomrule
    \end{tabular}
    \begin{tabular}{l@{~~}r@{~~}r@{~~}r@{~~}r} \toprule
      & $K$ & $n$~~ & feat & $\length$ \\ \midrule
      POS & 45 & 38k & 40k & 24 \\
      NER & 9 & 15k & 15k & 14 \\
      Wiki & 9 & 132k & 89k & 25 \\
      \bottomrule
    \end{tabular}
    \caption{Dataset statistics. $\length$ is the average sequence length
      and density is the percentage of non-zero features.}
    \label{tab:stats}
  \end{table}

We compare our online version of \alg to passive online learning.  We
use the cost-sensitive one-against-all (\csoaa) implementation in
Vowpal Wabbit\footnote{\url{http://hunch.net/~vw}}, which performs
online linear regression for each label separately.  There are two
tuning parameters in our implementation.  First, instead of
$\Delta_i$, we set the radius of the version space to $\Delta_i' =
\frac{\kappa \logfactor[i-1]}{i-1}$ (i.e. the $\log(n)$ term in the
definition of $\logfactor$ is replaced with $\log(i)$) and instead
tune the constant $\kappa$. This alternate ``mellowness" parameter
controls how aggressive the query strategy is. The second parameter is
the learning rate used by online linear regression\footnote{We use the
  default online learning algorithm in Vowpal Wabbit, which is a
  scale-free~\citep{normalized} importance weight
  invariant~\citep{invariant} form of AdaGrad~\citep{adagrad}.}.

For all experiments, we show the results obtained by the best learning
rate for each mellowness on each dataset, which is tuned as follows.
We randomly permute the training data 100 times and make one pass
through the training set with each parameter setting. For each dataset
let $\perf(\mel,l,q,t)$ denote the test performance of the algorithm
using mellowness $\mel$ and learning rate $l$ on the $t^{\textrm{th}}$
permutation of the training data under a query budget of $2^{(q-1)}
\cdot 10 \cdot K, q \geq 1$.  Let $\query(\mel,l,q,t)$ denote the number
of queries actually made. Note that $\query(\mel,l,q,t) < 2^{(q-1)} \cdot
10 \cdot K$ if the algorithm runs out of the training data before
reaching the $q^{\textrm{th}}$ query budget\footnote{In fact, we check
  the test performance only in between examples, so $\query(\mel,l,q,t)$
  may be larger than $2^{(q-1)}\cdot 10 \cdot K$ by an additive factor
  of $K$, which is negligibly small.}.  To evaluate the trade-off
between test performance and number of queries, we define the
following performance measure:
  \begin{equation}
    \auc(\mel,l,t) = \frac{1}{2} \sum_{q = 1}^{q_{\max}} \Big(\perf(\mel,l,q+1,t) +  \perf(\mel,l,q,t)\Big) \cdot \log_2 \frac{\query(\mel,l,q+1,t)}{\query(\mel,l,q,t)},\label{eq:auc}
  \end{equation}
where $q_{\max}$ is the minimum $q$ such that $2^{(q-1)}\cdot 10$ is
larger than the size of the training data.  This performance measure
is the area under the curve of test performance against number of
queries in $\log_2$ scale.  A large value means the test performance
quickly improves with the number of queries. The best learning rate
for mellowness $\mel$ is then chosen as
  \begin{equation*}
    l^\star(\mel) \triangleq \arg \max_{l} \median_{1 \leq t \leq 100} \quad \auc(\mel,l,t).
  \end{equation*}
  The best learning rates for different datasets and mellowness settings are in Table~\ref{tbl:learning_rates}.

  \begin{table}[t]
    \centering
    \begin{tabular}{l | c | c | c | c | c | c |}
      & ImageNet 20 & ImageNet 40 & RCV1-v2 & POS & NER & NER-wiki\\ \hline
      passive            & 1           & 1           & 0.5     & 1.0 & 0.5 & 0.5 \\ \hline
      active ($10^{-1}$) & 0.05        & 0.1         & 0.5     & 1.0 & 0.1 & 0.5\\ \hline
      active ($10^{-2}$) & 0.05        & 0.5         & 0.5     & 1.0 & 0.5 & 0.5 \\ \hline
      active ($10^{-3}$) & 1           & 10          & 0.5     & 10  & 0.5 & 0.5
    \end{tabular}
    \caption{Best learning rates for each learning algorithm and each dataset.}
    \label{tbl:learning_rates}
  \end{table}

In the top row of Figure~\ref{fig:simulated_results}, we plot, for each
dataset and mellowness, the number of queries against the median test
cost along with bars extending from the $15^{\textrm{th}}$ to
$85^{\textrm{th}}$ quantile.  Overall, \alg achieves a better
trade-off between performance and queries.  With proper mellowness
parameter, active learning achieves similar test cost as passive
learning with a factor of 8 to 32 fewer queries.  On ImageNet 40 and
RCV1-v2 (reproduced in Figure~\ref{fig:rcv1_results}), active learning
achieves \emph{better} test cost with a factor of 16 fewer queries.  On
RCV1-v2, \alg queries like passive up to around $256k$ queries, since
the data is very sparse, and linear regression has the property that
the cost range is maximal when an example has a new unseen feature.
Once \alg sees all features a few times, it queries much more
efficiently than passive. These plots correspond to the label
complexity $L_2$.

In the bottom row, we plot the test error as a function of the number
of examples for which at least one query was requested, for each
dataset and mellowness, which experimentally corresponds to the $L_1$
label complexity. In comparison to the top row, the improvements
offered by active learning are slightly less dramatic here. This
suggests that our algorithm queries just a few labels for each
example, but does end up issuing at least one query on most of the
examples. Nevertheless, one can still achieve test cost competitive
with passive learning using a factor of 2-16 less labeling effort, as
measured by $L_1$.

\begin{figure*}[t]
\begin{center}
  \includegraphics[width=0.32\textwidth,height=0.25\textwidth]{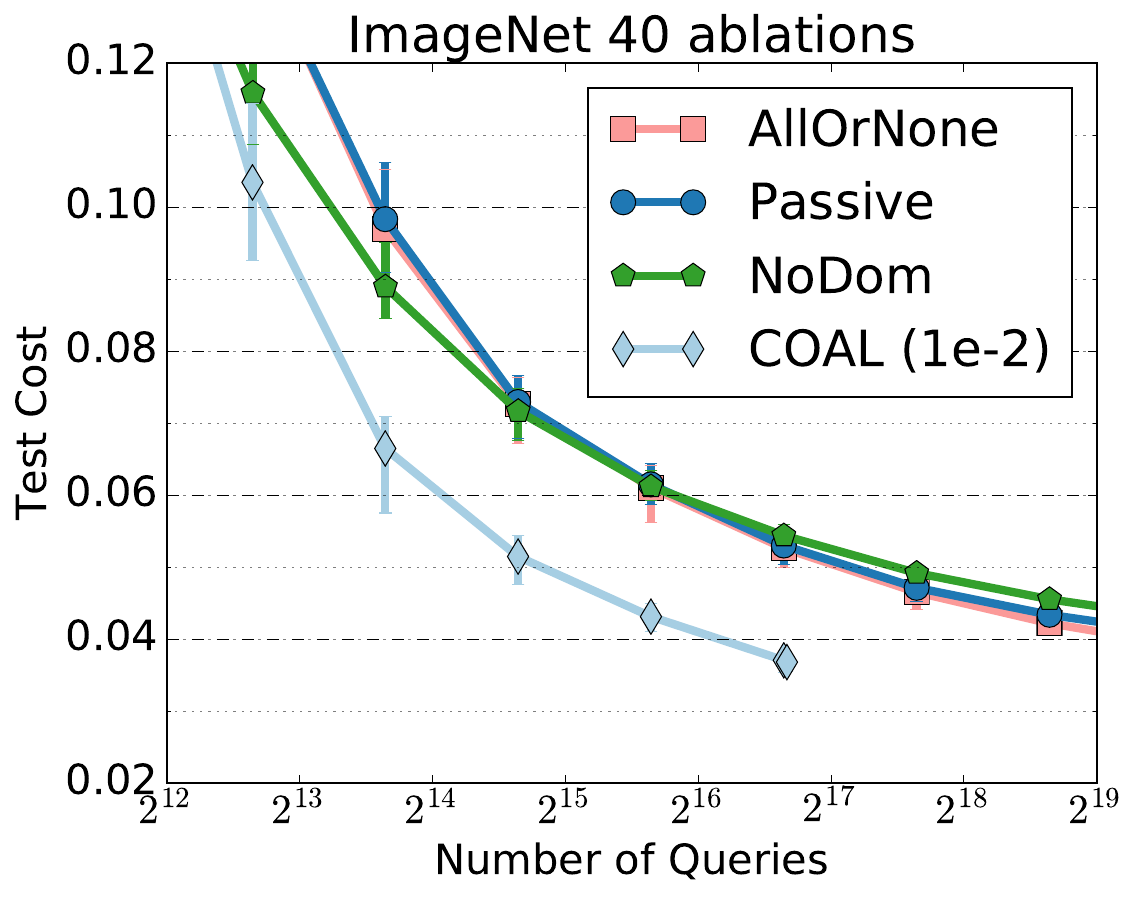}
  \includegraphics[width=0.32\textwidth,height=0.25\textwidth]{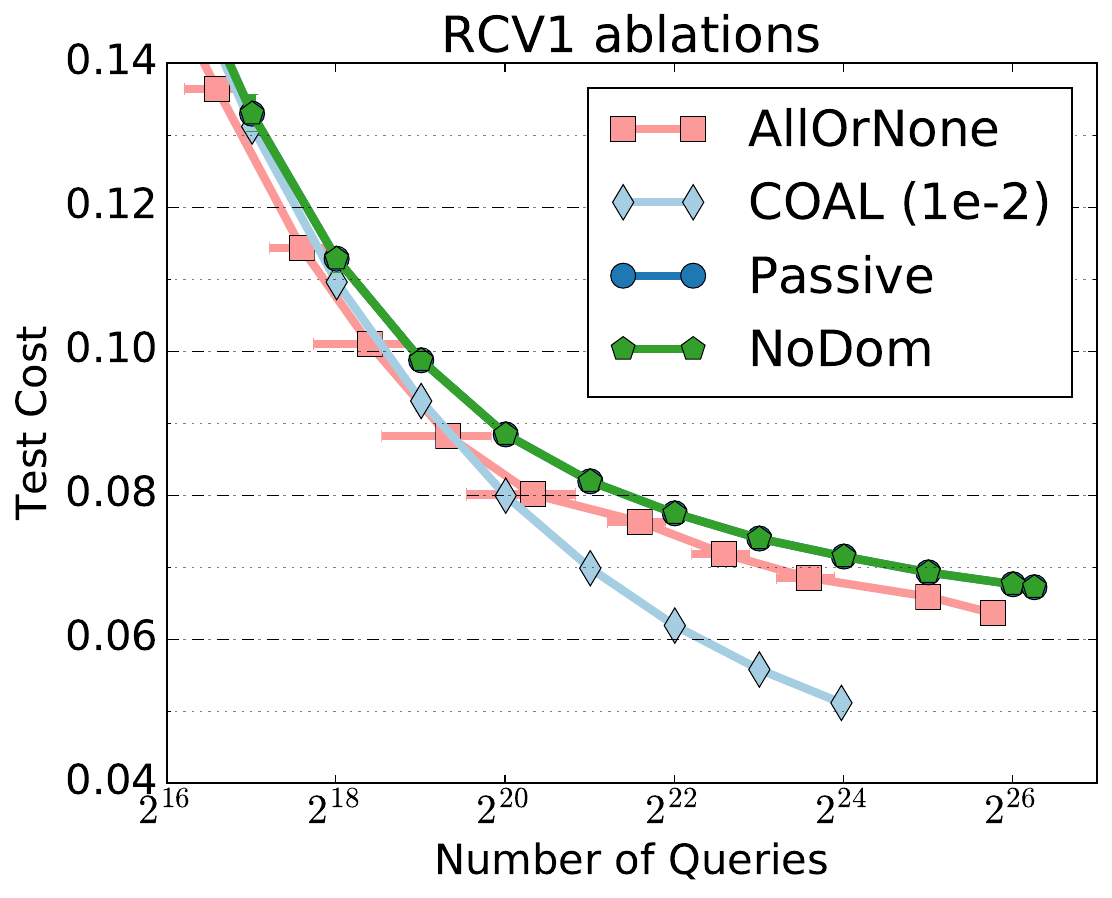}
\end{center}
\caption{Test cost versus number of queries for \alg, in comparison
  with active and passive baselines on the ImageNet40 and RCV1-v2
  dataset. On RCV1-v2, passive learning and \textsc{NoDom} are nearly
  identical.}
\label{fig:simulated_ablations}
\end{figure*}

We also compare \alg with two active learning baselines. Both
algorithms differ from \alg only in their query
rule. \textsc{AllOrNone} queries either all labels or no labels using
both domination and cost-range conditions and is an adaptation of
existing multiclass active
learners~\citep{agarwal2013selective}. \textsc{NoDom} just uses the
cost-range condition, inspired by active
regression~\citep{castro2005faster}. The results for ImageNet 40 and
RCV1-v2 are displayed in Figure~\ref{fig:simulated_ablations}, where
we use the AUC strategy to choose the learning rate. We choose the
mellowness by visual inspection for the baselines and use $0.01$ for
\alg\footnote{We use $0.01$ for \textsc{AllOrNone} and $10^{-3}$ for
  \textsc{NoDom}.}. On ImageNet 40, the ablations provide minimal
improvement over passive learning, while on RCV1-v2,
\textsc{AllOrNone} does provide marginal improvement. However, on both
datasets, \alg substantially outperforms both baselines and passive
learning.

While not always the best, we recommend a mellowness setting of $0.01$
as it achieves reasonable performance on all three datasets. This is
also confirmed by the learning-to-search experiments, which we discuss
next.

\subsection{Learning to Search}
\label{ssec:l2s_experiments}
We also experiment with \alg as the base leaner in
\emph{learning-to-search}~\citep{daume09searn,chang2015learning}, which
reduces joint prediction problems to CSMC. A joint prediction
example defines a search space, where a sequence of decisions are made
to generate the structured label. We focus here on sequence labeling
tasks, where the input is a sentence and the output is a sequence of
labels, specifically, parts of speech or named entities.

Learning-to-search solves such problems by generating the
output one label at a time, conditioning on all past
decisions.  Since mistakes may lead to compounding errors, it is
natural to represent the decision space as a CSMC problem, where the
classes are the ``actions'' available (e.g., possible labels for a
word) and the costs reflect the long term loss of each choice.
Intuitively, we should be able to avoid expensive computation of long
term loss on decisions like ``is \textsf{\small`the'} a
\textsc{determiner}?'' once we are quite sure of the answer. Similar
ideas motivate adaptive sampling for structured
prediction~\citep{shi15sampling}.

We specifically use
\aggravate~\citep{ross2014reinforcement,chang2015learning,sun2017deeply},
which runs a learned policy to produce a backbone sequence of labels.
For each position in the input, it then considers all possible
deviation actions and executes an oracle for the rest of the
sequence. The loss on this complete output is used as the cost for the
deviating action.  Run in this way, \aggravate requires $\length \times K$
roll-outs when the input sentence has $\length$ words and each word can
take one of $K$ possible labels.

\begin{figure}
\begin{center}
  \includegraphics[width=0.32\textwidth,height=0.25\textwidth]{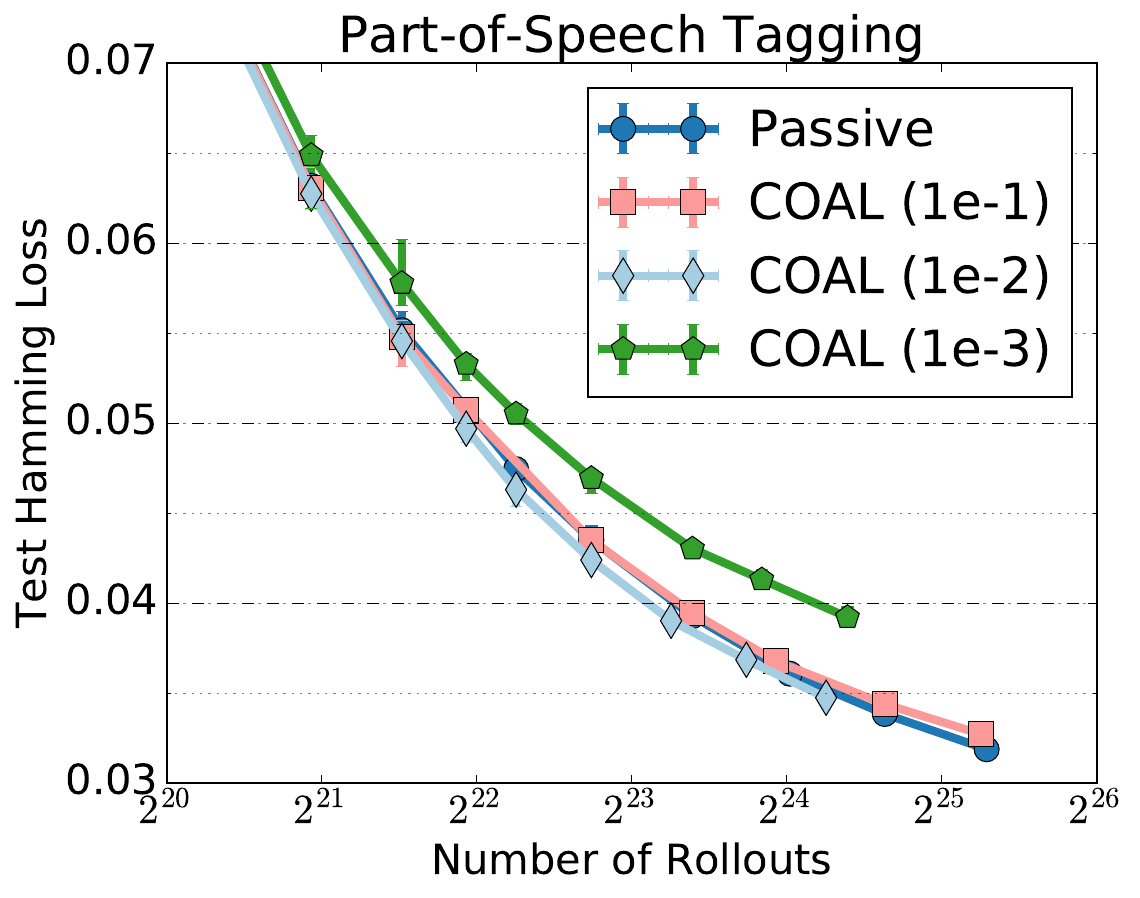}
  \includegraphics[width=0.32\textwidth,height=0.25\textwidth]{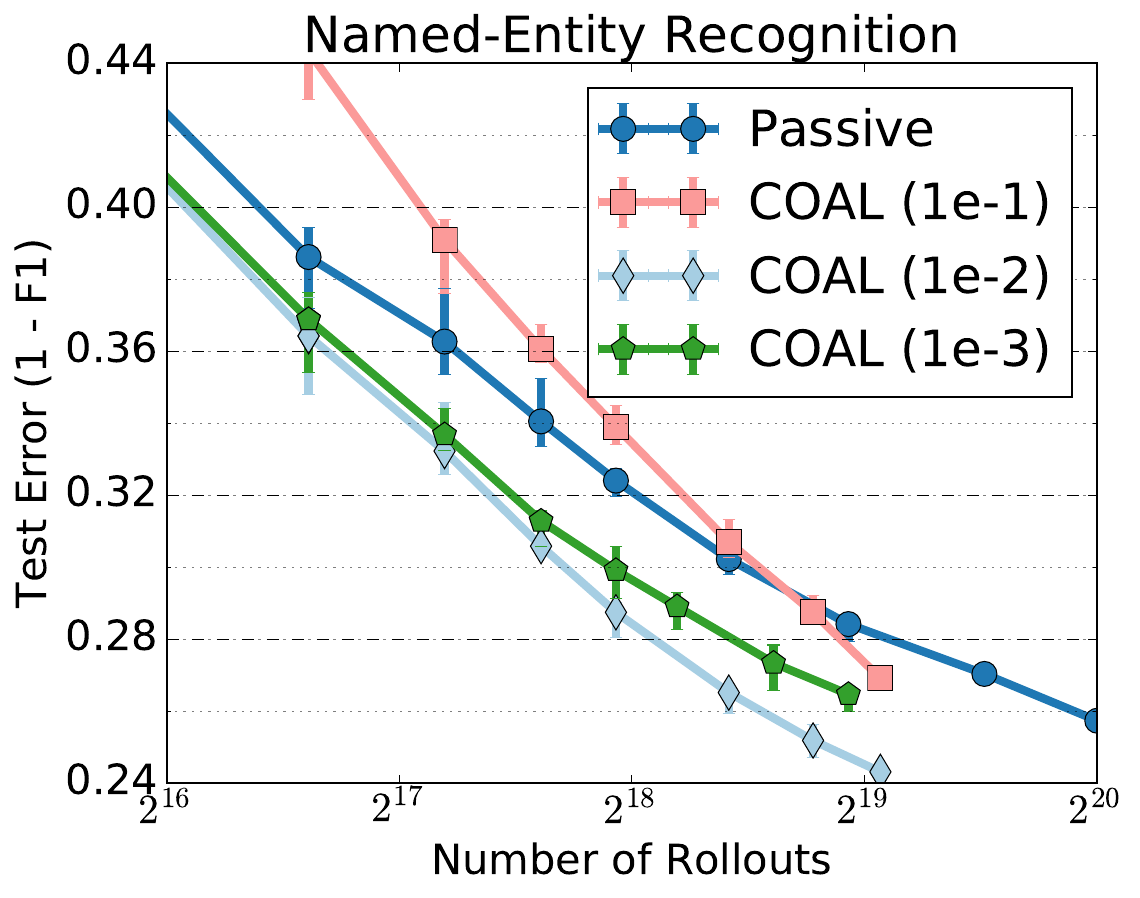}
  \includegraphics[width=0.32\textwidth,height=0.25\textwidth]{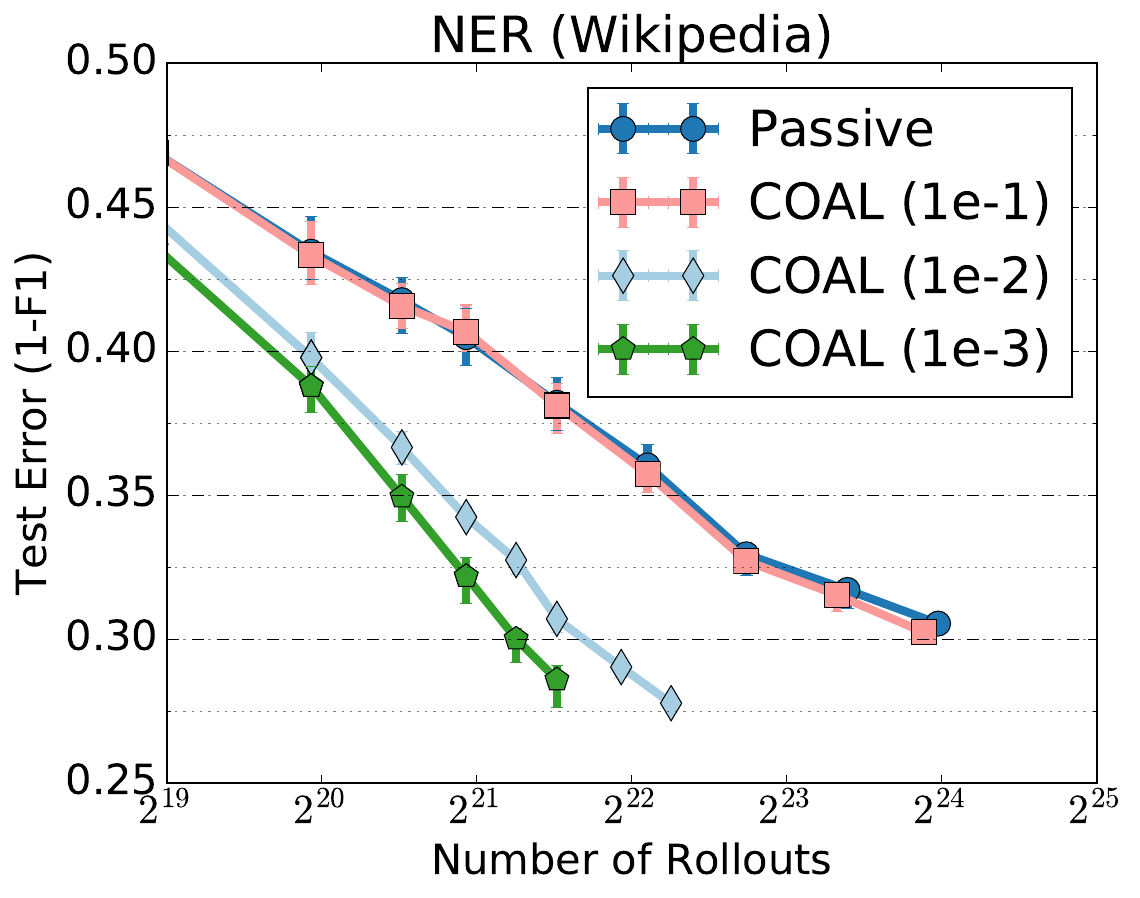}%
\caption{Learning to search experiments with \alg. Accuracy vs. number
  of rollouts for active and passive learning as the CSMC algorithm in
  learning-to-search.}%
\label{fig:l2s_results}%
\end{center}
\end{figure}

Since each roll-out takes $\order(\length)$ time, this can be
computationally prohibitive, so we use active learning to reduce the
number of roll-outs. We use \alg and a passive learning baseline
inside \aggravate on three joint prediction datasets (statistics are
in Table~\ref{tab:stats}). As above, we use several mellowness values
and the same AUC criteria to select the best learning rate (see
Table~\ref{tbl:learning_rates}). The results are in
Figure~\ref{fig:l2s_results}, and again our recommended mellowness is
$0.01$.

Overall, active learning reduces the number of roll-outs required, but
the improvements vary on the three datasets. On the Wikipedia data,
\alg performs a factor of 4 fewer rollouts to achieve similar
performance to passive learning and achieves substantially better test
performance. A similar, but less dramatic, behavior arises on the NER
task.  On the other hand, \alg offers minimal improvement over passive
learning on the POS-tagging task. This agrees with our theory and
prior empirical results~\citep{hsu2010algorithms}, which show that
active learning may not always improve upon passive learning.

\section{Proofs}
In this section we provide proofs for the main results, the
oracle-complexity guarantee and the generalization and label
complexity bounds. We start with some supporting results, including a
new uniform freedman-type inequality that may be of independent
interest. The proof of this inequality, and the proofs for several
other supporting lemmata are deferred to the appendices.

\subsection{Supporting Results}

\paragraph{A deviation bound.}
For both the computational and statistical analysis of \alg, we
require concentration of the square loss functional
$\Rhat_j(\cdot;y)$, uniformly over the class $\Gcal$. To describe
the result, we introduce the central random variable in the analysis:
\begin{align}
\label{eq:excess_sq_loss}
M_j(g;y) \triangleq Q_j(y)\left[(g(x_j)-c_j(y))^2 - (f^\star(x_j;y) - c_j(y))^2\right],
\end{align}
where $(x_j,c_j)$ is the $j^{\textrm{th}}$ example and cost presented
to the algorithm and $Q_j(y) \in \{0,1\}$ is the query indicator. For
simplicity we often write $M_j$ when the dependence on $g$ and $y$ is
clear from context. Let $\EE_j[\cdot]$ and $\Var_j[\cdot]$ denote the
expectation and variance conditioned on all randomness up to and
including round $j-1$.

\begin{theorem}
\label{thm:deviation}
Let $\Gcal$ be a function class with $\pdim(\Gcal) = d$, let $\delta
\in (0,1)$ and define $\logfactor \triangleq 324(d\log(n) +
\log(8Ke(d+1)n^2/\delta))$. Then with probability at least $1-\delta$,
the following inequalities hold simultaneously for all $g \in \Gcal$,
$y \in [K]$, and $i < i' \in [n]$.
\begin{align}
\sum_{j=i}^{i'} M_j(g;y) & \le \frac{3}{2}\sum_{j=i}^{i'} \EE_j M_j(g;y) + \logfactor,\label{eq:sample_bd}\\
\frac{1}{2} \sum_{j=i}^{i'} \EE_j M_j(g;y) &\le \sum_{j=i}^{i'} M_j(g;y) + \logfactor \label{eq:expectation_bd}.
\end{align}
\end{theorem}

This result is a uniform Freedman-type inequality for the martingale
difference sequence $\sum_i M_i - \EE_i M_i$. In general, such bounds
require much stronger assumptions (e.g., sequential complexity
measures~\citep{rakhlin2017equivalence}) on $\Gcal$ than the finite
pseudo-dimension assumption that we make. However, by exploiting the
structure of our particular martingale, specifically that the
dependencies arise \emph{only} from the query indicator, we are able
to establish this type of inequality under weaker assumptions. The
result may be of independent interest, but the proof, which is based
on arguments from~\citet{liang2015learning}, is quite technical and
deferred to Appendix~\ref{app:deviation}. Note that we did not optimize the constants.

\paragraph{The Multiplicative Weights Algorithm.}
We also use the standard analysis of multiplicative weights for
solving linear feasibility problems. We state the result here and, for
completeness, provide a proof in Appendix~\ref{app:mw}. See
also~\citet{arora2012multiplicative,plotkin1995fast} for more details.

Consider a linear feasibility problem with decision variable $v \in
\RR^d$, explicit constraints $\langle a_i, v\rangle \le b_i$ for $i
\in [m]$ and some implicit constraints $v \in S$ (e.g., $v$ is
non-negative or other simple constraints). The MW algorithm either
finds an approximately feasible point or certifies that the program is
infeasible assuming access to an oracle that can solve a simpler
feasibility problem with just one explicit constraint $\sum_i\mw_i
\langle a_i,v\rangle \le \sum_i \mw_i b_i$ for any non-negative
weights $\mw \in \RR^m_+$ and the implicit constraint $v \in
S$. Specifically, given weights $\mw$, the oracle either reports that
the simpler problem is infeasible, or returns any feasible point $v$ that
further satisfies $\langle a_i, v \rangle - b_i \in [-\rho_i,\rho_i]$
for parameters $\rho_i$ that are known to the MW algorithm.

The MW algorithm proceeds iteratively, maintaining a weight vector
$\mw^{(t)} \in \RR^m_+$ over the constraints. Starting with
$\mw^{(1)}_i = 1$ for all $i$, at each iteration, we query the oracle
with the weights $\mw^{(t)}$ and the oracle either returns a point
$v_t$ or detects infeasibility. In the latter case, we simply report
infeasibility and in the former, we update the weights using the rule
\begin{align*}
\mw^{(t+1)}_i \gets \mw^{(t)}_i \times \left(1 - \eta \frac{b_i - \langle a_i, v_t\rangle}{\rho_i}\right).
\end{align*}
Here $\eta$ is a parameter of the algorithm. The intuition is that if
$v_t$ satisfies the $i^{\textrm{th}}$ constraint, then we down-weight
the constraint, and conversely, we up-weight every constraint that is
violated. Running the algorithm with appropriate choice of $\eta$ and
for enough iterations is guaranteed to approximately solve the
feasibility problem.

\begin{theorem}[\citet{arora2012multiplicative,plotkin1995fast}]
  \label{thm:mw}
  Consider running the MW algorithm with parameter $\eta =
  \sqrt{\log(m)/T}$ for $T$ iterations on a linear feasibility problem
  where oracle responses satisfy $\langle a_i, v\rangle - b_i \in
  [-\rho_i,\rho_i]$. If the oracle fails to find a feasible point in
  some iteration, then the linear program is infeasible. Otherwise the
  point $\bar{v} \triangleq \frac{1}{T}\sum_{t=1}^Tv_t$ satisfies $\langle a_i,
  \bar{v}\rangle \le b_i + 2\rho_i\sqrt{\log(m)/T}$ for all $i \in
      [m]$.
\end{theorem}

\paragraph{Other Lemmata.}
Our first lemma evaluates the conditional expectation and variance of
$M_j$, defined in~\eqref{eq:excess_sq_loss}, which we will use heavily
in the proofs. Proofs of the results stated here are deferred to
Appendix~\ref{app:lemmata}.
\begin{lemma}[Bounding variance of regression regret]
  \label{lem:regret_transform}
  We have for all $(g,y) \in \Gcal \times \Ycal$,
  \begin{align*}
    \EE_j[M_j] = \EE_j\left[ Q_j(y) (g(x_j) - \fstar{x_j}{y})^2 \right], \qquad
    \Var_j[M_j] \leq 4\EE_i [M_j].
  \end{align*}
\end{lemma}

The next lemma relates the cost-sensitive error to the random
variables $M_j$. Define
\begin{align*}
\vF_i = \left\{f \in \Gcal^K \mid \forall y, f(\cdot;y) \in \Gcal_i(y)\right\},
\end{align*}
which is the version space of vector regressors at round $i$.
Additionally, recall that $P_\zeta$ captures the noise level in the
problem, defined in~\eqref{eq:noise_level} \df{and that $\psi_i=
  1/\sqrt{i}$ is defined in the algorithm pseudocode.}
\begin{lemma}
  \label{lem:refined_generalization_bd}
  For all $i > 0$, if $\vf^\star \in \vF_i$, then for all $\vf \in
  \vF_i$
  \begin{align*}
    \EE_{x,c}[c(h_{\vf}(x)) - c(h_{{\vf}^\star}(x))] \leq \min_{\zeta
      > 0}\left\{ \zeta P_\zeta + \ind{\zeta \le 2\psi_i} 2\psi_i +
    \frac{4\psi_i^2}{\zeta} + \frac{6}{\zeta}\sum_y
    \EE_i\left[\df{M_i}\right]\right\}.
  \end{align*}
\end{lemma}
Note that the lemma requires that both $\vf^\star$ and $\vf$ belong to
the version space $\vF_i$.

For the label complexity analysis, we will need to understand the
cost-sensitive performance of all $f \in \vF_i$, which requires a
different generalization bound. Since the proof is similar to that of
Theorem~\ref{thm:reg_bound}, we defer the argument to appendix.
\begin{lemma}
\label{lem:version_to_csr}
Assuming the bounds in Theorem~\ref{thm:deviation} hold, then for all
$i$, $\vF_i \subset \vFcsr(r_i)$ where $r_i \triangleq \min_{\zeta > 0}\left\{
\zeta P_\zeta + \frac{44K\Delta_i}{\zeta}\right\}.$
\end{lemma}

The final lemma relates the query rule of \alg to a hypothetical query
strategy driven by $\vFcsr(r_i)$, which we will subsequently bound by
the disagreement coefficients. Let us fix the round $i$ and introduce
the shorthand $\widehat{\gamma}(x_i,y) = \widehat{c}_+(x_i, y) -
\widehat{c}_-(x_i,y)$, where $\widehat{c}_+(x_i,y)$ and
$\widehat{c}_-(x_i,y)$ are the approximate maximum and minimum costs
computed in Algorithm~\ref{alg:a_cs} on the $i^{\textrm{th}}$ example,
which we now call $x_i$. Moreover, let $Y_i$ be the set of
non-dominated labels at round $i$ of the algorithm, which in the
pseudocode we call $Y'$. Formally, $Y_i = \{y \mid
\widehat{c}_{-}(x_i,y) \le \min_{y'}
\widehat{c}_{+}(x_i,y')\}$. Finally recall that for a set of vector
regressors $F \subset \Fcal$, we use $\gamma(x,y,F)$ to denote the
cost range for label $y$ on example $x$ witnessed by the regressors in
$F$.
\begin{lemma}
\label{lem:cost_range_translations}
Suppose that the conclusion of Lemma~\ref{lem:version_to_csr}
holds. Then for any example $x_i$ and any label $y$ at round $i$, we
have
\begin{align*}
\widehat{\gamma}(x_i,y) \le \gamma(x_i,y,\vFcsr(r_i)) + \psi_i.
\end{align*}
Further, with $y_i^\star = \argmin_y \vf^\star(x_i;y), \bar{y}_i = \argmin_y \widehat{c}_+(x_i,y)$, and $\tilde{y}_i = \argmin_{y \ne y_i^\star}\widehat{c}_-(x_i,y)$,
\begin{align*}
y \ne y_i^\star \wedge y \in Y_i &\Rightarrow \vf^\star(x_i;y) - \vf^\star(x_i;y_i^\star) \le \gamma(x_i,y,\vFcsr(r_i)) + \gamma(x_i,y_i^\star,\vFcsr(r_i)) + \psi_i/2,\\
|Y_i| > 1 \wedge y_i^\star \in Y_i &\Rightarrow \vf^\star(x_i;\tilde{y}_i) - \vf^\star(x_i;y_i^\star) \le \gamma(x_i,\tilde{y}_i,\vFcsr(r_i)) + \gamma(x_i,y_i^\star,\vFcsr(r_i)) + \psi_i/2.
\end{align*}
\end{lemma}

\subsection{Proof of Theorem~\ref{thm:mw_guarantee}}
The proof is based on expressing the optimization
problem~\eqref{eq:max_cost_opt} as a linear optimization in the space
of distributions over $\Gcal$. Then, we use binary search to
re-formulate this as a series of feasibility problems and apply
Theorem~\ref{thm:mw} to each of these.

Recall that the problem of finding the maximum cost for an $(x,y)$
pair is equivalent to solving the program~\eqref{eq:max_cost_opt} in
terms of the optimal $g$. For the problem~\eqref{eq:max_cost_opt}, we
further notice that since $\Gcal$ is a convex set, we can instead
write the minimization over $g$ as a minimization over $P \in
\Delta(\Gcal)$ without changing the optimum, leading to the modified
problem~\eqref{eq:max_cost_conv}.

Thus we have a linear program in variable $P$, and
Algorithm~\ref{alg:mw_max_cost} turns this into a feasibility problem
by guessing the optimal objective value and refining the guess using
binary search. For each induced feasibility problem, we use MW to
certify feasibility. Let $c \in [0,1]$ be some guessed upper bound on
the objective, and let us first turn to the MW component of the
algorithm. The program in consideration is
\begin{align}
?\exists P \in \Delta(G) \textrm{ s.t. } \EE_{g \sim P}(g(x_i)-1)^2
  \le c \textrm{ and } \forall j \in [i], \EE_{g \sim P}\Rhat_j(g;y)
  \le \tilde{\Delta}_j.
  \label{eqn:mw_feasible}
\end{align}
This is a linear feasibility problem in the infinite dimensional
variable $P$, with $i+1$ constraints. Given a particular set of
weights $\mw$ over the constraints, it is clear that we can use the
regression oracle over $g$ to compute
\begin{align}
  g_\mw = \arg\min_{g \in \Gcal} \mw_0 (g(x_i)-1)^2 + \sum_{j \in [i]} \mw_j
  \EE_{g \sim P}\Rhat_j(g;y).\label{eq:regression-mw}
\end{align}
Observe that solving this simpler program provides one-sided errors.
Specifically, if the objective of~\eqref{eq:regression-mw} evaluated
at $g_\mw$ is larger than $\mw_0 c + \sum_{j \in [i]} \mw_j
\tilde{\Delta}_j$ then there cannot be a feasible solution to
problem~\eqref{eqn:mw_feasible}, since the weights $\mw$ are all
non-negative. On the other hand if $g_\mu$ has small objective value
it does not imply that $g_{\mw}$ is feasible for the original
constraints in~\eqref{eqn:mw_feasible}.

At this point, we would like to invoke the MW algorithm, and
specifically Theorem~\ref{thm:mw}, in order to find a feasible
solution to~\eqref{eqn:mw_feasible} or to certify
infeasibility. Invoking the theorem requires the $\rho_j$ parameters
which specify how badly $g_{\mw}$ might violate the $j^{\textrm{th}}$
constraint. For us, $\rho_j \triangleq \regconst$ suffices since
$\hat{R}_j(g;y) - \hat{R}_j(g_{j,y};y) \in [0,1]$ (since $g_{j,y}$ is
the ERM) and $\Delta_j \le \regconst$. Since $\regconst \ge 2$ this
also suffices for the cost constraint.

If at any iteration, MW detects infeasibility, then our guessed value
$c$ for the objective is too small since no function satisfies both
$(g(x_i)-1)^2 \leq c$ and the empirical risk constraints
in~\eqref{eqn:mw_feasible} simultaneously. In this case, in
Line~\ref{line:mw_infeasible} of Algorithm~\ref{alg:mw_max_cost}, our
binary search procedure increases our guess for $c$.
On the other hand, if we apply MW for $T$ iterations and find a
feasible point in every round, then, while we do not have a point that
is feasible for the original constraints in~\eqref{eqn:mw_feasible},
we will have a distribution $P_T$ such that
\begin{align*}
\EE_{P_T} (g(x_i)-1)^2 \le c + 2\regconst \sqrt{\frac{\log(i+1)}{T}}
\textrm{ and } \forall j \in [i], \EE_{P_T}\Rhat_j(g;y) \le
\tilde{\Delta}_j + 2\regconst\sqrt{\frac{\log(i+1)}{T}}.
\end{align*}
We will set $T$ toward the end of the proof.

If we do find an approximately feasible solution, then we reduce $c$
and proceed with the binary search.
We terminate when $c_h - c_\ell \le \tau^2/2$
and we know that problem~\eqref{eqn:mw_feasible} is approximately
feasible with $c_h$ and infeasible with $c_\ell$. From $c_h$ we will
construct a strictly feasible point, and this will lead to a bound on
the true maximum $c_{+}(x,y,\Gcal_i)$.

Let $\bar{P}$ be the approximately feasible point found when running
MW with the final value of $c_h$. By Jensen's inequality and convexity
of $\Gcal$, there exists a single regressor that is also approximately
feasible, which we denote $\bar{g}$. Observe that $g^\star$ satisfies
all constraints with strict inequality, since
by~\eqref{eq:f_star_slack} we know that $\Rhat_j(g^\star;y) -
\Rhat_j(g_{j,y};y) \le \Delta_j/\regconst< \Delta_j$. We create a
strictly feasible point $g_\zeta$ by mixing $\bar{g}$ with $g^\star$
with proportion $1-\zeta$ and $\zeta$ for
\begin{align*}
\zeta = \frac{4\regconst}{\Delta_i}\sqrt{\frac{\log(i+1)}{T}},
\end{align*}
which will be in $[0,1]$ when we set $T$.
Combining inequalities, we get that for any $j \in [i]$
\begin{align*}
(1-\zeta)\Rhat_j(\bar{g};y) + \zeta \Rhat_j(g^\star;y) &\le \Rhat_j(g_{j,y};y) + (1-\zeta)\left(\Delta_j + 2\regconst\sqrt{\frac{\log(i+1)}{T}}\right) + \zeta\left(\frac{\Delta_j}{\regconst}\right)\\
& \le \Rhat_j(g_{j,y};y) + \Delta_j - \left(\frac{\zeta\Delta_j(\regconst-1)}{\regconst} - 2\regconst\sqrt{\frac{\log(i+1)}{T}}\right)\\
& \le \Rhat_j(g_{j,y};y) + \Delta_j,
\end{align*}
and hence this mixture regressor $g_\zeta$ is exactly feasible. Here
we use that $\regconst \ge 2$ and that $\Delta_i$ is monotonically
decreasing. With the pessimistic choice $g^\star(x_i)=0$, the
objective value for $g_\zeta$ is at most
\begin{align*}
(g_\zeta(x_i) - 1)^2 &\leq (1-\zeta)(\bar{g}(x_i) - 1)^2 + \zeta (g^\star(x_i) - 1)^2 \le
  (1-\zeta)\left(c_h +2\regconst\sqrt{\frac{\log(i+1)}{T}}\right) + \zeta\\
  &\le c_\ell+\tau^2/2 + \left(2\regconst + \frac{4\regconst}{\Delta_i}\right) \sqrt{\frac{\log(i+1)}{T}}.
\end{align*}
Thus $g_\zeta$ is \emph{exactly} feasible and achieves the objective
value above, which provides an upper bound on the maximum cost. On the
other hand $c_\ell$ provides a lower bound. Our setting of $T =
\frac{\log(i+1)(8\regconst^2/\Delta_i)^2}{\tau^4}$ ensures that that
this excess term is at most $\tau^2$, since $\Delta_i \le 1$. Note
that since $\tau \le [0,1]$, this also ensures that $\zeta \in [0,1]$.
With this choice of $T$, we know that $c_\ell \le (\cmax{y} - 1)^2 \le
c_\ell + \tau^2$, which implies that $\cmax{y} \in [1 -
  \sqrt{c_\ell+\tau^2}, 1 - \sqrt{c_\ell}]$. Since
$\sqrt{c_\ell+\tau^2} \le \sqrt{c_\ell} + \tau$, we obtain the
guarantee.

As for the oracle complexity, since we start with $c_\ell=0$ and
$c_h=1$ and terminate when $c_h - c_\ell\le \tau^2/2$, we perform
$\order(\log(1/\tau^2))$ iterations of binary search. Each iteration
requires $T = \order\left(\max\{1,
i^2/\logfactor^2\}\frac{\log(i)}{\tau^4}\right)$ rounds of MW, each of
which requires exactly one oracle call. Hence the oracle complexity is
$\order\left(\max\{1,i^2/\logfactor^2\}\frac{
  \log(i)\log(1/\tau)}{\tau^4}\right)$.\hfill\ifthenelse{\equal{\version}{main}}{\qedsymbol}{\qedhere}


\subsection{Proof of the Generalization Bound}
Recall the central random variable $M_j(g;y)$, defined
in~\eqref{eq:excess_sq_loss}, which is the excess square loss for
function $g$ on label $y$ for the $j^{\textrm{th}}$ example, if we
issued a query.  The idea behind the proof is to first apply
Theorem~\ref{thm:deviation} to argue that all the random variables
$M_j(g;y)$ concentrate uniformly over the function class $\Gcal$. Next
for a vector regressor $f$, we relate the cost-sensitive risk to the
excess square loss via
Lemma~\ref{lem:refined_generalization_bd}. Finally, using the fact
that $g_{i,y}$ minimizes the empirical square loss at round $i$, this
implies a cost-sensitive risk bound for the vector regressor $f_i =
(g_{i,y})$ at round $i$.

First, condition on the high probability event in
Theorem~\ref{thm:deviation}, which ensures that the empirical square
losses concentrate. We first prove that $\vf^\star \in \vF_i$ for all
$i \in [n]$. At round $i$, by~\eqref{eq:expectation_bd}, for each $y$
and for any $g$ we have
\begin{align*}
0 \le \frac{1}{2}\sum_{j=1}^i \EE_jM_j(g;y) \le \sum_{j=1}^i M_j(g;y) + \logfactor.
\end{align*}
The first inequality here follows from the fact that $\EE_jM_j(g;y)$
is a quadratic form by Lemma~\ref{lem:regret_transform}. Expanding
$M_j(g;y)$, this implies that
\begin{align*}
\Rhat_{i+1}(\vf^\star(\cdot;y);y) \le \Rhat_{i+1}(g;y) + \frac{\logfactor}{i}.
\end{align*}
Since this bound applies to all $g \in \Gcal$ it proves that
$\vf^\star \in \vF_{i+1}$ for all $i$, using the definition of
$\Delta_i$ and $\regconst$. Trivially, we know that $\vf^\star \in
\vF_1$. Together with the fact that the losses are in $[0,1]$ and the
definition of $\Delta_i$, the above analysis yields
\begin{align}
\Rhat_{i}(\vf^\star(\cdot;y);y) \le \Rhat_{i}(g;y) + \frac{\Delta_i}{\regconst}. \label{eq:f_star_slack}
\end{align}
This implies that $f^\star(\cdot;y)$ strictly satisfies the
inequalities defining the version space, which we used in the MW
proof.

We next prove that $f_{i+1} \in \vF_j$ for all $j \in [i]$. Fix some
label $y$ and to simplify notation, we drop dependence on $y$.
If $g_{i+1} \notin \Gcal_{t+1}$ for some $t \in \{0,\ldots,i\}$ then,
first observe that we must have $t$ large enough so that $\logfactor/t
\le 1$. In particular, since $\Delta_{t+1} = \regconst\min\{1,
\logfactor/t\}$ and we always have $\df{\hat{R}_{t+1}(g_{i+1};y) \le
\hat{R}_{t+1}(g_{t+1};y) + 1}$ due to boundedness, we do not evict any
functions until $\logfactor/t \le 1$. For $t \ge \logfactor$, we get
\begin{align*}
\sum_{j=1}^tM_j &= t\left(\Rhat_{t+1}(g_{i+1}) - \Rhat_{t+1}(g^\star)\right)\\
& = t\left(\Rhat_{t+1}(g_{i+1}) - \Rhat_{t+1}(g_{t+1}) + \Rhat_{t+1}(g_{t+1}) - \Rhat_{t+1}(g^\star)\right)\\
& \ge \regconst\logfactor - \logfactor = (\kappa-1)\logfactor.
\end{align*}
The inequality uses the radius of the version space and the fact that
by assumption $g_{i+1}\notin \Gcal_{t+1}$, so the excess empirical
risk is at least $\Delta_{t+1}=\regconst\logfactor/t$ since we are
considering large $t$. We also use~\eqref{eq:f_star_slack} on the
second term.
Moreover, we know that since $g_{i+1}$ is the empirical square loss
minimizer for label $y$ after round $i$, we have
$\sum_{j=1}^i\df{M_j(g_{i+1};y)} \le 0$. These two facts together establish
that
\begin{align*}
\sum_{j=t+1}^i M_j(g_{i+1}) \le (1-\regconst)\logfactor.
\end{align*}
However, by Theorem~\ref{thm:deviation} on this intermediary sum, we
know that
\begin{align*}
0 \le \frac{1}{2} \sum_{j=t+1}^i \EE_jM_j(g_{i+1}) \le \sum_{j=t+1}^iM_j(g_{i+1}) + \logfactor < (2-\regconst) \logfactor < 0
\end{align*}
using the definition of $\regconst$. This is a contradiction, so we
must have that $g_{i+1} \in \Gcal_j$ for all $j
\in\{1,\ldots,i\}$. The same argument applies for all $y$ and hence
we can apply Lemma~\ref{lem:refined_generalization_bd} on all rounds
to obtain
\begin{align*}
  & i\big( \EE_{x,c}[c(h_{\vf_{i+1}}(x)) -
    c(h_{\vf^\star}(x))] \big)\\
  &\leq \min_{\zeta > 0} \left\{i \zeta P_\zeta
  + \sum_{j=1}^i\left(\ind{\zeta \le 2\psi_j} 2\psi_j +
  \frac{4\psi_j^2}{\zeta} + \frac{6}{\zeta}\sum_y
  \EE_j\left[M_j(\vf_{i+1};y)\right]\right)\right\}.
\end{align*}
We study the four terms separately.  The first one is straightforward
and contributes $\zeta P_\zeta$ to the instantaneous cost sensitive
regret.  Using our definition of $\psi_j = 1/\sqrt{j}$ the second term
can be bounded as
\begin{align*}
  \sum_{j=1}^i \ind{\zeta < 2\psi_j} 2\psi_j &= \sum_{j=1}^{\lceil
    4/\zeta^2\rceil} \frac{2}{\sqrt{j}} \le 4 \sqrt{\lceil 4/\zeta^2
    \rceil} \le \frac{12}{\zeta}.
\end{align*}
The inequality above, $\sum_{i=1}^n \frac{1}{\sqrt{i}} \le
2\sqrt{n}$, is well known.  For the third term, using our definition of $\psi_j$ gives
\begin{align*}
  \sum_{j=1}^i\frac{4 \psi_j^2}{\zeta} = \frac{4}{\zeta}\sum_{j=1}^i
  \frac{1}{j} \le \frac{4}{\zeta}(1+\log(i)).
\end{align*}
Finally, the fourth term can be bounded
using~\eqref{eq:expectation_bd}, which reveals
\begin{align*}
  \sum_{j=1}^i \EE_j[M_j] \le 2\sum_{j=1}^i M_j + 2\logfactor
\end{align*}
Since for each $y$, $\sum_{j=1}^iM_j(f_{i+1};y) \le 0$ for the empirical
square loss minimizer (which is what we are considering now), we get
\begin{align*}
  \frac{6}{\zeta} \sum_y \sum_{j=1}^i \EE_j[M_j(\vf_{i+1};y)] \le
  \frac{12}{\zeta} K\logfactor.
\end{align*}
And hence, we obtain the generalization bound
\begin{align*}
  \EE_{x,c}[c(x;h_{\vf_{i+1}}(x)) - c(x;h_{\vf^\star}(x))] &\leq
  \min_{\zeta > 0}\left\{ \zeta P_\zeta + \frac{1}{\zeta i}\left(4 \log(i) + 16 +
  12K\logfactor\right)\right\}\\
  & \leq \min_{\zeta > 0}\left\{ \zeta P_\zeta +
  \frac{32K\logfactor}{\zeta i}\right\}.  \tag*{\ifthenelse{\equal{\version}{main}}{\qedsymbol}{\qedhere}}
\end{align*}

\paragraph{Proof of Corollary~\ref{cor:reg_low_noise}.}
Under the Massart noise condition, set $\zeta = \tau$ so that $P_\zeta
= 0$ and we immediately get the result.\hfill\ifthenelse{\equal{\version}{main}}{\qedsymbol}{\qedhere}

\paragraph{Proof of Corollary~\ref{cor:reg_tsybakov}.}
Set $\zeta =
\min\left\{\tau_0,
\left(\frac{32\df{K}\logfactor}{i\beta}\right)^{\frac{1}{\alpha+2}}\right\}$,
so that for $i$ sufficiently large the second term is selected
and we obtain the bound.\hfill\ifthenelse{\equal{\version}{main}}{\qedsymbol}{\qedhere}

\subsection{Proof of the Label Complexity bounds}
The proof for the label complexity bounds is based on first relating
the version space $\Fcal_i$ at round $i$ to the cost-sensitive regret
ball $\vFcsr$ with radius $r_i$. In particular, the containment $\vF_i
\subset \vFcsr(r_i)$ in Lemma~\ref{lem:version_to_csr} implies that
our query strategy is more aggressive than the query strategy induced
by $\vFcsr(r_i)$, except for a small error introduced when computing
the maximum and minimum costs. This error is accounted for by
Lemma~\ref{lem:cost_range_translations}. Since the probability that
$\vFcsr$ will issue a query is intimately related to the disagreement
coefficient, this argument leads to the label complexity bounds for
our algorithm.

\paragraph{Proof of Theorem~\ref{thm:high_noise}.}
Fix some round $i$ with example $x_i$, let $\vF_i$ be the vector regressors used at round
$i$ and let $\Gcal_i(y)$ be the corresponding regressors for label
$y$. Let $\bar{y}_i = \argmin_y \cmaxhat{x_i,y}, y_i^\star =
\argmin_y\vf^\star(x_i;y)$, and $\tilde{y}_i = \argmin_{y \ne
  y_i^\star}\cminhat{x_i,y}$. Assume that
Lemma~\ref{lem:version_to_csr} holds. The $L_2$ label complexity is
\begin{align*}
\sum_y Q_i(y) &= \sum_y \one\{|Y_i| > 1 \wedge y \in Y_i\} \one\{\widehat{\gamma}(x_i,y) > \psi_i\}\\
& \le \sum_y \one\{|Y_i| > 1 \wedge y \in Y_i\} \one\{\gamma(x_i,y,\vFcsr(r_i)) > \psi_i/2\}.
\end{align*}
For the former indicator, observe that $y \in Y_i$ implies that there
exists a vector regressor $f \in \vF_i\subset \vFcsr(r_i)$ such that
$h_f(x_i)=y$. This follows since the domination condition means that
there exists $g \in \Gcal_i(y)$ such that $g(x_i) \le \min_{y'}
\max_{g' \in \Gcal_i(y')} g'(x_i)$. Since we are using a factored
representation, we can take $f$ to use $g$ on the $y^{\textrm{th}}$
coordinate and use the maximizers for all the other
coordinates. Similarly, there exists another regressor $f' \in \vF_i$
such that $h_{f'}(x_i) \ne y$. Thus this indicator can be bounded by
the disagreement coefficient
\begin{align*}
\sum_y Q_i(y) & \le \sum_y \one\{\exists f, f' \in \vFcsr(r_i) \mid h_f(x_i)=y \ne h_{f'}(x_i)\}\one\{\gamma(x_i,y,\vFcsr(r_i)\df{)} > \psi_i/2\}\\
& = \sum_y \one\{x \in \dis(r_i,y)\wedge \gamma(x_i,y,\vFcsr(r_i)\df{)} \ge \psi_i/2\}.
\end{align*}
We will now apply Freedman's inequality on the sequence $\{\sum_y
Q_i(y)\}_{i=1}^n$, which is a martingale with range $K$. Moreover, due
to non-negativity, the conditional variance is at most $K$ times the
conditional mean, and in such cases, Freedman's inequality reveals
that with probability at least $1-\delta$
\begin{align*}
X \le \EE X + 2 \sqrt{R \EE X \log(1/\delta)} + 2R\log(1/\delta) \le 2
\EE X + 3 R\log(1/\delta),
\end{align*}
where $X$ is the non-negative martingale with range $R$ and
expectation $\EE X$. The last step is by the fact that $2\sqrt{ab} \le
a + b$.

For us, Freedman's inequality implies that with probability at least $1-\delta/2$
\begin{align*}
\sum_{i=1}^n \sum_y Q_i(y) &\le 2 \sum_i \EE_i \sum_y Q_i(y) + 3K\log(2/\delta)\\
& \le 2 \sum_i \EE_i \sum_y \one\{x \in \dis(r_i,y)\wedge \gamma(x_i,y,\vFcsr(r_i)\df{)} \ge \psi_i/2\} + 3K\log(2/\delta)\\
& \le 4 \sum_i \frac{r_i}{\psi_i}\theta_2 + 3K\log(2/\delta).
\end{align*}
The last step here uses the definition of the disagreement coefficient
$\theta_2$. To wrap up the proof we just need to upper bound the
sequence, using our choices of $\psi_i = 1/\sqrt{i}$, $r_i =
2\sqrt{44K\Delta_i}$, and $\Delta_i = \regconst\min\{1,
\frac{\logfactor}{i-1}\}$. With simple calculations this is
easily seen to be at most
\begin{align*}
8n\theta_2\sqrt{88 K\regconst\logfactor} + 3K\log(2/\delta).
\end{align*}

Similarly for $L_1$ we can derive the bound
\begin{align*}
L_1 \le \sum_i \ind{\exists y \mid \gamma(x_i,y,\vFcsr(r_i)) \ge \psi_i/2 \wedge x \in \textrm{DIS}(r_i,y)},
\end{align*}
and then apply Freedman's inequality to obtain that
with probability at least $1-\delta/2$
\begin{align*}
L_1 &\le 2\sum_{i=1}^n\frac{2r_i}{\psi_i}\theta_1 + 3\log(2/\delta) \le 8\theta_1n\sqrt{88K\regconst \logfactor} + 3\log(2/\delta).\tag*{\ifthenelse{\equal{\version}{main}}{\qedsymbol}{\qedhere}}
\end{align*}

\paragraph{Proof of Theorem~\ref{thm:massart}.}
Using the same notations as in the bound for the high noise case we first express the $L_2$ label complexity as
\begin{align*}
\sum_y Q_i(y) &= \sum_y \one\{|Y_i|> 1, y \in Y_i\}Q_i(y).
\end{align*}
We need to do two things with the first part of the query indicator,
so we have duplicated it here.  For the second, we will use the
derivation above to relate the query rule to the disagreement
region. For the first, by Lemma~\ref{lem:cost_range_translations}, for
$y \ne y_i^\star$, we can derive the bound
\begin{align*}
& \one\{\vf^\star(x_i;y) - \vf^\star(x_i;y_i^\star) \le \gamma(x_i,y,\vFcsr(r_i)) + \gamma(x_i,y_i^\star, \vFcsr(r_i)) + \psi_i/2\}\\
& \le \one\{\tau- \psi_i/2 \le \gamma(x_i,y,\vFcsr(r_i)) + \gamma(x_i,y_i^\star, \vFcsr(r_i)) \}.
\end{align*}
For $y = y_i^\star$, we get the same bound but with $\tilde{y}_i$,
also by Lemma~\ref{lem:cost_range_translations}. Focusing on just one
of these terms, say where $y \ne y_i^\star$ and any round where $\tau \ge \psi_i$, we get
\begin{align*}
  & \one\{\tau - \psi_i/2 \le \gamma(x_i,y,\vFcsr(r_i)) + \gamma(x_i,y_i^\star,\vFcsr(r_i))\}Q_i(y) \\
  &\le \one\{\tau/2 \le\gamma(x_i,y,\vFcsr(r_i)) + \gamma(x_i,y_i^\star,\vFcsr(r_i))\}Q_i(y)\\
& \le \one\{\tau/4 \le \gamma(x_i,y,\vFcsr(r_i))\}Q_i(y) + \one\{\tau/4 \le \gamma(x_i,y_i^\star,\vFcsr(r_i))\}Q_i(y)\\
& \le D_i(y) + D_i(y_i^\star),
\end{align*}
where for shorthand we have defined $D_i(y) \triangleq \one\{\tau/4 \le
\gamma(x_i,y,\vFcsr(r_i)) \wedge x_i \in \dis(r_i,y)\}$.  The
derivation for the first term is straightforward. We obtain the
disagreement region for $y_i^\star$ since the fact that we query $y$
(i.e. $Q_i(y)$) implies there is $f$ such that $h_f(x_i) = y$, so this
function witnesses disagreement to $y_i^\star$.

The term involving $y_i^\star$ and $\tilde{y}_i$ is bounded in
essentially the same way, since we know that when $|Y_i| > 1$, there
exists two function $f,f' \in \vF_i$ such that $h_f(x_i) =
\tilde{y}_i$ and $h_{f'}(x_i) = y_i^\star$. In total, we can bound the
$L_2$ label complexity at any round $i$ such that $\tau \ge \psi_i$ by
\begin{align*}
L_2 \le D_i(\tilde{y}_i) + D_i(y_i^\star) + \sum_{y \ne y_i^\star}\left(D_i(y) + D_i(y_i^\star)\right) \le KD_i(y_i^\star) + 2\sum_y D_i(y)
\end{align*}
For the earlier rounds, we simply upper bound the label complexity by
$K$.  Since the range of this random variable is at most $3K$, using
Freedman's inequality just as in the high noise case, we get that with
probability at least $1-\delta/2$
\begin{align*}
  L_2 &\le \sum_{i=1}^n K\one\{\tau \le \psi_i\} + 2\sum_{i=2}^n \EE_i \df{\left[KD_i(y_i^\star) + 2\sum_y D_i(y)\right]} + 9K\log(2/\delta)\\
  & \le K \lceil 1/\tau^2 \rceil + \frac{8}{\tau}(K\theta_1 + \theta_2)\sum_{i=2}^nr_i + 9K\log(2/\delta).
\end{align*}
The first line here is the application of Freedman's inequality. In
the second, we evaluate the expectation, which we can relate to the
disagreement coefficients $\theta_1,\theta_2$. Moreover, we use the
setting $\psi_i=1/\sqrt{i}$ to evaluate the first term. As a
technicality, we remove the index $i=1$ from the second summation,
since we are already accounting for queries on the first round in the
first term. The last step is to evaluate the series, for which we use
the definition of $r_i = \min_{\zeta > 0} \left\{\zeta P_\zeta +
44K\Delta_i/\zeta \right\}$ and set $\zeta = \tau$, the Massart noise
level. This gives $r_i = 44K\Delta_i/\tau$. In total, we get
\begin{align*}
L_2 \le K\lceil 1/\tau^2\rceil + \frac{352K\regconst\logfactor}{\tau^2}(K\theta_1 + \theta_2) \log(n) + 9K\log(2/\delta).
\end{align*}
\df{As $\theta_2 \leq K\theta_1$ always, we drop $\theta_2$ from the above expression to obtain the stated bound.}

For $L_1$ we use a very similar argument. First, by
Lemmas~\ref{lem:version_to_csr} and~\ref{lem:cost_range_translations}
\begin{align*}
  L_1 \le \sum_{i=1}^n\one\{|Y_i|>  1\wedge \exists y \in Y_i, \gamma(x_i,y,\vFcsr(r_i)) \ge \psi_i/2\}.
\end{align*}
Again by Lemma~\ref{lem:cost_range_translations}, we know that
\begin{align*}
|Y_i| > 1 \wedge y \in Y_i \Rightarrow \exists f, f' \in \vFcsr(r_i), h_f(x_i) = y \ne h_{f'}(x_i).
\end{align*}
Moreover, one of the two classifiers can be $f^\star$, and so, when $\tau \ge \psi_i$, we can deduce
\begin{align*}
&\Rightarrow \vf^\star(x_i,y) - \vf^\star(x_i,y_i^\star) \le \gamma(x_i,y,\vFcsr(r_i)) + \gamma(x_i,y_i^\star,\vFcsr(r_i)) + \psi_i/2\\
 &\Rightarrow \tau/4 \le \gamma(x_i,y,\vFcsr(r_i)\df{)} \wedge \tau/4 \le \gamma(x_i,y_i^\star,\vFcsr(r_i)\df{)}.
\end{align*}
Combining this argument, we bound the $L_1$ label complexity as
\begin{align}
L_1 \le \sum_{i=1}^n\one\{\tau \le \psi_i\} + \sum_{i=2}^n\one\{\exists y \mid x_i \in \dis(\vFcsr(r_i),\df{y}) \wedge \gamma(x_i,y,\vFcsr(r_i)) \ge \tau/4 \}.\label{eq:massart_l1_bd}
\end{align}
Applying Freedman's inequality just as before gives
\begin{align*}
  L_1 &\le \lceil 1/\tau^2\rceil + 2\sum_{i=2}^n \frac{4 r_i}{\tau}\theta_1 + 2\log(2/\delta)
  \le \lceil 1/\tau^2 \rceil + \frac{352K\regconst\logfactor}{\tau^2}\theta_1\log(n) + 2 \log(2/\delta),
\end{align*}
with probability at least $1-\delta/2$. \hfill\ifthenelse{\equal{\version}{main}}{\qedsymbol}{\qedhere}


\paragraph{Proof of Theorem~\ref{thm:tsybakov}.}
For the Tsybakov case, the same argument as in the Massart case gives that with probability at least $1-\delta/2$
\begin{align*}
L_2 \le K\lceil 1/\tau^2\rceil + \frac{8}{\tau}(K\theta_1 + \theta_2)\sum_{i=2}^nr_i + \df{2}n\beta\tau^\alpha + 9K\log(2/\delta).
\end{align*}
The main difference here is the term scaling with $n\tau^\alpha$ which
arises since we do not have the deterministic bound $\vf^\star(x_i,y)
- \vf^\star(x_i,y_i^\star) \ge \tau$ as we used in the Massart case,
but rather this happens except with probability $\beta\tau^\alpha$
(provided $\tau \le \tau_0$). Now we must optimize $\zeta$ in the
definition of $r_i$ and then $\tau$.

For $\zeta$ the optimal setting is
$(44K\Delta_i/\beta)^{\frac{1}{\alpha+2}}$ which gives $r_i \le
2\df{\beta^{\frac{1}{\alpha+2}}}(44K\Delta_i)^{\frac{\alpha+1}{\alpha+2}}$. Since we want to
set $\zeta \le \tau_0$, this requires $i \ge 1 +
\frac{44K\regconst\logfactor}{\beta\tau_0^{\alpha+2}}$. For these
early rounds we will simply pay $K$ in the label complexity, but this
will be dominated by other higher order terms. For the later rounds,
we get
\begin{align*}
\sum_{i=2}^n r_i \le 2\df{\beta^{\frac{1}{\alpha+2}}}\left(44K\regconst\logfactor\right)^{\frac{\alpha+1}{\alpha+2}}\sum_{i=2}^n \left(\frac{1}{i-1}\right)^{\frac{\alpha+1}{\alpha+2}} \le \df{2(\alpha+2)}\left(44K\regconst\logfactor\right)^{\frac{\alpha+1}{\alpha+2}}(\df{\beta}n)^{\frac{1}{\alpha+2}}.
\end{align*}
\df{This bound uses the integral approximation $\sum_{i=2}^n(i-1)^{-\frac{\alpha+1}{\alpha+2}} \leq 1 + \int_1^{n-1} x^{-\frac{\alpha+1}{\alpha+2}}dx \leq (\alpha+2)n^{\frac{1}{\alpha+2}}$.}
\dfc{This bound uses that for any exponent $z \in (0,1)$ and $i \ge 2,
(i-1)^{-z} \le 2 (i-1)^{1-z} - 2(i-2)^{1-z}$, which telescopes. }
At this point, the terms involving $\tau$ in our bound are
\begin{align*}
K\lceil 1/\tau^2\rceil + \frac{\df{16}}{\tau}(K\theta_1 + \theta_2)\df{(\alpha+2)}\left(44K\regconst\logfactor\right)^{\frac{\alpha+1}{\alpha+2}}(\df{\beta}n)^{\frac{1}{\alpha+2}} + \df{2}n\beta\tau^\alpha.
\end{align*}
We set $\tau =
(\df{8(\alpha+2)}(K\theta_1+\theta_2))^{\frac{1}{\alpha+1}}\left(44K\regconst\logfactor\right)^{\frac{1}{\alpha+2}}(\df{\beta}n)^{\frac{-1}{\alpha+2}}$
by optimizing the second two terms which gives a final bound of
\begin{align*}
L_2 \le \order\left((K\theta_1+\theta_2)^{\frac{\alpha}{\alpha+1}}(K\logfactor)^{\frac{\alpha}{\alpha+2}}n^{\frac{2}{\alpha+2}} + K\log(1/\delta)\right).
\end{align*}
This follows since the $1/\tau^2$ term agrees in the $n$ dependence
and is lower order in other parameters, while the unaccounted for
querying in the early rounds is independent of $n$. The bound of
course requires that $\tau \le \tau_0$, which again requires $n$ large
enough. \df{Note we are treating $\alpha$ and $\beta$ as constants and
  we drop $\theta_2$ from the final statement.}

The $L_1$ bound requires only slightly different
calculations. Following the derivation for the Massart case, we get
\begin{align*}
L_1 &\le \lceil 1/\tau^2\rceil + \frac{8\theta_1}{\tau}\sum_{i=2}^nr_i + \df{2n\beta\tau^{\alpha}} + 2 \log(2/\delta)\\
& \le \lceil 1/\tau^2 \rceil + \frac{\df{16}\theta_1\df{(\alpha+2)}}{\tau}\left(44K\regconst\logfactor\right)^{\frac{\alpha+1}{\alpha+2}}(\df{\beta}n)^{\frac{1}{\alpha+2}} + \df{2n\beta\tau^\alpha} + 2\log(2/\delta),
\end{align*}
not counting the lower order term for the querying in the early
rounds. Here we set $\tau =
(\df{8(\alpha+2)}\theta_1)^{\frac{1}{\alpha+1}}\left(44K\regconst\logfactor\right)^{\frac{1}{\alpha+2}}(\df{\beta}n)^{\frac{-1}{\alpha+2}}$
to obtain
\begin{align*}
L_1 \le \order\left(\theta_1^{\frac{\alpha}{\alpha+1}}(K\logfactor)^{\frac{\alpha}{\alpha+2}}n^{\frac{2}{\alpha+2}} + \log(1/\delta)\right). \tag*{\ifthenelse{\equal{\version}{main}}{\qedsymbol}{\qedhere}}
\end{align*}

\paragraph{\df{Proof of Proposition~\ref{prop:bc_disagreement}: Massart Case.}}
\df{Observe that with Massart noise, we have $\vFcsr(r) \subset
  \vFmin(r/\tau)$, which implies that
\begin{align*}
\EE\one\cbr{\exists y \mid x \in \dis(r,y) \wedge \gamma(x,y,\vFcsr(r))\geq \tau/4} &\leq \EE\one\cbr{x \in \dismin(r/\tau)}\\
& \leq \frac{r}{\tau} \sup_{r > 0} \frac{1}{r} \EE \one\cbr{x \in \dismin(r)}.
\end{align*}
Thus we may replace $\theta_1$ with $\theta_0$ in the proof of the
$L_1$ label complexity bound above.  \hfill\ifthenelse{\equal{\version}{main}}{\qedsymbol}{\qedhere}}

\paragraph{\df{Proof of Proposition~\ref{prop:bc_disagreement}: Tsybakov Case.}}
\df{The proof is identical to Theorem~\ref{thm:tsybakov} except we
  must introduce the alternative coefficient $\theta_0$. To do so,
  define $\Xcal(\tau) \triangleq \{x \in \Xcal: \min_{y \ne y^\star(x)}
  f^\star(x,y) - f^\star(x,y^\star(x)) \geq \tau \}$, and note that
  under the Tsybakov noise condition, we have $\PP[x \notin
    \Xcal(\tau)] \leq \beta \tau^\alpha$ for all $0 \leq \tau \leq
  \tau_0$. For such a value of $\tau$ we have
\begin{align*}
\PP[h_f(x) \ne h_{f^\star}(x)] &\leq \EE\one\{x \in \Xcal(\tau) \wedge h_f(x) \ne h_{f^\star}(x)\} + \PP[x \notin \Xcal(\tau)]\\
& \leq \frac{1}{\tau} \EE c(h_f(x)) - c(h_{f^\star}(x)) + \beta\tau^\alpha.
\end{align*}
We use this fact to prove that $\vFcsr(r) \subset \vFmin(2r/\tau)$ for
$\tau$ sufficiently small. This can be seen from above by noting that
if $f \in \vFcsr(r)$ then the right hand side is $\frac{r}{\tau} +
\beta\tau^\alpha$, and if $\tau \leq (r/\beta)^{\frac{1}{1+\alpha}}$
the containment holds. Therefore, we have
\begin{align*}
\EE\one\cbr{\exists y \mid x \in \dis(r,y) \wedge \gamma(x,y,\vFcsr(r))\geq \tau/4} & \leq \EE\one\cbr{\exists y \mid x \in \dis(r,y)}\\
& \leq \EE\one\cbr{x \in \dismin(2r/\tau)}\\
& \leq \frac{2r}{\tau} \cdot \sup_{r>0} \EE\one\cbr{x \in \dismin(r)}.
\end{align*}
Thus, provided $\tau \leq (r_n/\beta)^{\frac{1}{\alpha+1}}$, we
can replace $\theta_1$ with $\theta_0$ in the above argument. This
gives
\begin{align*}
L_1 &\leq \lceil 1/\tau^2 \rceil + \frac{4\theta_0}{\tau}\sum_{i=2}^nr_i + 2n\beta\tau^\alpha + 2 \log(2/\delta)\\
& \leq \lceil 1/\tau^2 \rceil + \frac{8\theta_0(\alpha+2)}{\tau}\left(44K\regconst\logfactor\right)^{\frac{\alpha+1}{\alpha+2}}(\beta n)^{\frac{1}{\alpha+2}} + 2n\beta\tau^\alpha + 2 \log(2/\delta).
\end{align*}
As above, we have taken $r_i=
2\beta^{\frac{1}{\alpha+2}}(44K\Delta_i)^{\frac{\alpha+1}{\alpha+2}}$
and approximated the sum by an integral. Since $\Delta_n =
\regconst\logfactor/(n-1)$, we can set $\tau =
2^{\frac{\alpha+1}{\alpha+2}}(44K\regconst\logfactor)^{\frac{1}{\alpha+2}}(\beta
n)^{-\frac{1}{\alpha+2}}$. This is a similar choice to what we used in
the proof of Theorem~\ref{thm:tsybakov} except that we are not
incorporating $\theta_0$ into the choice of $\tau$, and it yields a
final bound of $\order\rbr{\theta_0
  n^{\frac{2}{\alpha+2}}\nu_n^{\frac{\alpha}{\alpha+2}} +
  \log(1/\delta)}$.\hfill\ifthenelse{\equal{\version}{main}}{\qedsymbol}{\qedhere}}

\paragraph{\df{Proof of Proposition~\ref{prop:bc_dis_high_noise}.}}
\df{First we relate $\theta_1$ to $\theta_0$ in the multiclass
case. For $f \in \vFcsr(r)$, we have
\begin{align*}
\PP[h_f(x) \ne h_{f^\star}(x)] &\leq \PP[h_f(x) \ne y] +
\PP[h_{f^\star}(x) \ne y] = \EE c(h_f(x)) - c(h_{f^\star}(x))\\ & \leq
2\textrm{error}(h_{f^\star}) + r
\end{align*}
Therefore for any $r > 0$ and any $x$,
\begin{align*}
\one\{ \exists y \mid x \in \dis(r,y) \wedge \gamma(x,y,\vFcsr(r)) \geq \psi/2\} \leq \one\{x \in \dismin(r_i+2\textrm{error}(h_{f^\star}))\}.
\end{align*}
Applying this argument in the $L_1$ derivation above, we obtain
\begin{align*}
L_1 &\leq  2\sum_i \EE_i \one\{x \in \dismin(r_i+2\textrm{error}(h_{f^\star}))\} + 3\log(2/\delta)\\
& \leq 2\sum_i (r_i + 2\textrm{error}(h_{f^\star})) \theta_0 + 3\log(2/\delta).
\end{align*}
We now bound $r_i$ via Lemma~\ref{lem:version_to_csr}. In multiclass
classification, the fact that $c = \one - e_y$ for some $y$ implies
that $f^\star(x,y)$ is one minus the probability that the true label
is $y$. Thus $f^\star(x,y) \in [0,1]$, $\sum_{y} f^\star(x,y) = K-1$,
and for any $x$ we always have
\begin{align*}
\min_{y \ne y^\star(x)} f^\star(x,y) - f^\star(x,y^\star(x)) \geq 1 - 2f^\star(x,y^\star(x)).
\end{align*}
Hence, we may bound $P_\zeta$, for $\zeta \leq 1/2$, as follows
\begin{align*}
P_\zeta &= \PP_{x \sim \Dcal}\left[\min_{y \ne y^\star(x)} f^\star(x,y) - f^\star(x,y^\star(x)) \leq \zeta\right]
\leq \PP_{x \sim \Dcal}\left[1 - 2f^\star(x,y^\star(x)) \leq \zeta\right]\\
& \leq \EE_x f^\star(x,y^\star(x)) \cdot \frac{2}{1-\zeta} \leq 4 \textrm{error}(h_{f^\star}).
\end{align*}
Now apply Lemma~\ref{lem:version_to_csr} with $\zeta = \min\{1/2,
\sqrt{44K\Delta_i/\textrm{error}(h_{f^\star})}\}$, and we obtain
\begin{align*}
r_i \leq \sqrt{44K\Delta_i \textrm{error}(h_{f^\star})} + 264K\Delta_i.
\end{align*}
Using the definition of $\Delta_i$ the final $L_1$ label complexity bound is
\begin{align*}
L_1 \leq 4\theta_0 n \cdot \textrm{error}(h_{f^\star}) + \order\rbr{ \theta_0 \rbr{\sqrt{Kn\regconst\logfactor \cdot \textrm{error}(h_{f^\star})} + K\regconst\logfactor\log(n)} + \log(1/\delta)}.\tag*{\ifthenelse{\equal{\version}{main}}{\qedsymbol}{\qedhere}}
\end{align*}
 }

\section{Discussion}
This paper presents a new active learning algorithm for cost-sensitive
multiclass classification. The algorithm enjoys strong theoretical
guarantees on running time, generalization error, and label
complexity. The main algorithmic innovation is a new way to compute
the maximum and minimum costs predicted by a regression function in
the version space. We also design an online algorithm inspired by \df{our}
theoretical analysis that outperforms passive baselines both in CSMC
and structured prediction.

On a technical level, our algorithm uses a square loss oracle to
search the version space and drive the query strategy. This contrasts
with many recent results using argmax or 0/1-loss minimization oracles
for information acquisition problems like contextual
bandits~\citep{agarwal2014taming}. As these involve NP-hard
optimizations in general, an intriguing question is whether we can use
a square loss oracle for other information acquisition problems. We hope to answer
this question in future work.

%% file: appendix.tex
\section{Proof of Theorem~\ref{thm:deviation}}
\label{app:deviation}

In the proof, we mostly work with the empirical $\ell_1$ covering
number for $\Gcal$. At the end, we translate to pseudo-dimension using
a lemma of Haussler.
\begin{definition}[Covering numbers]
Given class $\Gcal \subset \Xcal \rightarrow \RR$, $\alpha >
0$, and sample $X = (x_1,\ldots,x_n) \in \Xcal^n$, the covering
number $\Ncal_1(\alpha,\Hcal,X)$ is the minimum cardinality of a set
$V \subset \RR^n$ such that for any $g \in \Gcal$, there exists a $(v_1,\ldots,v_n) \in V$ with $ \frac{1}{n}\sum_{i=1}^n|g(x_i) -
v_i| \le \alpha$.
\end{definition}
\begin{lemma}[Covering number and Pseudo-dimension~\citep{haussler1995sphere}]
\label{lem:haussler_covering}
Given a hypothesis class $\Gcal \subset \Xcal \to \RR$ with
$\pdim(\Gcal) \le d$, for any $X \in \Xcal^n$ we have
\begin{align*}
\Ncal_1(\alpha,\Gcal,X) \le e(d+1)\left(\frac{2e}{\alpha}\right)^d.
\end{align*}
\end{lemma}

Fixing $i$ and $y$, and working toward~\eqref{eq:sample_bd} we
seek to bound
\begin{align}
\PP\left(\sup_{g \in \Gcal} \sum_{j=1}^i M_j(g;y) - \frac{3}{2} \EE_j[M_j(g;y)] \ge \tau \right).
\label{eq:martingale_supremum}
\end{align}
The bound on the other tail is similar. In this section, we sometimes
treat the query rule as a function which maps an $x$ to a query
decision. We use the notation $\Qf_j: \Xcal \rightarrow \{0,1\}$ to
denote the query function used after seeing the first $j-1$
examples. Thus, our query indicator $Q_j$ is simply the instantiation
$\Qf_j(x_j)$. In this section, we work with an individual label $y$
and omit the explicit dependence in all our arguments and
notation. For notational convenience, we use $z_j = (x_j,c_j,Q_j)$ and
with $g^\star(\cdot) = f^\star(\cdot;y)$, we define
\begin{equation}
  \xi_j = g^\star(x_j) - c_j \quad \mbox{and} \quad \ell(g,x) = (g(x)
  - g^\star(x)).
  \label{eqn:xi_ell}
\end{equation}
Note that $\xi_j$ is a centered random variable, independent of
everything else. We now introduce some standard concepts from
martingale theory for the proof of 
Theorem~\ref{thm:deviation}. 

\begin{definition}[Tangent sequence]
For a dependent sequence $z_1,\ldots,z_i$ we use $z_1',\ldots,z_i'$ to
denote a \emph{tangent sequence}, where $z_j' | z_{1:j-1}
\overset{d}{=} z_j | z_{1:j-1}$, and, conditioned on $z_1,\ldots,z_i$, the random variables $z_1',\ldots,z_i'$ are independent.
\label{def:tangent}
\end{definition}

In fact, in our case we have $z_j' = (x_j',c_j', Q_j')$ where
$(x_j',c_j') \sim \Dcal$ and $Q_j': \Xcal \rightarrow \{0,1\}$ is
identical to $Q_j$. We use $M_j'(g)$ to denote the empirical excess
square loss on sample $z_j'$. Note that we will continue to use our
previous notation of $\EE_j$ to denote conditioning on
$z_1,\ldots, z_{j-1}$. We next introduce one more random process, and
then proceed with the proof.

\begin{definition}[Tree process]
  A \emph{tree process} $\Qtree$ is a binary tree of depth $i$ where
  each node is decorated with a value from $\{0,1\}$. For a Rademacher
  sequence $\epsilon \in \{-1,1\}^i$ we use $\Qtree_i(\epsilon)$ to denote
  the value at the node reached when applying the actions
  $\epsilon_1,\ldots,\epsilon_{i-1}$ from the root, where $+1$ denotes
  left and $-1$ denotes right.
  \label{def:tree}
\end{definition}

The proof follows a fairly standard recipe for proving uniform
convergence bounds, but has many steps that all require minor
modifications from standard arguments. We compartmentalize each step
in various lemmata:
\begin{enumerate}
\item In Lemma~\ref{lem:ghost}, we introduce a ghost sample and
  replace the conditional expectation $\EE_j[\cdot]$
  in~\eqref{eq:martingale_supremum} with an empirical term evaluated
  on an \emph{tangent sequence}.
\item In Lemma~\ref{lem:symmetrization}, we perform symmetrization and
  introduce Rademacher random variables and the associated tree
    process.
\item In Lemma~\ref{lem:finite_class} we control the symmetrized
  process for finite $\Gcal$.
\item In Lemma~\ref{lem:discretization}, we use the covering number to
  discretize $\Gcal$. 
\end{enumerate}

We now state and prove the intermediate results.
\begin{lemma}[Ghost sample]
\label{lem:ghost}
Let $Z = (z_1,\ldots,z_i)$ be the sequence of $(x,c,Q)$ triples and
let $Z' = (z_1',\ldots,z_i')$ be a tangent sequence. Then for $\beta_0
\ge \beta_1 > 0$ if $\tau \ge \frac{4(1+\beta_1)^2}{(\beta_0 -
  \beta_1)}$, then
\begin{align*}
\PP_{Z}& \left[\sup_{g \in \Gcal} \sum_{j=1}^iM_j(g) -
  (1+2\beta_0)\EE_jM_j(g) \ge \tau \right] \\ 
&\le 2\PP_{Z,Z'}\left[\sup_{g \in \Gcal} \sum_{j=1}^iM_j(g) - M_j'(g)
  - 2\beta_1Q_j(x_j')\ell^2(g,x_j') \ge \frac{\tau}{2}\right],\\ 
\PP_{Z}& \left[\sup_{g \in \Gcal} \sum_{j=1}^i(1-2\beta_0)\EE_jM_j(g)
  - M_j(g) \ge \tau \right] \\ 
&\le 2\PP_{Z,Z'}\left[\sup_{g \in \Gcal} \sum_{j=1}^iM_j'(g) - M_j(g)
  - 2\beta_1Q_j(x_j')\ell^2(g,x_j') \ge \frac{\tau}{2}\right]. 
\end{align*}
\end{lemma}
\begin{proof}
We derive the first inequality, beginning with the right hand side and
working toward a lower bound. The main idea is to condition on $Z$ and
just work with the randomness in $Z'$. To this end, let $\hat{g}$
achieve the supremum on the left hand side, and define the events
\begin{align*}
  E  &= \left\{\sum_{j=1}^i M_j(\hat{g}) -
  (1+2\beta_0)\EE_jM_j(\hat{g}) \ge \tau\right\}\\ 
  E' &= \left\{ \sum_{i=1}^n M_j'(\hat{g}) +
  2\beta_1Q_j(x_j')\ell^2(\hat{g},x_j') -
  (1+2\beta_0)\EE_jM_j'(\hat{g}) \le \tau/2\right\}. 
\end{align*}
Starting from the right hand side, by adding and subtracting
$(1+2\beta_0)\EE_jM_j(\hat{g})$ we get 
\begin{align*} 
& \PP_{Z,Z'}\left[\sup_{g \in \Gcal} \sum_{j=1}^i M_j(g) - M_j'(g) -
    2\beta_1Q_j(x_j')\ell^2(g,x_j') \ge \tau/2\right]\\ 
& \ge \PP_{Z,Z'}\left[ \sum_{j=1}^i M_j(\hat{g}) - M_j'(\hat{g}) -
    2\beta_1Q_j(x_j')\ell^2(\hat{g},x_j') \ge \tau/2\right]\\ 
& \ge \PP_{Z,Z'}\left[E \bigcap E'\right] = \EE_{Z}\left[\one\{E\}
    \times \PP_{Z'}[E'| Z]\right]. 
\end{align*}
Since we have defined $\hat{g}$ to achieve the supremum, we know that
\begin{align*}
\EE_Z\one\{E\} = \PP_Z\left[\sup_{g \in \Gcal} \sum_{j=1}^i M_j(g) -
  (1+2\beta_0)\EE_jM_j(g) \ge \tau\right]. 
\end{align*}
which is precisely the left hand side of the desired inequality. Hence
we need to bound the $\PP_{Z'}[E'|Z]$ term. For this term, we note that 
\begin{align*}
  \PP_{Z'}[E'|Z] &= \PP_{Z'}\left[ \sum_{i=1}^n M_j'(\hat{g}) +
  2\beta_1Q_j(x_j')\ell^2(\hat{g},x_j') -
  (1+2\beta_0)\EE_jM_j'(\hat{g}) \le \tau/2~|~Z\right]\\ 
  &\stackrel{(a)}{=} \PP_{Z'}\left[ \sum_{i=1}^n M_j'(\hat{g}) +
  2\beta_1\EE_jM_j'(\hat{g}) - (1+2\beta_0)\EE_jM_j'(\hat{g}) \le
  \tau/2~|~Z\right]\\ 
  &= \PP_{Z'}\left[ \sum_{i=1}^n \left(M_j'(\hat{g}) -
  \EE_jM_j'(\hat{g})\right) + 2\left(\beta_1 - \beta_0)
  \EE_jM_j'(\hat{g})\right)  \le \tau/2~|~Z\right]. 
\end{align*}
Here, $(a)$ follows since $\EE_j[M'_j(g)~|~Z] =
Q_j(x'_j)\ell^2(g,x'_j)$ for any $g$ by
Lemma~\ref{lem:regret_transform}. Since we are conditioning on $Z$,
$\hat{g}$ is also not a random function and the same equality holds
when we take expectation over $Z'$, even for $\hat{g}$. With this, we
can now invoke Chebyshev's inequality: 
\begin{align*}
\PP_{Z'}[E' | Z ] &= 1 - \PP_{Z'}\left[\sum_{j=1}^i M_j'(\hat{g}) +
  2\beta_1Q_j(x_j')\ell^2(\hat{g},x_j') -
  (1+2\beta_0)\EE_jM_j'(\hat{g}) \ge \tau/2 \middle| Z\right]\\ 
& \ge 1 - \frac{\Var\left[\sum_{j=1}^i M_j'(\hat{g}) +
    2\beta_1Q_j(x_j')\ell(\hat{g},x_j')\middle| Z\right]}{\left(\tau/2
  + 2(\beta_0 -
  \beta_1)\sum_i\EE_j[Q_j(x_j')\ell^2(\hat{g},x_j')\middle|
    Z]\right)^2}. 
\end{align*}
Since we are working conditional on $Z$, we can leverage the
independence of $Z'$ (recall Definition~\ref{def:tangent}) to bound
the variance term.
\begin{align*}
  & \Var\left[\sum_{j=1}^i M_j'(g) + 2\beta_1Q_j(x_j')\ell^2(\hat{g},x_j')\middle| Z\right]\\
  &\le \sum_{j=1}^i \EE_j\left[M_j'(\hat{g})^2 + (2\beta_1)^2Q_j(x_j')\ell^4(\hat{g},x_j') + 4\beta_1 M_j'(\hat{g})Q_j(x_j')\ell^2(\hat{g},x_j')\right]\\
  & \le 4(1+\beta_1)^2\sum_{j=1}^i\EE_j Q_j(x_j')\ell^2(\hat{g},x_j').
\end{align*}
Here we use that $\Var[X] \le \EE[X^2]$ and then we use that $\ell^2
\le \ell$ since the loss is bounded in $[0,1]$ along with
Lemma~\ref{lem:regret_transform}.  Returning to the application of
Chebyshev's inequality, if we expand the quadratic in the denominator
and drop all but the cross term, we get the bound
\begin{align*}
\PP_Z'[E'|Z] \ge 1- \frac{2(1+\beta_1)^2}{\tau (\beta_0 - \beta_1)} \ge 1/2,
\end{align*}
where the last step uses the requirement on $\tau$. This
establishes the first inequality.

For the second inequality the steps are nearly identical. Let
$\hat{g}$ achieve the supremum on the left hand side and define
\begin{align*}
E &= \left\{\sum_{j=1}^i(1-2\beta_0)\EE_jM_j(\hat{g}) - M_j(\hat{g}) \ge \tau\right\}\\
E' &= \left\{\sum_{j=1}^i(1-2\beta_0)\EE_jM'_j(\hat{g}) - M'_j(\hat{g}) + 2\beta_1 Q_j(x_j')\ell^2(\hat{g},x_j') \le \tau/2\right\}.
\end{align*}
Using the same argument, we can lower bound the right hand side by
\begin{align*}
\PP\left[\sup_{g \in \Gcal} \sum_{j=1}^i M_j'(g) - M_j(g) - 2\beta_1 Q_j(x_j')\ell^2(g,x_j') \ge \tau/2\right] \ge \EE_Z\left[\one\{E\} \times \PP_{Z'}\left[E' \middle| Z\right]\right]
\end{align*}
Applying Chebyshev's inequality yields the
same expression as for the other tail.
\end{proof}

\begin{lemma}[Symmetrization]
\label{lem:symmetrization}
Using the same notation as in Lemma~\ref{lem:ghost}, we have
\begin{align*}
  & \PP_{Z,Z'}\left[\sup_{g \in \Gcal} \sum_{j=1}^i M_j(g) - M_j'(g) - 2\beta_1Q_j(x_j')\ell^2(g,x_j') \ge \tau\right]\\
  & \le 2 \EE \sup_{\Qtree}\PP_{\epsilon}\left[\sup_{g \in \Gcal}\sum_{j=1}^i\epsilon_j\Qtree_j(\epsilon)((1+\beta_1)\ell(g,x_j)^2 + 2\xi_j\ell(g,x_j)) - \beta_1\Qtree_j(\epsilon)\ell(g,x_j)^2 \ge \tau/2 \right].
\end{align*}
the same bound holds on the lower tail with $(1-\beta_1)$ replacing
$(1+\beta_1)$.
\end{lemma}
\begin{proof}
For this proof, we think of $Q_j$ as a binary variable that is
dependent on $z_1,\ldots,z_{j-1}$ and $x_j$. Similarly $Q_j'$ depends
on $z_1,\ldots,z_{j-1}$ and $x_j'$. Using this notation, and
decomposing the square loss, we get
\begin{align*}
M_j(g) = Q_j \left[(g(x_j) - c_j)^2 - (g^\star(x_j)-c_j)^2\right] = Q_j\ell^2(g,x_j) + 2 Q_j\xi_j\ell(g,x_j).
\end{align*}
As such, we can write
\begin{align*}
M_j(g) - M_j'(g) - 2\beta_1Q_j'\ell^2(g,x_j') & = \underbrace{(1+\beta_1)Q_j\ell^2(g,x_j) + 2Q_j\xi_j\ell(g,x_j)}_{\triangleq T_{1,j}} - \underbrace{\beta_1Q_j\ell^2(g,x_j)}_{\triangleq T_{2,j}}\\
& - \underbrace{(1+\beta_1)Q_j'\ell^2(g,x_j') - 2Q_j'\xi_j'\ell(g,x_j')}_{\triangleq T_{1,j}'} - \underbrace{\beta_1Q_j'\ell^2(g,x_j')}_{\triangleq T_{2,j}'}.
\end{align*}
Here we have introduce the short forms $T_{1,j}, T_{2,j}$ and the
primed version just to condense the derivations. Overall we must bound
\begin{align*}
\PP\left[\sup_{g \in \Gcal} \sum_{j=1}^i T_{1,j} - T_{2,j} - T_{1,j}' - T_{2,j}' \ge \tau\right] = \EE\one\left\{\sup_{g \in \Gcal} \sum_{j=1}^i T_{1,j} - T_{2,j} - T_{1,j}' - T_{2,j}' \ge \tau\right\}.
\end{align*}
Observe that in the final term $T_{1,i}, T_{1,i}'$ are random
variables with identical conditional distribution, since there are no
further dependencies and $(x_i,\xi_i,Q_i)$ are identically distributed
to $(x_i',\xi_i',Q_i')$. As such, we can symmetrize the
$i^{\textrm{th}}$ term by introducing the Rademacher random variable
$\epsilon_i \in \{-1,+1\}$ to obtain
\begin{align*}
& \EE_{Z,Z'}\EE_{\epsilon_i} \one\left\{\sup_{g \in \Gcal} \sum_{j=1}^{i-1} T_{1,j} - T_{1,j}' + \epsilon_i(T_{1,i} - T_{1,i}') - \sum_{j=1}^i (T_{2,j} + T_{2,j}') \ge \tau\right\}\\
& \le \EE_{Z,Z'} \sup_{Q_i,Q_i'} \EE_{\epsilon_i}\one\left\{\sup_{g \in \Gcal} \sum_{j=1}^{i-1} T_{1,j} - T_{1,j}' + \epsilon_i(T_{1,i} - T_{1,i}') - \sum_{j=1}^i (T_{2,j} + T_{2,j}') \ge \tau\right\}.
\end{align*}
Here in the second step, we have replaced the expectation over
$Q_i,Q_i'$ with supremum, which breaks the future dependencies for the
$(i-1)^{\textrm{st}}$ term. Note that while we still write
$\EE_{Z,Z'}$, we are no longer taking expectation over $Q_i,Q_i'$
here. The important point is that since $x_j,\xi_j$ are all i.i.d.,
the only dependencies in the martingale are through $Q_j$s and by
taking supremum over $Q_i,Q_i'$, swapping the role of $T_{1,i-1}$ and
$T_{1,i-1}'$ no longer has any future effects. Thus we can symmetrize
the $(i-1)^{\textrm{st}}$ term. Continuing in this way, we get
\begin{align*}
  & \EE \EE_{\epsilon_{i-1}} \sup_{Q_i,Q_i'} \EE_{\epsilon_i} \one\left\{\sup_{g \in \Gcal} \sum_{j=1}^{i-2} T_{1,j} - T_{1,j}' + \sum_{j=i-1}^i\epsilon_j(T_{1,j} - T_{1,j}') - \sum_{j=1}^i (T_{2,j} + T_{2,j}') \ge \tau\right\}\\
  & \le 
  \EE\left\langle \sup_{Q_j,Q_j'} \EE_{\epsilon_j}\right\rangle_{j=1}^i \one\left\{\sup_{g \in \Gcal}\sum_{j=1}^i \epsilon_j (T_{1,j} - T_{1,j}') - \sum_{j=1}^i (T_{2,j} + T_{2,j}') \ge \tau \right\}.
\end{align*}
Here in the final expression the outer expectation is just over the
variables $x_j,x_j',\xi_j,\xi_j'$ and the bracket notation denotes
interleaved supremum and expectation. Expanding the definitions of
$T_{1,i},T_{2,i}$, we currently have
\begin{align*}
  \EE\left\langle \sup_{Q_j,Q_j'} \EE_{\epsilon_j}\right\rangle_{j=1}^i&\one\left\{\sup_{g \in \Gcal}\sum_{j=1}^i \epsilon_j\left[(1+\beta_1)Q_j\ell^2(g,x_j) +2Q_j\xi_j\ell(g,x_j)\right] - \sum_{j=1}^i\beta_1Q_j\ell^2(g,x_j)\right.\\
& \left. - \sum_{j=1}^i\epsilon_j\left[(1+\beta_1)Q_j'\ell^2(g,x_j') + 2Q_j'\xi_j'\ell(g,x_j')\right] - \sum_{j=1}^i\beta_1Q_j'\ell^2(g,x_j') \ge \tau \right\}.
\end{align*}
Next we use the standard trick of splitting the supremum over $g$ into
a supremum over two functions $g,g'$, where $g'$ optimizes the primed
terms. This provides an upper bound, but moreover if we replace $\tau$
with $\tau/2$ we can split the indicator into two and this becomes
\begin{align*}
& 2 \EE\left\langle\sup_{Q_j}\EE_{\epsilon_j}\right\rangle_{j=1}^i \one\left\{\sup_{g \in \Gcal} \sum_{j=1}^i\epsilon_jQ_j\left[(1+\beta_1)\ell^2(g,x_j) + 2\xi_j\ell(g,x_j)\right] - \beta_1Q_j\ell^2(g,x_j) \ge \tau/2\right\}\\
& = 2 \EE\sup_{\Qtree}\PP_{\epsilon}\left[\sup_{g \in \Gcal} \sum_{j=1}^i\epsilon_j\Qtree_j(\epsilon)\left[(1+\beta_1)\ell^2(g,x_j) + 2\xi_j\ell(g,x_j)\right] - \beta_1 \Qtree_j(\epsilon) \ell^2(g,x_j) \ge \tau/2\right].
\end{align*}
The tree process $\Qtree$ arises here because the interleaved supremum
and expectation is equivalent to choosing a binary tree decorated with
values from $\{0,1\}$ and then navigating the tree using the
Rademacher random variables $\epsilon$. The bound for the other tail
is proved in the same way, except $(1+\beta_1)$ is replaced by
$(1-\beta_1)$.
\end{proof}

The next lemma is more standard, and follows from the union bound and
the bound on the Rademacher moment generating function. 
\begin{lemma}[Finite class bound]
\label{lem:finite_class}
For any $x_{1:i},\xi_{1:i},\Qtree$, and for finite $\Gcal$, we have
\ifthenelse{\equal{\version}{arxiv}}{
\begin{align*}
  &\PP_{\epsilon}\left[\max_{g \in \Gcal}\sum_{j=1}^i\epsilon_j\Qtree_j(\epsilon)\left[(1+\beta_1)\ell^2(g,x_j) + 2\xi_j\ell(g,x_j)\right] - \beta_1\Qtree_j(\epsilon)\ell^2(g,x_j)\ge \tau\right]
  \le |\Gcal|\exp\left(\frac{-2\beta_1\tau}{(3+\beta_1)^2}\right).
\end{align*}
}{
\begin{align*}
  &\PP_{\epsilon}\left[\max_{g \in \Gcal}\sum_{j=1}^i\epsilon_j\Qtree_j(\epsilon)\left[(1+\beta_1)\ell^2(g,x_j) + 2\xi_j\ell(g,x_j)\right] - \beta_1\Qtree_j(\epsilon)\ell^2(g,x_j)\ge \tau\right]\\
  &\le |\Gcal|\exp\left(\frac{-2\beta_1\tau}{(3+\beta_1)^2}\right).
\end{align*}
}
The same bound applies for the lower tail. 
\end{lemma}
\begin{proof}
Applying the union bound and the Chernoff trick, we get that for any
$\lambda > 0$ the LHS is bounded by
\begin{align*}
& \sum_g \exp(-\tau\lambda)\EE_{\epsilon} \exp \left\{\lambda\left[\sum_{j=1}^i\underbrace{\epsilon_j\Qtree_j(\epsilon)\left[(1+\beta_1)\ell^2(g,x_j) + 2\xi_j\ell(g,x_j)\right] - \beta_1\Qtree_i(\epsilon)\ell^2(g,x_j)}_{\triangleq T_{3,j}}\right]\right\}\\
& = \sum_g \exp(-\tau\lambda) \EE_{\epsilon_{1:i-1}} \exp\left\{\lambda \sum_{j=1}^{i-1}T_{3,j}\right\} \times \EE_{\epsilon_i | \epsilon_{1:i-1}} \exp(\lambda T_{3,i}).
\end{align*}
Let us examine the $i^{\textrm{th}}$ term conditional on
$\epsilon_{1:i-1}$. Conditionally on $\epsilon_{1:i-1}$,
$\Qtree_i(\epsilon)$ is no longer random, so we can apply the MGF
bound for Rademacher random variables to get
\begin{align*}
&\EE_{\epsilon_i|\epsilon_{1:i-1}}\exp\left\{\lambda\epsilon_i\Qtree_i(\epsilon)\left[(1+\beta_1)\ell^2(g,x_i) + 2\xi_i\ell(g,x_i)\right] - \lambda\beta_1\Qtree_i(\epsilon)\ell^2(g,x_i)\right\}\\
& \le \exp\left\{\frac{\lambda^2}{2}\Qtree_i(\epsilon)\left[(1+\beta_1)\ell^2(g,x_i) + 2\xi_i\ell(g,x_i)\right]^2 - \lambda\beta_1\Qtree_i(\epsilon)\ell^2(g,x_i)\right\}\\
& \le \exp\left\{\lambda^2\Qtree_i(\epsilon)\frac{(3+\beta_1)^2}{2}\ell^2(g,x_i) - \lambda\beta_1\Qtree_i(\epsilon)\ell^2(g,x_i)\right\}
 \le 1.
\end{align*}
Here the first inequality is the standard MGF bound on Rademacher
random variables $\EE \exp(a \epsilon) \le \exp(a^2/2)$. In the second
line we expand the square and use that $\ell,\xi \in [-1,1]$ to upper
bound all the terms. Finally, we use the choice of $\lambda =
\frac{2\beta_1}{(3+\beta_1)^2}$. Repeating this argument from $i$ down
to $1$, finishes the proof of the upper tail. The same argument
applies for the lower tail, but we actually get $(3-\beta_1)^2$ in the
denominator, which is of course upper bounded by $(3+\beta_1)^2$,
since $\beta_1 > 0$.
\end{proof}

\begin{lemma}[Discretization]
\label{lem:discretization}
Fix $x_1,\ldots,x_i$ and let $V \subset \RR^i$ be a cover of
$\Gcal$ at scale $\alpha$ on points $x_1,\ldots,x_i$. Then for any $\Qtree$
\begin{align*}
& \PP_{\epsilon}\left[\sup_{g \in \Gcal} \sum_{j=1}^i \epsilon_j\Qtree_j(\epsilon)\left[(1+\beta_1)\ell^2(g,x_j) + 2\xi_j\ell(g,x_j)\right] - \beta_1\Qtree_j(\epsilon)\ell^2(g,x_j) \ge \tau\right] \le \\
& \PP_{\epsilon}\left[\sup_{v \in V} \sum_{j=1}^i \epsilon_j\Qtree_j(\epsilon)\left[(1+\beta_1)\ell^2(v,x_j) + 2\xi_j\ell(v,x_j)\right] - \beta_1\Qtree_j(\epsilon)\ell^2(v,x_j) \ge \tau - 4i(1+\beta)\alpha \right].
\end{align*}
The same bound holds for the lower tail with $1-\beta_1$. Here $\ell(v,x_j) = (v_j - g^\star(x_j))$.
\end{lemma}
\begin{proof}
Observe first that if $v$ is the covering element for $g$, then we are
guaranteed that
\begin{align*}
\frac{1}{i} \sum_{j=1}^i |\ell(g,x_j) - \ell(v,x_j)| &= \frac{1}{i}\sum_{j=1}^i|g(x_j) - v_j| \le \alpha,\\
\frac{1}{i} \sum_{j=1}^i |\ell^2(g,x_j) - \ell^2(v,x_j)| &= \frac{1}{i}\sum_{j=1}^i|(g(x_j) - v_j)(g(x_j)+v_j-2g^\star(x_j)| \le 2\alpha,
\end{align*}
since $g,v,g^\star \in [0,1]$. Thus, adding and subtracting the
corresponding terms for $v$, and applying these bounds, we get a
residual term of $i\alpha (2(1+\beta_1) + 2 + 2\beta_1) = 4i\alpha(1 +
\beta_1)$.
\end{proof}

\textbf{Proof of Theorem~\ref{thm:deviation}.}  Finally we can derive
the deviation bound. We first do the upper tail, $M_j - \EE_j
M_j$. Set $\beta_0 = 1/4, \beta_1 = 1/8$ and apply
Lemmas~\ref{lem:ghost} and~\ref{lem:symmetrization}
to~\eqref{eq:martingale_supremum}.
\begin{align*}
  & \PP\left[\sup_{g\in\Gcal} \sum_{j=1}^iM_j(g) - \frac{3}{2}\EE_jM_j(g) \ge \tau\right] \\
  & \le 2\PP\left[\sup_{g \in \Gcal}\sum_{j=1}^iM_j(g) - M_j'(g) - \frac{1}{4}Q_j(x_j')\ell^2(g,x_j') \ge \tau/2\right]\\
  & \le 4 \EE \sup_{\Qtree}\PP_{\epsilon}\left[\sup_{g \in \Gcal} \sum_{j=1}^i \epsilon_j\Qtree_j(\epsilon)\left(\frac{9}{8}\ell^2(g,x_j) + 2\xi_j\ell(g,x_j)\right) - \frac{1}{8}\Qtree_j(\epsilon_i)\ell^2(g,x_j) \ge \frac{\tau}{4}\right].
\end{align*}
Now let $V(X)$ be the cover for $\Gcal$ at scale $\alpha = \frac{\tau}{32i
  (9/8)} = \frac{\tau}{36i}$, which makes
$\tau/4 - 4i(1+\beta_1)\alpha = \frac{\tau}{8}$.
Thus we get the bound
\begin{align*}
& \le 4 \EE_{X,\xi} \sup_{\Qtree}\PP_{\epsilon}\left[\sup_{h \in \Hcal(X)} \sum_{j=1}^i \epsilon_j\Qtree_j(\epsilon)\left(\frac{9}{8}\ell^2(h,x_j) + 2\xi_j\ell(h,x_j)\right) - \frac{1}{8} \Qtree_j(\epsilon_i)\ell^2(h,x_j) \ge \frac{\tau}{8}\right]\\
& \le 4 \EE_{X} |V(X)| \exp\left(\frac{- 2 (1/8) (\tau/8)}{(3+1/8)^2} \right) = 2\exp(-2\tau/625) \EE_{X}\Ncal\left(\frac{\tau}{36i}, \Gcal, X\right).
\end{align*}
This entire derivation requires that $\tau \ge
\frac{4(9/8)^2}{(1/8)^2} = 324$.

The lower tail bound is similar. By
Lemmas~\ref{lem:ghost} and~\ref{lem:symmetrization}, with $\beta_0 =
1/4$ and $\beta_1 = 1/8$,
\begin{align*}
& \PP\left[\sup_{g \in \Gcal} \sum_{i=1}^n\frac{1}{2}\EE_jM_j(g) - M_j(g) \ge \tau\right]\\
& \le 2\PP\left[\sup_{g \in \Gcal} \sum_{i=1}^n\frac{1}{2}\EE_jM_j(g) - M_j(g) - 2(1/8)Q_j(x_j')\ell^2(g,x_j') \ge \frac{\tau}{2}\right]\\
& \le 4\EE\sup_{\Qtree}\PP_{\epsilon}\left[\sup_{g \in \Gcal} \sum_{j=1}^i \epsilon_j\Qtree_j(\epsilon)\left[\frac{7}{8} \ell^2(g,x_j)+2\xi_j\ell(g,x_j)\right] - \frac{1}{8}\sum_{j=1}^i\Qtree_i(\epsilon)\ell^2(g,x_j) \ge \frac{\tau}{4}\right]\\
& \le 4\EE\sup_{\Qtree}\PP_{\epsilon}\left[\sup_{g \in \Gcal} \sum_{j=1}^i \epsilon_j\Qtree_j(\epsilon)\left[\frac{9}{8}\ell^2(g,x_j)+2\xi_j\ell(g,x_j)\right] - \frac{1}{8}\sum_{j=1}^i\Qtree_i(\epsilon)\ell^2(g,x_j) \ge \frac{\tau}{4}\right].
\end{align*}
This is the intermediate term we had for the upper tail, so we obtain
the same bound.

To wrap up the proof, apply Haussler's
Lemma~\ref{lem:haussler_covering}, to bound the covering number
\begin{align*}
\EE_X \Ncal\left(\frac{\tau}{36i},\Gcal,X\right) \le e(d+1) \left( \frac{72i e}{\tau}\right)^d.
\end{align*}
Finally take a union bound over all pairs of starting and ending
indices $i<i'$, all labels $y$, and both tails to get that the total
failure probability is at most
\begin{align*}
8Ke(d+1)\exp\left( - 2\tau/625 \right)\sum_{i<i' \in [n]}\left(\frac{72e(i'-i)}{\tau}\right)^d.
\end{align*}
The result now follows from standard approximations. Specifically
we use the fact that we anyway require $\tau \ge 324$ to upper bound that
$1/\tau^d$ term, use $(i'-i)^d \le n^d$ and set the whole expression
to be at most $\delta$.


\section{Multiplicative Weights}
\label{app:mw}
For completeness we prove Theorem~\ref{thm:mw} here. For this section
only let $q^{(t)} \propto p^{(t)}$ be the distribution used by the
algorithm at round $t$. If the program is feasible, then there exists
a point that is also feasible against every distribution $q$. By
contraposition, if on iteration $t$, the oracle reports infeasibility
against $q^{(t)}$, then the original program must be infeasible.

Now suppose the oracle always finds $v_t$ that is feasible against
$q^{(t)}$. This implies
\begin{align*}
0 \le \sum_{t=1}^T \sum_{i=1}^nq^{(t)}_i(b_i - \langle a_i,v_t\rangle).
\end{align*}
Define $\ell_t(i) = \frac{b_i - \langle a_i,v_t\rangle}{\rho_i}$ which
is in $ [-1,1]$ by assumption.  We compare this term to
the corresponding term for a single constraint $i$. Using the
standard potential-based analysis, with $\Phi^{(t)} =
\sum_{i}p_i^{(t)}$, we get
\ifthenelse{\equal{\version}{arxiv}}{
\begin{align*}
\Phi^{(T+1)} = \sum_{i=1}^{m}p_i^{(T)}\left(1-\eta \ell_T(i)\right)
& \le \Phi^{(T)}\exp\left(-\eta \sum_{i=1}^mq_i^{(T)}\ell_T(i)\right)
 \le m \exp\left( -\eta\sum_{t=1}^T\sum_{i=1}^m q_i^{(t)}\ell_t(i)\right).
\end{align*}
}{
\begin{align*}
\Phi^{(T+1)} = \sum_{i=1}^{m}p_i^{(T)}\left(1-\eta \ell_T(i)\right)
& \le \Phi^{(T)}\exp\left(-\eta \sum_{i=1}^mq_i^{(T)}\ell_T(i)\right)\\
& \le m \exp\left( -\eta\sum_{t=1}^T\sum_{i=1}^m q_i^{(t)}\ell_t(i)\right).
\end{align*}
}
For any $i$, we also have
\begin{align*}
  \Phi^{(T+1)} \ge \prod_{t=1}^T (1 - \eta \ell_t(i)).
\end{align*}
Thus, taking taking logarithms and re-arranging we get
\begin{align*}
0 &\le \sum_{t=1}^T\sum_{i=1}^m q_i^{(t)}\ell_t(i) 
\le \frac{\log(m)}{\eta} + \frac{1}{\eta} \sum_{t=1}^T\log\left(\frac{1}{1-\eta\ell_t(i)}\right) 
 \le \frac{\log m}{\eta} + \sum_{t=1}^T \ell_t(i) + \eta T. 
\end{align*}
Here we use standard approximations $\log(1/(1-x)) \le x +
x^2$ (which holds for $x \le 1/2$)
along with the fact that $|\ell_t(i)| \le 1$. Using the
definition of $\ell_t(i)$ we have hence proved that
\begin{align*}
\sum_{t=1}^T \langle a_i,v_t\rangle \le Tb_i + \frac{\rho_i \log m}{\eta} + \rho_i \eta T.
\end{align*}
Now with our choice of $\eta = \sqrt{\log(m)/T}$ we get the desired
bound. If $\eta \ge 1/2$ then the result is trivial by the boundedness
guarantee on the oracle.


\section{Proofs of Lemmata}
\label{app:lemmata}

\begin{proof}[Proof of Lemma~\ref{lemma:monotone}]
  By the definitions,
  \begin{align*}
    \Rhat(g') + w'(g'(x) - c)^2 & = \Rhat(g',w',c) \le \Rhat(g) + w'(g(x) - c)^2 \\
    & = \Rhat(g) + w(g(x) - c)^2 + (w'-w)(g(x)-c)^2\\
    & \le \Rhat(g') + w(g'(x) - c)^2 + (w'-w)(g(x)-c)^2.
  \end{align*}
  Rearranging shows that
    $(w'-w)(g'(x) - c)^2 \le (w'-w)(g(x) - c)^2$.
  Since $w' \ge w$, we have $(g'(x)-c)^2 \le (g(x) - c)^2$, which
  is the second claim.
  For the first, the definition of $g$ gives
  \begin{align*}
    \Rhat(g) + w(g(x) - c)^2 \le \Rhat(g') + w(g'(x)-c)^2.
  \end{align*}
  Rearranging this inequality gives,
    $\Rhat(g') - \Rhat(g) \ge w((g(x) - c)^2 - (g'(x) - c)^2) \ge 0$, 
  which yields the result.
\end{proof}

\vspace{-0.5cm}
\begin{proof}[Proof of Lemma~\ref{lem:regret_transform}]
  We take expectation of $M_j$ over the cost conditioned on a fixed example $x_j = x$ and a fixed query outcome $Q_j(y)$:
  \begin{align*}
    \EE[M_j \mid x_j=x, Q_j(y)] &= Q_j(y)\times\EE_c[g(x)^2 -
      \fstar{x}{y}^2 - 2c(y)(g(x)-\fstar{x}{y}) \mid x_j = x]
    \nonumber \\ 
    &=  Q_j(y)\big(g(x)^2 - \fstar{x}{y}^2 -
    2\fstar{x}{y}(g(x)-\fstar{x}{y})\big) \\ 
    &=  Q_j(y) (g(x) - \fstar{x}{y})^2.
  \end{align*}
  The second equality is by Assumption~\ref{assumption:realizability},
  which implies $\EE[c(y) \mid x_j =x ] = \fstar{x}{y}$. Taking
  expectation over $x_j$ and $Q_j(y)$, we have
  \begin{align*}
    \EE_j[M_j] \;=\; \EE_j\left[ Q_j(y) (g(x_j) - \fstar{x_j}{y})^2 \right].
  \end{align*}
  For the variance:
  \begin{align*}
    \Var_j[M_j] &\leq \EE_j[M_j^2] 
    = \EE_j\left[ Q_j(y) (g(x_j) - \fstar{x_j}{y}
      )^2(g(x_j)+\fstar{x_j}{y} - 2c(y))^2 \right]\\ 
    &\leq 4 \cdot \EE_j  \left[ Q_j(y)(g(x_j) - \fstar{x_j}{y} )^2 \right]
    = 4 \EE_j [M_j]. 
  \end{align*}
  This concludes the proof.
\end{proof}

\vspace{-0.5cm}
\begin{proof}[Proof of Lemma~\ref{lem:refined_generalization_bd}]
  Fix some $f \in \Fcal_i$, and let $\hat{y} = h_{\vf}(x)$ and $y^\star = h_{\vf^\star}(x)$ for shorthand, but note that both depend on $x$.
  Define
  \begin{align*}
    S_\zeta(x) &= \ind{\vf^\star(x,\hat{y}) \le \vf^\star(x,y^\star) + \zeta}, \qquad 
    S_\zeta'(x) = \ind{\min_{y \ne y^\star} \vf^\star(x,y) \le \vf^\star(x,y^\star) + \zeta}.
  \end{align*}
  Observe that for fixed $\zeta$, $S_\zeta(x) \ind{\hat{y} \ne y^\star} \le S_\zeta'(x)$ for all $x$.
  We can also majorize the complementary indicator to obtain the inequality
  \begin{align*}
    S_\zeta^C(x) \le \frac{f^\star(x,\hat{y}) - f^\star(x,y^\star)}{\zeta}.
  \end{align*}
  We begin with the definition of realizability, which gives
  \begin{align*}
    \EE_{x,c}[c(h_{\vf}(x))- c(h_{{\vf}^\star}(x)] &=
    \EE_x\left[\left(\vf^\star(x,\hat{y}) - \vf^\star(x,y^\star)\right)\ind{\hat{y}\ne
        y^\star}\right]\\ 
    & = \EE_x\left[\left(S_\zeta(x) +
      S_\zeta^C(x)\right)\left(\vf^\star(x,\hat{y}) -
      \vf^\star(x,y^\star)\right)\ind{\hat{y}\ne y^\star}\right]\\ 
    & \le \zeta \EE_x S_\zeta'(x)  + \EE_i \left[S_\zeta^C(x)
      \ind{\hat{y}\ne y^\star} \left(\vf^\star(x,\hat{y}) -
      \vf^\star(x,y^\star)\right)\right].
  \end{align*}
  The first term here is exactly the $\zeta P_\zeta$ term in the bound.
  We now focus on the second term, which depends on our query rule.
  For this we must consider three cases.

  \textbf{Case 1.} If both $\hat{y}$ and $y^\star$ are not queried,
  then it must be the case that both have small cost ranges.  This
  follows since $\vf \in \vF_i$ and $h_{\vf}(x) = \hat{y}$ so
  $y^\star$ does not dominate $\hat{y}$.  Moreover, since the cost
  ranges are small on both $\hat{y}$ and $y^\star$ and since we know that
  $f^\star$ is well separated under event $S_\zeta^C(x)$, the
  relationship between $\zeta$ and $\psi_i$ governs whether we make a
  mistake or not.  Specifically, we get that $S_\zeta^C(x)
  \ind{\hat{y}\ne y^\star}\ind{\textrm{no query}} \le \ind{\zeta \le
    2\psi_i}$ at round $i$.  In other words, if we do not query and
  the separation is big but we make a mistake, then it must mean that
  the cost range threshold $\psi_i$ is also big.

  Using this argument, we can bound the second term as,
  \begin{align*}
    & \EE_i\left[S_\zeta^C(x) \ind{\hat{y}\ne y^\star} (1-Q_i(\hat{y}))(1-Q_i(y^\star))\left(\vf^\star(x,\hat{y}) -
      \vf^\star(x,y^\star)\right)\right]\\ 
    & \le \EE_i\left[S_\zeta^C(x) \ind{\hat{y}\ne y^\star} (1-Q_i(\hat{y}))(1-Q_i(y^\star)) 2\psi_i \right]\\ 
    & \le \EE_i\left[\ind{\zeta \le 2\psi_i} 2\psi_i \right] =
    \ind{\zeta \le 2\psi_i} 2\psi_i. 
  \end{align*}

  \textbf{Case 2.} If both $\hat{y}$ and $y^\star$ are queried, we can
  relate the second term to the square loss,
  \begin{align*}
    & \EE_i\left[S_\zeta^C(x) Q_i(\hat{y})Q_i(y^\star) \left(\vf^\star(x,\hat{y}) -
      \vf^\star(x,y^\star)\right)\right]\\ 
    & \le \frac{1}{\zeta}\EE_i\left[Q_i(\hat{y})Q_i(y^\star)\left(\vf^\star(x,\hat{y}) -
      \vf^\star(x,y^\star)\right)^2\right]\\ 
    & \le \frac{1}{\zeta}\EE_i\left[Q_i(\hat{y})Q_i(y^\star) \left(\vf^\star(x,\hat{y}) - \vf(x,\hat{y}) +
      \vf(x,y^\star) - \vf^\star(x,y^\star)\right)^2\right]\\ 
    & \le \frac{2}{\zeta}\EE_i\left[Q_i(\hat{y})(\vf^\star(x,\hat{y}) - \vf(x,\hat{y}))^2 + Q_i(y^\star)(\vf(x,y^\star) -
      \vf^\star(x,y^\star))^2\right]\\ 
    & \le \frac{2}{\zeta}\sum_y \EE_i\left[Q_i(y)(\vf^\star(x,y) - \vf(x,y))^2 \right]
    = \frac{2}{\zeta}\sum_y \EE_i\left[ M_i(\vf;y)\right]. 
  \end{align*}
  Passing from the second to third line here is justified by the fact
  that $\vf^\star(x,\hat{y}) \ge \vf^\star(x,y^\star)$ and
  $\vf(x,\hat{y}) \le \vf(x,y^\star)$ so we added two non-negative
  quantities together.  The last step uses
  Lemma~\ref{lem:regret_transform}. While not written, we also use the
  event $\ind{\hat{y} \ne y^\star}$ to save a factor of $2$.

  \textbf{Case 3.} The last case is if one label is queried and the
  other is not.  Both cases here are analogous, so we do the
  derivation for when $y(x)$ is queried but $y^\star(x)$ is not.
  Since in this case, $y^\star(x)$ is not dominated ($h_{\vf}(x)$ is
  never dominated provided $\vf \in \vF_i$), we know that the cost
  range for $y^\star(x)$ must be small.  Using this fact, and
  essentially the same argument as in case 2, we get
  \begin{align*}
    & \EE_i\left[S_\zeta^C(x)Q_i(\hat{y})(1-Q_i(y^\star)) \left(\vf^\star(x,\hat{y}) -
      \vf^\star(x,y^\star)\right)\right]\\ 
    & \le \frac{1}{\zeta} \EE_i\left[Q_i(\hat{y})(1-Q_i(y^\star)) \left(\vf^\star(x,\hat{y}) -
      \vf^\star(x,y^\star)\right)^2\right]\\ 
    & \le \frac{2}{\zeta} \EE_i\left[ Q_i(\hat{y})\left(\vf^\star(x,\hat{y}) -
      \vf(x,\hat{y})\right)^2 + (1-Q_i(y^\star))\left(\vf(x,y^\star) -
      \vf^\star(x,y^\star)\right)^2\right]\\ 
    & \le \frac{2\psi_i^2}{\zeta} + \frac{2}{\zeta}
    \EE_i\left[Q_i(\hat{y}) \left(\vf^\star(x,\hat{y}) -
      \vf(x,\hat{y})\right)^2\right] 
    \le \frac{2\psi_i^2}{\zeta} + \frac{2}{\zeta}\sum_y
    \EE_i\left[M_i(\vf;y)\right]. 
  \end{align*}
  We also obtain this term for the other case where $y^\star$ is
  queried but $\hat{y}$ is not.

  To summarize, adding up the contributions from these cases (which is
  an over-estimate since at most one case can occur and all are
  non-negative), we get
  \begin{align*}
    \EE_{x,c}[c(h_{\vf}(x))- c(h_{{\vf}^\star}(x)] \le \zeta P_\zeta +
    \ind{\zeta \le 2\psi_i} 2\psi_i + \frac{4\psi_i^2}{\zeta} +
    \frac{6}{\zeta}\sum_y \EE_i\left[M_i(\vf;y)\right].
  \end{align*}
  This bound holds for any $\zeta$, so it holds for the minimum.
\end{proof}

\vspace{-0.5cm}
\begin{proof}[Proof of Lemma~\ref{lem:version_to_csr}]
The proof here is an easier version of the generalization bound proof
for $f_i$. First, condition on the high probability event in
Theorem~\ref{thm:deviation}, under which we already showed that
$\vf^\star \in \vF_i$ for all $i \in [n]$. Now fix some $f \in
\Fcal_i$. Since by the monotonicity property defining the sets
$\Gcal_i$, we must have $f \in \Fcal_j$ for all $1 \le j \le i$, we
can immediately apply Lemma~\ref{lem:refined_generalization_bd} on all
rounds to bound the cost sensitive regret by
\begin{align*}
\frac{1}{i-1}\sum_{j=1}^{i-1}\min_{\zeta > 0}\left\{\zeta P_\zeta + \one\{\zeta \le 2\psi_j\}2\psi_j + \frac{4\psi_j^2}{\zeta} + \frac{6}{\zeta}\sum_y \EE_j[M_j(f;y)] \right\}.
\end{align*}
As in the generalization bound, the first term contributes $\zeta
P_\zeta$, the second is at most $\frac{12}{\zeta}$ and the third is at
most $4/\zeta (1 + \log(i-1))$. The fourth term is slightly different. We still apply~\eqref{eq:expectation_bd} to obtain
\begin{align*}
\frac{1}{i-1} \sum_{j=1}^{i-1} \EE_j[M_j(f;y)] &\le \frac{2}{i-1} \sum_{j=1}^{i-1}M_j(f;y) + \frac{2\logfactor}{i-1}\\
& = 2\left(\Rhat_i(f;y) - \Rhat_i(g_{i,y};y) + \Rhat_i(g_{i,y};y) - \Rhat_i(f^\star;y)\right) + \frac{2 \logfactor}{i-1}\\
& \le 2\Delta_i + \frac{2\logfactor}{i-1}
= \frac{2(\regconst+1)\logfactor}{i-1}.
\end{align*}
We use this bound for each label. Putting terms together, the
cost sensitive regret is
\begin{align*}
& \min_{\zeta > 0}\left\{\zeta P_\zeta + \frac{12}{\zeta} + \frac{4(1+\log(i-1))}{\zeta} + \frac{12K(\regconst+1)\logfactor}{\zeta(i-1)}\right\}
\le \min_{\zeta > 0}\left\{\zeta P_\zeta + \frac{44\regconst K \logfactor}{\zeta(i-1)}\right\}.
\end{align*}
This proves containment, since this upper bounds the cost
sensitive regret of every $f \in \vF_i$.
\end{proof}

\vspace{-0.5cm}
\begin{proof}[Proof of Lemma~\ref{lem:cost_range_translations}]
  The first claim is straightforward, since $\vF_i \subset
  \vFcsr(r_i)$ and since we set the tolerance parameter in the calls
  to \maxcost and \mincost to $\psi_i/4$. Specifically,
  \begin{align*}
    \widehat{\gamma}(x_i,y) \le \gamma(x_i,y,\vF_i) + \psi_i/2 \le \gamma(x_i,y,\vFcsr(r_i)) + \psi_i/2.
  \end{align*}
  For the second claim, suppose $y \ne y_i^\star$. Then
  \begin{align*}
    y \in Y_i & \Rightarrow \cminhat{x_i,y} \le
    \cmaxhat{x_i,\bar{y}_i}\\ 
    &\Rightarrow \cminhat{x_i,y} \le
    \cmaxhat{x_i,y_i^\star}\\ 
    &\Rightarrow \cmin{x_i,y, \vFcsr(r_i)} \le
    \cmax{x_i,y_i^\star, \vFcsr(r_i)} +
    \psi_i/2\\  
    &\Rightarrow \fstar{x_i}{y} -
    \gamma(x_i,y,\vFcsr(r_i)) \le
    \fstar{x_i}{y^\star_i} + \gamma(x_i,y_i^\star,\vFcsr(r_i)) + \psi_i/2.
  \end{align*}
  This argument uses the tolerance setting $\psi_i/4$,
  Lemma~\ref{lem:version_to_csr} to translate between the version
  space and the cost-sensitive regret ball, and finally the fact that
  $\vf^\star \in \vFcsr(r_i)$ since it has zero cost-sensitive
  regret. This latter fact lets us lower (upper) bound the minimum
  (maximum) cost by $\vf^\star$ prediction minus (plus) the cost
  range.

  For $y_i^\star$ we need to consider two cases. First assume
  $y_i^\star = \bar{y}_i$. Since by assumption $|Y_i| > 1$, it must be
  the case that $\cminhat{x_i,\tilde{y}_i} \le
  \cmaxhat{x_i,y_i^\star}$ at which point the above derivation
  produces the desired implication. On the other hand, if $y_i^\star
  \ne \bar{y}_i$ then $\cmaxhat{x_i,\bar{y_i}} \le
  \cmaxhat{x_i,y_i^\star}$, but this also implies that
  $\cminhat{x_i,\tilde{y}_i} \le \cmaxhat{x_i,y_i^\star}$, since
  minimum costs are always smaller than maximum costs, and $\bar{y}_i$
  is included in the search defining $\tilde{y}_i$.
\end{proof}